\documentclass{new_tlp}

\usepackage[utf8]{inputenc}
\usepackage{amsmath}
\usepackage{amssymb}
\usepackage{microtype}
\usepackage{url}
\usepackage{listings}
\usepackage[inline]{enumitem}
\usepackage{multicol}
\usepackage{multirow}
\usepackage{graphics}
\usepackage{booktabs}
\usepackage{xcolor}

\newcommand{\tuple}[1]{\langle #1 \rangle}
\renewcommand{\H}{\ensuremath{\mathbf{H}}} \newcommand{\T}{\ensuremath{\mathbf{T}}}
\newcommand{\M}{\ensuremath{\mathbf{M}}}

 \RequirePackage{bm}
\RequirePackage{textcomp}
\RequirePackage{upgreek}

\makeatletter

\def\starttrans{\xdef\endtrans{\catcode`\noexpand\@=\the\catcode`\@
  \catcode`\noexpand\:=\the\catcode`\:
 }\catcode`\@=11 \catcode`\:=11 }

\starttrans

\newbox\transbox

\newcount\transfactor
\newdimen\trans:dim
\newdimen\trans:dim:a
\newdimen\trans:dim:b
\newdimen\trans:dim:c
\newdimen\trans:dim:d

\newcount\trans:count

\def\trans:def{}
\def\trans:def:a{}
\def\trans:def:b{}
\def\trans:def:c{}
\def\trans:def:d{}

\def\transboxini{\afterassignment\transboxcheck
 \setbox\transbox}

\def\transboxcheck{\ifvoid\transbox
  \expandafter\aftergroup
 \fi\transboxtodo}

\def\transhboxini{\afterassignment\transhboxcheck
 \setbox\transbox}

\def\transhboxcheck{\ifvoid\transbox
  \expandafter\aftergroup\expandafter\transhboxwrap
 \else
  \expandafter\transhboxwrap
 \fi}

\def\transhboxwrap{\ifvbox\transbox
  \setbox\transbox\hbox{\box\transbox}\fi
 \transboxtodo}

\long\def\transboxdef#1\transboxend{\bgroup\def\transboxtodo{#1\egroup}\transboxini}

\long\def\transhboxdef#1\transboxend{\bgroup\def\transboxtodo{#1\egroup}\transhboxini}

\newif\iftransbbox
\newif\iftranscbox
\transbboxtrue
\transcboxtrue

\long\def\transbcboxdef#1\transbboxdef#2\transcboxdef#3\transboxend{\bgroup\edef\transboxtodo{\unexpanded{#1}\iftransbbox\unexpanded{#2}\fi
  \iftranscbox\unexpanded{#3}\else\box\transbox\fi
 \egroup}\transhboxini}

\def\bboxtranson{\hbox\bgroup\transbboxtrue
  \def\transboxtodo{\box\transbox\egroup}\transboxini}

\def\bboxtransoff{\hbox\bgroup\transbboxfalse
  \def\transboxtodo{\box\transbox\egroup}\transboxini}

\def\cboxtranson{\hbox\bgroup\transcboxtrue
  \def\transboxtodo{\box\transbox\egroup}\transboxini}

\def\cboxtransoff{\hbox\bgroup\transcboxfalse
  \def\transboxtodo{\box\transbox\egroup}\transboxini}

\def\boxresizeto#1#2#3{\hbox\transboxdef
  \ifx\relax#1\relax\else \wd\transbox=\dimexpr#1\relax \fi
  \ifx\relax#2\relax\else \ht\transbox=\dimexpr#2\relax \fi
  \ifx\relax#3\relax\else \dp\transbox=\dimexpr#3\relax \fi
  \box\transbox
 \transboxend}

\def\boxresize#1#2#3{\hbox\transboxdef
  \ifx\relax#1\relax\else \wd\transbox=\dimexpr\wd\transbox+#1\relax \fi
  \ifx\relax#2\relax\else \ht\transbox=\dimexpr\ht\transbox+#2\relax \fi
  \ifx\relax#3\relax\else \dp\transbox=\dimexpr\dp\transbox+#3\relax \fi
  \box\transbox
 \transboxend}

\def\boxextents#1#2#3#4{\hbox\transboxdef
  \kern\dimexpr#1\relax
   \vbox{\kern\dimexpr#3\relax
    \box\transbox
    \kern\dimexpr#4\relax
   }\kern\dimexpr#2\relax
 \transboxend}

\def\boxhextent#1#2{\boxextents{#1}{#2}\z@\z@}
\def\boxvextent#1#2{\boxextents\z@\z@{#1}{#2}}

\def\boxexts#1#2#3#4{\hbox\transboxdef
  \trans:dim:a=\wd\transbox
  \trans:dim:b=\dimexpr\trans:dim:a+#1+#2\relax
  \trans:dim:c=\dimexpr\ht\transbox+\dp\transbox\relax
  \trans:dim:d=\dimexpr\trans:dim:c+#3+#4\relax
  \edef\trans:def:a{\fdivide\trans:dim:b\trans:dim:a}\edef\trans:def:b{\fdivide\trans:dim:d\trans:dim:c}\savebp\trans:def:c-\dimexpr#1\relax
  \savebp\trans:def:d\dimexpr-#4+(#3+#4)*\dp\transbox/\trans:dim:c\relax
  \pdfliteral{q \trans:def:a\space 0 0 \trans:def:b\space
                \trans:def:c\space \trans:def:d\space cm}\savebp\trans:def\wd\transbox
  \box\transbox
  \pdfliteral{Q 1 0 0 1 \trans:def\space 0 cm}\transboxend}

\def\boxhext#1#2{\boxexts{#1}{#2}\z@\z@}
\def\boxvext#1#2{\boxexts\z@\z@{#1}{#2}}

\def\boxrevolveleft{\hbox\transbcboxdef
  \trans:dim:a=\wd\transbox
  \trans:dim:b=\ht\transbox
 \transbboxdef
  \wd\transbox=\dimexpr\ht\transbox+\dp\transbox\relax
  \ht\transbox=\trans:dim:a
  \dp\transbox=\z@
 \transcboxdef
  \pdfliteral{q 0 1 -1 0 \tobp{\trans:dim:b} 0 cm}\savebp\trans:def\wd\transbox
  \box\transbox
  \pdfliteral{Q 1 0 0 1 \trans:def\space 0 cm}\transboxend}

\def\boxrevolveright{\hbox\transbcboxdef
  \trans:dim:a=\wd\transbox
  \trans:dim:b=\dp\transbox
 \transbboxdef
  \wd\transbox=\dimexpr\ht\transbox+\dp\transbox\relax
  \ht\transbox=\trans:dim:a
  \dp\transbox=\z@
 \transcboxdef
  \pdfliteral{q 0 -1 1 0 \tobp{\trans:dim:b} \tobp{\trans:dim:a} cm}\savebp\trans:def\wd\transbox
  \box\transbox
  \pdfliteral{Q 1 0 0 1 \trans:def\space 0 cm}\transboxend}

\def\boxrotatexy#1#2#3{\hbox\transboxdef
  \floatsincos\trans:def:a\trans:def:b{#1}\savebp\trans:def:c\dimexpr#2\relax
  \savebp\trans:def:d\dimexpr#3\relax
  \pdfliteral{q \trans:def:b\space
                \negbp\trans:def:a\space
                \trans:def:a\space
                \trans:def:b\space
                \trans:def:c\space
                \trans:def:d\space cm
                1 0 0 1 \negbp\trans:def:c\space
                        \negbp\trans:def:d\space cm}\savebp\trans:def\wd\transbox
  \box\transbox
  \pdfliteral{Q 1 0 0 1 \trans:def\space 0 cm}\transboxend}

\def\boxrotatell#1{\boxrotatexy{#1}\z@{-\dp\transbox}}

\def\boxrotateul#1{\boxrotatexy{#1}\z@{\ht\transbox}}

\newif\ifbboxright

\def\box:rotate:bb#1{\trans:dim:a=\wd\transbox
 \trans:dim:b=\ht\transbox
 \trans:dim:c=\dp\transbox
 \trans:dim:d=\dimexpr\ht\transbox+\dp\transbox\relax
 \trans:count=\reducetrigangle{#1}\fractperiod\relax
 \ifcase\fracttrigfourth\trans:count\relax
  \fr@ct:sin:cos:i\trans:def:a\trans:def:b\trans:count
  \wd\transbox=\dimexpr\fr@ct:mul\trans:dim:a\trans:def:b
                      +\fr@ct:mul\trans:dim:d\trans:def:a\relax
  \ht\transbox=\dimexpr\fr@ct:mul\trans:dim:b\trans:def:b\relax
  \dp\transbox=\dimexpr\fr@ct:mul\trans:dim:a\trans:def:a
                      +\fr@ct:mul\trans:dim:c\trans:def:b\relax
  \savebp\trans:def:c=\dimexpr\fr@ct:mul\trans:dim:c\trans:def:a\relax
 \or
  \fr@ct:sin:cos:ii\trans:def:a\trans:def:b\trans:count
  \wd\transbox=\dimexpr-\fr@ct:mul\trans:dim:a\trans:def:b
                       +\fr@ct:mul\trans:dim:d\trans:def:a\relax
  \ht\transbox=\dimexpr-\fr@ct:mul\trans:dim:c\trans:def:b\relax
  \dp\transbox=\dimexpr-\fr@ct:mul\trans:dim:b\trans:def:b
                        +\fr@ct:mul\trans:dim:a\trans:def:a\relax
  \savebp\trans:def:c=\dimexpr-\fr@ct:mul\trans:dim:a\trans:def:b
                              +\fr@ct:mul\trans:dim:c\trans:def:a\relax
 \or
  \fr@ct:sin:cos:iii\trans:def:a\trans:def:b\trans:count
  \wd\transbox=\dimexpr-\fr@ct:mul\trans:dim:a\trans:def:b
                       -\fr@ct:mul\trans:dim:d\trans:def:a\relax
  \ht\transbox=\dimexpr-\fr@ct:mul\trans:dim:a\trans:def:a
                       -\fr@ct:mul\trans:dim:c\trans:def:b\relax
  \dp\transbox=\dimexpr-\fr@ct:mul\trans:dim:b\trans:def:b\relax
  \savebp\trans:def:c=\dimexpr-\fr@ct:mul\trans:dim:a\trans:def:b
                              -\fr@ct:mul\trans:dim:b\trans:def:a\relax
 \or
  \fr@ct:sin:cos:iv\trans:def:a\trans:def:b\trans:count
  \wd\transbox=\dimexpr\fr@ct:mul\trans:dim:a\trans:def:b
                      -\fr@ct:mul\trans:dim:d\trans:def:a\relax
  \ht\transbox=\dimexpr\fr@ct:mul\trans:dim:b\trans:def:b
                      -\fr@ct:mul\trans:dim:a\trans:def:a\relax
  \dp\transbox=\dimexpr\fr@ct:mul\trans:dim:c\trans:def:b\relax
  \savebp\trans:def:c=\dimexpr-\fr@ct:mul\trans:dim:b\trans:def:a\relax
 \fi
 \ifbboxright
  \trans:dim:d=\dimexpr\fr@ct:mul\trans:dim:a\trans:def:a\relax
  \ht\transbox=\dimexpr\ht\transbox+\trans:dim:d\relax
  \dp\transbox=\dimexpr\dp\transbox-\trans:dim:d\relax
  \savebp\trans:def:d\trans:dim:d
 \else
  \def\trans:def:d{0}\fi
 \edef\trans:def:a{\fr@ct:div\trans:def:a}\edef\trans:def:b{\fr@ct:div\trans:def:b}\pdfliteral{q \trans:def:b\space
               \negbp\trans:def:a\space
               \trans:def:a\space
               \trans:def:b\space
               \trans:def:c\space
               \trans:def:d\space cm}\savebp\trans:def=\wd\transbox
 \box\transbox
 \pdfliteral{Q 1 0 0 1 \trans:def\space 0 cm}}

\bboxrightfalse

\def\box:slant:bb#1#2{\trans:dim:a=\wd\transbox
 \trans:dim:b=\ht\transbox
 \trans:dim:c=\dp\transbox
 \trans:dim:d=\dimexpr\ht\transbox+\dp\transbox\relax
 \trans:count=\reducetrigangle{#1}{2*\fractfourth}\relax
 \ifcase\fracttrigfourth\trans:count\relax
  \fr@ct:sin:cos:i\trans:def:a\trans:def:b\trans:count
  \wd\transbox=\dimexpr\trans:dim:a+\trans:dim:d*\trans:def:a/\trans:def:b\relax
  \savebp\trans:def:c=\dimexpr\trans:dim:c*\trans:def:a/\trans:def:b\relax
 \or
  \fr@ct:sin:cos:ii\trans:def:a\trans:def:b\trans:count
  \wd\transbox=\dimexpr\trans:dim:a-\trans:dim:d*\trans:def:a/\trans:def:b\relax
  \savebp\trans:def:c=\dimexpr-\trans:dim:b*\trans:def:a/\trans:def:b\relax
 \fi
 \edef\trans:def{\fdivide\trans:def:a\trans:def:b}\trans:count=\reducetrigangle{#2}{2*\fractfourth}\relax
 \ifcase\fracttrigfourth\trans:count\relax
  \fr@ct:sin:cos:i\trans:def:a\trans:def:b\trans:count
  \ht\transbox=\dimexpr\trans:dim:b+\trans:dim:a*\trans:def:a/\trans:def:b\relax
 \or
  \fr@ct:sin:cos:ii\trans:def:a\trans:def:b\trans:count
  \dp\transbox=\dimexpr\trans:dim:c-\trans:dim:a*\trans:def:a/\trans:def:b\relax
 \fi
 \ifbboxright
  \trans:dim:d=\dimexpr-\trans:dim:a*\trans:def:a/\trans:def:b\relax
  \ht\transbox=\dimexpr\ht\transbox+\trans:dim:d\relax
  \dp\transbox=\dimexpr\dp\transbox-\trans:dim:d\relax
  \savebp\trans:def:d\trans:dim:d
 \else
  \def\trans:def:d{0}\fi
 \edef\trans:def:a{\fdivide\trans:def:a\trans:def:b}\pdfliteral{q 1
               \trans:def:a\space
               \trans:def\space
               1
               \trans:def:c\space
               \trans:def:d\space cm}\savebp\trans:def=\wd\transbox
 \box\transbox
 \pdfliteral{Q 1 0 0 1 \trans:def\space 0 cm}}

\def\boxxform{\hbox\transboxdef
  \immediate\pdfxform\transbox
  \pdfrefxform\pdflastxform
 \transboxend}

\def\boxxformspec#1\boxxform{\hbox\transboxdef
  \immediate\pdfxform#1\transbox
  \pdfrefxform\pdflastxform
 \transboxend}

\def\boxraise#1{\hbox\transboxdef
  \raise\dimexpr#1\relax\box\transbox
 \transboxend}

\def\boxlower#1{\hbox\transboxdef
  \lower\dimexpr#1\relax\box\transbox
 \transboxend}

\def\boxbaselineat#1{\hbox\transboxdef
  \lower\dimexpr(\ht\transbox+\dp\transbox)*(#1)/\transfactor-\dp\transbox\relax
  \box\transbox
 \transboxend}

\def\boxmoveleft#1{\vbox\transboxdef
  \moveleft\dimexpr#1\relax\box\transbox
 \transboxend}

\def\boxmoveright#1{\vbox\transboxdef
  \moveright\dimexpr#1\relax\box\transbox
 \transboxend}

\def\b@x:rule#1#2#{#1\transboxdef
  \setbox0\hbox{\vrule width\wd\transbox height\ht\transbox depth\dp\transbox #2}\wd\transbox=\wd0
  \ht\transbox=\ht0
  \dp\transbox=\dp0
  \box\transbox
 \transboxend#1}

\def\hboxr{\b@x:rule\hbox}
\def\vboxr{\b@x:rule\vbox}
\def\vtopr{\b@x:rule\vtop}

\def\boxgs#1#2{\hbox\transboxdef
  \pdfliteral{q #1}\savebp\trans:def\wd\transbox
  \box\transbox
  \pdfliteral{#2 Q 1 0 0 1 \trans:def\space 0 cm}\transboxend}

\def\boxmarkers#1#2#3{\hbox\transboxdef
  \copy\transbox
  \trans:dim:a=\dimexpr#1\relax
  \trans:dim:b=\dimexpr#2\relax
  \pdfliteral{q #3}\savebp\trans:def-\dp\transbox
  \box:markers:h
  \savebp\trans:def\ht\transbox
  \box:markers:h
  \savebp\trans:def-\wd\transbox
  \box:markers:v
  \savebp\trans:def\z@
  \box:markers:v
  \pdfliteral{S Q}\setbox\transbox\box\voidb@x
 \transboxend}

\def\box:markers:h{\savebp\trans:def:a\trans:dim:a
  \savebp\trans:def:b\trans:dim:b
  \pdfliteral{\trans:def:a\space\trans:def\space m \trans:def:b\space\trans:def\space l}\savebp\trans:def:a\dimexpr-\wd\transbox-\trans:dim:a\relax
  \savebp\trans:def:b\dimexpr-\wd\transbox-\trans:dim:b\relax
  \pdfliteral{\trans:def:a\space\trans:def\space m \trans:def:b\space\trans:def\space l}}

\def\box:markers:v{\savebp\trans:def:a\dimexpr-\dp\transbox-\trans:dim:a\relax
  \savebp\trans:def:b\dimexpr-\dp\transbox-\trans:dim:b\relax
  \pdfliteral{\trans:def\space \trans:def:a\space m \trans:def\space \trans:def:b\space l}\savebp\trans:def:a\dimexpr\ht\transbox+\trans:dim:a\relax
  \savebp\trans:def:b\dimexpr\ht\transbox+\trans:dim:b\relax
  \pdfliteral{\trans:def\space \trans:def:a\space m \trans:def\space \trans:def:b\space l}}

\def\boxphantom{\hbox\transboxdef
  \hbox to\wd\transbox
   {\vrule width\z@ height\ht\transbox depth\dp\transbox\hss}\transboxend}

\def\boxsmash{\hbox\transhboxdef
  \wd\transbox=\z@
  \ht\transbox=\z@
  \dp\transbox=\z@
  \box\transbox
 \transboxend}

\def\hboxsmash{\hbox\transhboxdef
  \wd\transbox=\z@
  \box\transbox
 \transboxend}

\def\vboxsmash{\vbox\transhboxdef
  \ht\transbox=\z@
  \dp\transbox=\z@
  \box\transbox
 \transboxend}

\def\box:about#1{\hbox\bgroup
  \def\transboxtodo{\trans:dim:a=\wd\transbox
   \trans:dim:b=\ht\transbox
   \trans:dim:c=\dp\transbox
   \box\transbox
   \hbox to\z@{\hss
    \hbox to\trans:dim:a{\hss
     \lower\trans:dim:c\vbox to\z@{\vss
      \vbox to\dimexpr\trans:dim:b+\trans:dim:c{\vss\tt
       #1\vbox{\vskip1ex
        \halign{\hskip1ex plus 1fil####&####\hskip1ex plus 1fil\cr
         \trans:def\span\cr
         wd & \the\trans:dim:a\cr
         ht & \the\trans:dim:b\cr
         dp & \the\trans:dim:c\cr}\vskip1ex
       }\vss}}\hss}}\egroup}\def\trans:def{}\def\trans:def:a{}\box:@bout}

\def\box:@bout#1{\ifcase
  \ifx#1\hbox 0 \else
  \ifx#1\vbox 1 \else
  \ifx#1\vtop 2 \else
  \ifx#1\box  3 \else
  \ifx#1\copy 4 \else 5 \fi\fi\fi\fi\fi
 \edef\trans:def{\trans:def\string\hbox}\expandafter\transboxini\or
 \edef\trans:def{\trans:def\string\vbox}\expandafter\transboxini\or
 \edef\trans:def{\trans:def\string\vtop}\expandafter\transboxini\or
 \edef\trans:def{\trans:def\string\box}\expandafter\boxabout:register\or
 \edef\trans:def{\trans:def\string\copy}\expandafter\boxabout:register\or
 \ifx#1\trans:def:a\errmessage{`#1' is not a box}\fi\let\trans:def:a#1\edef\trans:def{\trans:def\string#1->}\expandafter\expandafter\expandafter\box:@bout\fi
 #1}

\def\boxabout:register#1{\let\trans:def:a#1\afterassignment\boxabout:r@gister\trans:count}

\def\boxabout:r@gister{\edef\trans:def{\trans:def\the\trans:count\space
  (\ifvoid\trans:count void\else
   \ifhbox\trans:count hbox\else
   \ifvbox\trans:count vbox\fi\fi\fi)}\afterassignment\transboxtodo
 \setbox\transbox\trans:def:a\trans:count}

\def\boxabout#1{\box:about{\boxgs{#1}{}}}

\def\expandnumberafter#1#2{\expandafter#1\expandafter{\number#2}}

\def\expandtwonumbersafter#1#2#3{\expandafter#1\expandafter
 {\number#2\expandafter}\expandafter
 {\number#3}}

\def\expandthreenumbersafter#1#2#3#4{\expandafter#1\expandafter
 {\number#2\expandafter}\expandafter
 {\number#3\expandafter}\expandafter
 {\number#4}}

\def\expandnumexprafter#1#2{\expandafter#1\expandafter{\number\numexpr#2}}

\def\expandtwonumexprafter#1#2#3{\expandafter#1\expandafter
 {\number\numexpr#2\expandafter}\expandafter
 {\number\numexpr#3}}

\def\expandthreenumexprafter#1#2#3#4{\expandafter#1\expandafter
 {\number\numexpr#2\expandafter}\expandafter
 {\number\numexpr#3\expandafter}\expandafter
 {\number\numexpr#4}}

\def\expanddimexprafter#1#2{\expandafter#1\expandafter{\the\dimexpr#2}}

\edef\pt:f@ctor{\number\dimexpr100pt} \edef\bp:f@ctor{\number\dimexpr100bp}

\begingroup
 \catcode`\P=12
 \catcode`\T=12
 \lccode`P=`p
 \lccode`T=`t
 \lowercase{\gdef\with@ut:pt#1PT{#1}}
\endgroup

\def\withoutpt{\expandafter\with@ut:pt}
\def\negbp#1{\withoutpt\the\dimexpr-#1pt\relax}
\def\asbp#1{\withoutpt\the\dimexpr#1*\pt:f@ctor/\bp:f@ctor\relax}

\def\roundbp#1{\expandafter\r@undbp\the\dimexpr(#1)*\pt:f@ctor/\bp:f@ctor\relax0000\relax}
\def\r@undbp{\csname r@und:bp:\the\pdfdecimaldigits\expandafter\endcsname
	\with@ut:pt}

\expandafter\def\csname r@und:bp:0\endcsname #1.#2#3\relax{\number\numexpr#1#2/10\relax}
\expandafter\def\csname r@und:bp:1\endcsname #1.#2#3#4\relax{\round:bp:once{#1}{#2#3}\relax}
\expandafter\def\csname r@und:bp:2\endcsname #1.#2#3#4#5\relax{\round:bp:once{#1}{#2#3#4}\relax}
\expandafter\def\csname r@und:bp:3\endcsname #1.#2#3#4#5#6\relax{\round:bp:once{#1}{#2#3#4#5}\relax}
\expandafter\def\csname r@und:bp:4\endcsname #1.#2#3#4#5#6#7\relax{\round:bp:once{#1}{#2#3#4#5#6}\relax}

\def\round:bp:once#1#2{\ifnum#11<0-\number\numexpr-\else\number\numexpr\fi
 #1+(\m@ne+\expandafter\r@und:bp:once\number\numexpr1#2/10\relax}

\def\r@und:bp:once#1#2\relax{#1)\relax\ifnum#2>0.#2\fi}

\def\set:bp:rounder#1#2{\expandafter\edef\csname #1:\the\numexpr#2\relax\endcsname##1{\unexpanded{\expandafter\expandafter\expandafter}\expandafter\noexpand
  \csname r@und:bp:\the\numexpr#2\relax\endcsname
  \unexpanded{\expandafter\with@ut:pt\the}\dimexpr(##1)*\unexpanded{\pt:f@ctor/\bp:f@ctor}\relax0000\relax}}

\set:bp:rounder{roundbpto}{0}
\set:bp:rounder{roundbpto}{1}
\set:bp:rounder{roundbpto}{2}
\set:bp:rounder{roundbpto}{3}
\set:bp:rounder{roundbpto}{4}

\def\roundbpto#1{\csname roundbpto:#1\endcsname}

\def\enablebpround{\let\tobp\roundbp}
\def\disablebpround{\let\tobp\asbp}
\def\setbpround#1{\expandafter\let\expandafter\tobp\csname roundbpto:\the\numexpr#1\relax\endcsname}

\enablebpround

\def\savebp#1{\def\s@vebp{\edef#1{\tobp{\bp:dim@n}}}\afterassignment\s@vebp\bp:dim@n}

\newdimen\bp:dim@n

\def\absoluteint#1{\numexpr\ifnum#1<\z@-\fi#1}

\def\absolutedim#1{\dimexpr\ifdim#1<\z@-\fi#1}

\def\expanddivisionafter#1#2#3{\expandnumexprafter#1{#2/#3}{#2}{#3}}

\def\divfloor{\expandtwonumexprafter\dividefloor}
\def\dividefloor{\expanddivisionafter\divide:fl@@r}
\def\divide:fl@@r#1#2#3{\numexpr#1\ifcase\ifnum#2<0 \ifnum#3<0 1 \else 0 \fi
        \else      \ifnum#3<0 1 \else 0 \fi \fi
 \ifnum\numexpr#1*#3>#2-\@ne\fi\or
 \ifnum\numexpr#1*#3<#2-\@ne\fi\fi}

\def\divceil{\expandtwonumexprafter\divideceil}
\def\divideceil{\expanddivisionafter\divide:c@il}
\def\divide:c@il#1#2#3{\numexpr#1\ifcase\ifnum#2<0 \ifnum#3<0 1 \else 0 \fi
        \else      \ifnum#3<0 1 \else 0 \fi \fi
 \ifnum\numexpr#1*#3<#2+\@ne\fi\or
 \ifnum\numexpr#1*#3>#2+\@ne\fi\fi}

\def\divint{\expandtwonumexprafter\divideint}
\def\divideint{\expanddivisionafter\divide:int}
\def\divide:int#1#2#3{\numexpr#1\ifcase\ifnum#2<0 \ifnum#3<0 3 \else 1 \fi
        \else      \ifnum#3<0 2 \else 0 \fi \fi
 \ifnum\numexpr#1*#3>#2-\@ne\fi\or
 \ifnum\numexpr#1*#3<#2+\@ne\fi\or
 \ifnum\numexpr#1*#3>#2+\@ne\fi\or
 \ifnum\numexpr#1*#3<#2-\@ne\fi\fi}

\def\divnint{\expandtwonumexprafter\dividenint}
\def\dividenint#1#2{\numexpr#1/#2}

\def\modulo{\expanddivisionafter\do:m@dulo}
\def\do:m@dulo#1#2#3{\numexpr#2-#3*\divide:fl@@r{#1}{#2}{#3}\relax}

\def\dividefloorpos{\expanddivisionafter\divide:fl@@r:pos}
\def\divide:fl@@r:pos#1#2#3{\numexpr#1\ifnum\numexpr#1*#3>#2-\@ne\fi}

\def\divide:c@il:pos#1#2#3{\numexpr#1\ifnum\numexpr#1*#3<#2+\@ne\fi}

\def\modpos{\expandtwonumexprafter\modulopos}
\def\modulopos{\expanddivisionafter\modulo:p@s}
\def\modulo:p@s#1#2#3{\numexpr#2-#3*\divide:fl@@r:pos{#1}{#2}{#3}\relax}

\def\floatround#1{\divnint{\dimexpr#1pt}\p@}
\def\floatfloor#1{\divfloor{\dimexpr#1pt}\p@}
\def\floatceil#1{\divceil{\dimexpr#1pt}\p@}
\def\floatint#1{\divint{\dimexpr#1pt}\p@}
\def\floatnint#1{\divnint{\dimexpr#1pt}\p@}

\def\fdivide{\expandtwonumexprafter\flo@t:divide}
\def\flo@t:divide#1#2{\withoutpt\the\dimexpr\numexpr#1*\p@/#2\relax sp\relax}

\def\divfloat{\expandthreenumexprafter\dividefloat}

\def\dividefloat#1#2#3{\expandnumberafter\divide:flo@t {\absoluteint{\divideint{#1}{#2}}}{#1}{#2}{#3}}

\def\divide:flo@t#1#2#3{\ifnum#2<0 \ifnum#3>0 -\fi\else
 \ifnum#2>0 \ifnum#3<0 -\fi\fi\fi
 #1.\expandthreenumexprafter\divide:fl@@t
 {#1}{\absoluteint{#2}}{\absoluteint{#3}}}

\def\divide:fl@@t#1#2#3{\expandnumexprafter\divide:flo@t:modulo{#2-#1*#3}{#3}}

\def\divide:flo@t:modulo#1#2{\ifnum#1<214748365
  \expandtwonumbersafter\divide:flo@t:result{#10}{#2\expandafter}\else
  \expandtwonumexprafter\divide:flo@t:modulo{#1/2}{#2/2\expandafter}\fi}

\def\divide:flo@t:result#1#2#3{\ifnum#3>1
  \expandtwonumexprafter\divide:flo@t:repeat
  {#3-\@ne}{\dividefloorpos{#1}{#2}\expandafter}\else
  \number\divide:flo@t:last{#1}{#2}\relax
  \expandafter\gobbletwo
 \fi{#1}{#2}}

\def\divide:flo@t:repeat#1#2#3#4{#2\expandnumexprafter\divide:flo@t:modulo{#3-#2*#4}{#4}{#1}}

\newcount\floatprecision
\def\roundlast{\let\divide:flo@t:last\dividenint}
\def\floorlast{\let\divide:flo@t:last\divideint}
\roundlast

\def\tofixedbp#1{\divfloat{\dimexpr#1}\b@\floatprecision}
\def\roundfixedbp#1{\divfloat{\dimexpr#1}\b@\pdfdecimaldigits}

\def\fractdegree#1{\numexpr16*\dimexpr#1pt}           \edef\fractfactor{\number\numexpr\maxdimen+\@ne}      \edef\fractfourth{\number\numexpr90*\fractdegree\@ne} \edef\fractperiod{\number\numexpr4*\fractfourth}

\def\reducefractangle#1{\expandnumberafter\reduce:fr@ct:angle{\fractdegree{#1}}}

\def\reduce:fr@ct:angle#1#2{\ifnum#1<0
  \numexpr#2-\modpos{-#1}{#2}\relax
 \else
  \modpos{#1}{#2}\fi}

\def\reduceintangle#1#2{\expandtwonumexprafter\reduce:int:@ngle{#1}{#2/\fractdegree\@ne}}

\def\reduce:int:@ngle#1#2{\fractdegree{\reduce:fr@ct:angle{#1}{#2}}}

\def\enablefractangle{\let\reducetrigangle\reducefractangle}
\def\disablefractangle{\let\reducetrigangle\reduceintangle}
\enablefractangle

\def\fracttrigfourth#1{\dividefloorpos{#1}\fractfourth}

\def\fr@ct:mul#1#2{#1*#2/\fractfactor}
\def\fr@ct:div#1{\fdivide{#1}\fractfactor}

\def\fr@ct:angle#1{\ifcase\numexpr#1\relax
62914560\or 47185920\or 31457280\or 16777216\or 8388608\or 4194304\or 2097152\or 1048576\or 524288\or 262144\or 131072\or 65536\or 32768\or 16384\or 8192\or 4096\or 2048\or 1024\or 512\or 256\or 128\or 64\or 32\or 16\or 8\or 4\or 2\or 1\fi}

\def\fr@ct:sin#1{\ifcase\numexpr#1\relax
929887697\or 759250125\or 536870912\or 295963357\or 149435979\or 74900443\or 37473049\or 18739379\or 9370046\or 4685068\or 2342539\or 1171270\or 585635\or 292818\or 146409\or 73204\or 36602\or 18301\or 9151\or 4575\or 2288\or 1144\or 572\or 286\or 143\or 71\or 36\or 18\fi}

\def\fr@ct:cos#1{\ifcase\numexpr#1\relax
 536870912\or 759250125\or 929887697\or 1032146887\or 1063292242\or 1071126243\or 1073087729\or 1073578288\or 1073700939\or 1073731603\or 1073739269\or 1073741185\or 1073741664\or 1073741784\or 1073741814\or 1073741822\or 1073741823\or 1073741824\or 1073741824\or 1073741824\or 1073741824\or 1073741824\or 1073741824\or 1073741824\or 1073741824\or 1073741824\or 1073741824\or 1073741824\fi}

\trans:count=0
\loop
 \expandafter\edef\csname
   fractangle:\the\trans:count\endcsname{\fr@ct:angle\trans:count}
 \expandafter\edef\csname
   fractsinvalue:\the\trans:count\endcsname{\fr@ct:sin\trans:count}
 \expandafter\edef\csname
   fractcosvalue:\the\trans:count\endcsname{\fr@ct:cos\trans:count}
 \ifnum\trans:count<27
 \advance\trans:count by1
\repeat
\def\fr@ct:angle#1{\csname fractangle:\number#1\endcsname}
\def\fr@ct:sin#1{\csname fractsinvalue:\number#1\endcsname}
\def\fr@ct:cos#1{\csname fractcosvalue:\number#1\endcsname}

\def\fracttrig#1{\expandnumexprafter\fr@ct:trig{\reducetrigangle{#1}\fractperiod}}

\def\fr@ct:trig#1{\csname fr@ct:trig:\romannumeral\fracttrigfourth{#1}+\@ne\endcsname
 {#1}}

\def\fr@ct:trig:i#1#2#3{\expandthreenumexprafter\fr@ct:trig:cont{#1}{#2}{#3}\z@}

\def\fr@ct:trig:ii#1#2#3{\expandthreenumexprafter\fr@ct:trig:cont{#1-\fractfourth}{#3}{-#2}\z@}

\def\fr@ct:trig:iii#1#2#3{\expandthreenumexprafter\fr@ct:trig:cont{#1-2*\fractfourth}{-#2}{-#3}\z@}

\def\fr@ct:trig:iv#1#2#3{\expandthreenumexprafter\fr@ct:trig:cont{#1-3*\fractfourth}{-#3}{#2}\z@}

\def\fr@ct:trig:cont#1#2#3#4{\ifcase
  \ifnum#1>0 \ifnum#1<\fr@ct:angle{#4} 0 \else 1 \fi \else 2 \fi
  \expandafter\fr@ct:trig:cont\expandafter
  {\number#1\expandafter}\expandafter
  {\number#2\expandafter}\expandafter
  {\number#3\expandafter}\expandafter
  {\number\numexpr#4+\@ne\expandafter}\or
  \expandafter\fr@ct:trig:cont\expandafter
  {\number\numexpr#1-\fr@ct:angle{#4}\expandafter}\expandafter
  {\number\numexpr\fr@ct:mul{#2}{\fr@ct:cos{#4}}+\fr@ct:mul{#3}{\fr@ct:sin{#4}}\expandafter}\expandafter
  {\number\numexpr\fr@ct:mul{#3}{\fr@ct:cos{#4}}-\fr@ct:mul{#2}{\fr@ct:sin{#4}}\expandafter}\expandafter
  {\number\numexpr#4+\@ne\expandafter}\or
  \fracttrigend{#2}{#3}\fi}

\def\fracttrigend#1#2\fi#3{\fi#3{#1}{#2}}

\def\fractsincos#1#2#3{\fracttrig{#3}\z@\fractfactor\fr@ct:sin:cos#1#2}
\def\fr@ct:sin:cos#1#2#3#4{\def#3{#1}\def#4{#2}}

\def\fr@ct:sin:cos:i#1#2#3{\fr@ct:trig:i{#3}\z@\fractfactor\fr@ct:sin:cos#1#2}
\def\fr@ct:sin:cos:ii#1#2#3{\fr@ct:trig:ii{#3}\z@\fractfactor\fr@ct:sin:cos#1#2}
\def\fr@ct:sin:cos:iii#1#2#3{\fr@ct:trig:iii{#3}\z@\fractfactor\fr@ct:sin:cos#1#2}
\def\fr@ct:sin:cos:iv#1#2#3{\fr@ct:trig:iv{#3}\z@\fractfactor\fr@ct:sin:cos#1#2}

\def\floatsincos#1#2#3{\fracttrig{#3}\z@\fractfactor\flo@t:sin:c@s#1#2}
\def\flo@t:sin:c@s#1#2#3#4{\edef#3{\fr@ct:div{#1}}\edef#4{\fr@ct:div{#2}}}

\endtrans
 \newbox\qbox
\def\usecolor#1{\csname\string\color@#1\endcsname\space}
\newcommand\bordercolor[1]{\colsplit{1}{#1}}
\newcommand\fillcolor[1]{\colsplit{0}{#1}}
\newcommand\outline[1]{\leavevmode \def\maltext{#1}\setbox\qbox=\hbox{\maltext}\boxgs{Q q 2 Tr \thickness\space w \fillcol\space \bordercol\space}{}\copy\qbox }
\makeatother
\newcommand\colsplit[2]{\colorlet{tmpcolor}{#2}\edef\tmp{\usecolor{tmpcolor}}\def\tmpB{}\expandafter\colsplithelp\tmp\relax \ifnum0=#1\relax\edef\fillcol{\tmpB}\else\edef\bordercol{\tmpC}\fi}
\def\colsplithelp#1#2 #3\relax{\edef\tmpB{\tmpB#1#2 }\ifnum `#1>`9\relax\def\tmpC{#3}\else\colsplithelp#3\relax\fi
}
\bordercolor{black}
\fillcolor{white}
\def\thickness{.3}

\newcommand{\next}{\text{\rm \raisebox{-.5pt}{\Large\textopenbullet}}}  \newcommand{\previous}{\text{\rm \raisebox{-.5pt}{\Large\textbullet}}}  \newcommand{\wnext}{\ensuremath{\widehat{\next}}}
\newcommand{\wprevious}{\ensuremath{\widehat{\previous}}}
\newcommand{\alwaysF}{\ensuremath{\square}}
\newcommand{\alwaysP}{\ensuremath{\blacksquare}}
\newcommand{\eventuallyF}{\ensuremath{\Diamond}}
\newcommand{\eventuallyP}{\ensuremath{\blacklozenge}}
        \newcommand{\until}{\ensuremath{\mathbin{\mbox{\outline{$\bm{\mathsf{U}}$}}}}}
        \newcommand{\release}{\ensuremath{\mathbin{\mbox{\outline{$\bm{\mathsf{R}}$}}}}}
        \newcommand{\while}{\ensuremath{\mathbin{\mbox{\outline{$\bm{\mathsf{W}}$}}}}}
\newcommand{\since}{\ensuremath{\mathbin{\bm{\mathsf{S}}}}}
\newcommand{\trigger}{\ensuremath{\mathbin{\bm{\mathsf{T}}}}}
        \newcommand{\finally}{\ensuremath{\mbox{\outline{$\bm{\mathsf{F}}$}}}}
\newcommand{\initially}{\ensuremath{\bm{\mathsf{I}}}}

\newcommand{\intervc}[2]{[#1..#2]}
\newcommand{\intervo}[2]{[#1..#2)}
\newcommand{\ointerv}[2]{(#1..#2]}
\newcommand{\rangec}[3]{#1 \in \intervc{#2}{#3}}
\newcommand{\rangeo}[3]{#1 \in \intervo{#2}{#3}}
\newcommand{\orange}[3]{#1 \in \ointerv{#2}{#3}}
 \newcommand{\HT}{\ensuremath{\mathrm{HT}}}

\newcommand{\LTL}{\ensuremath{\mathrm{LTL}}}
\newcommand{\LTLf}{\ensuremath{\mathrm{LTL}_{\!f}}}
\newcommand{\LTLo}{\ensuremath{\mathrm{LTL}_{\omega}}}
\newcommand{\THT}{\ensuremath{\mathrm{THT}}}
\newcommand{\THTf}{\ensuremath{\mathrm{THT}_{\!f}}}
\newcommand{\THTo}{\ensuremath{\mathrm{THT}_{\!\omega}}}
\newcommand{\TEL}{\ensuremath{\mathrm{TEL}}}
\newcommand{\TELf}{\ensuremath{{\TEL}_{\!f}}}
\newcommand{\TELo}{\ensuremath{{\TEL}_{\omega}}}

\newcommand{\LDL}{\ensuremath{\mathrm{LDL}}}
\newcommand{\LDLf}{\ensuremath{\mathrm{LDL}_{\!f}}}

\newcommand{\DEL}{\ensuremath{\mathrm{DEL}}}

 \newcommand{\sysfont}{\textit}

\newcommand{\abstem}{\sysfont{abstem}}

\newcommand{\clingo}{\sysfont{clingo}}

\newcommand{\stelp}{\sysfont{stelp}}

\newcommand{\telingo}{\sysfont{telingo}}
\newcommand{\tel}{\sysfont{tel}}

 \renewcommand{\sysfont}{\texttt}
\newtheorem{definition}{Definition}
\newtheorem{proposition}{Proposition}
\newtheorem{corollary}{Corollary}

\newtheorem{lemma}{Lemma}
\newtheorem{example}{Example}
\newtheorem{theorem}{Theorem}

\newcommand{\eqdef}{\ensuremath{\mathbin{\raisebox{-1pt}[-3pt][0pt]{$\stackrel{\mathit{def}}{=}$}}}}

\newcommand{\PV}{\ensuremath{\mathcal{A}}}
\newcommand{\myatom}{\ensuremath{a}}

\newcommand{\QLTL}{\ensuremath{\mathrm{QLTL}}}
\newcommand{\IK}{\ensuremath{\mathrm{IK}}}

\newcommand{\QLTLo}{\ensuremath{\mathrm{QLTL}_{\omega}}}

\newcommand{\SMD}{\ensuremath{\mathrm{SM}}}
\newcommand{\SM}[1]{\ensuremath{\mathrm{SM}[#1]}}

\newcommand{\vtuple}[1]{\bm{#1}}

\newcommand{\trivaluation}[3]{\ensuremath{\bm{#3}(#1,#2)}}
\newcommand{\trival}[2]{\trivaluation{#1}{#2}{m}}
\newcommand{\EM}{\ensuremath{\mathrm{EM}}}
\newcommand{\iequiv}{\ensuremath{\mathrel{\equiv_0}}}
\newcommand{\gequiv}{\ensuremath{\mathrel{\equiv}}}
\newcommand{\qed}{\hfill~\ensuremath{\boxtimes}}
\newcommand{\qedmath}{\tag*{\ensuremath{\boxtimes}}}
\newcommand{\Next}{\text{\rm \raisebox{-.5pt}{\Large\textopenbullet}}}

\def\loaded{\ensuremath{\mathit{loaded}}}
\def\unloaded{\ensuremath{\mathit{unloaded}}}

\newcommand{\Lab}[1]{\ensuremath{{\ell_{#1}}}}
\newcommand{\Bd}{\ensuremath{B}}
\newcommand{\Hd}{\ensuremath{A}}

\def\MFO{MFO($<$)}
\def\MHT{MHT($<$)}
\def\Nat{\ensuremath{\mathbb{N}}}
\newcommand{\tr}[1]{[#1]}

\def\QHT{QHT}

\def\DLTL{DLTL}
\def\u{\mathbf{u}}
\def\nec{\alwaysF}

\def\TSM{\ensuremath{\mathrm{TSM}}}
\newcommand{\TS}{\ensuremath{\mathrm{TS}}}

\newcommand{\module}[1]{\ensuremath{\mathbb{#1}}}
\newcommand{\Head}[1]{\ensuremath{H(#1)}}
\newcommand{\Stamp}[2]{\ensuremath{#1_{#2}}} \newcommand{\initial}[1]{\ensuremath{\mathit{I}(#1)}} \newcommand{\dynamic}[1]{\ensuremath{\mathit{D}(#1)}} \newcommand{\final}[1]{\ensuremath{\mathit{F}(#1)}}

\newcommand{\Terms}{\ensuremath{\mathcal{T}}} \newcommand{\Atoms}{\ensuremath{\mathcal{A}}}

\newcommand{\df}{\ensuremath{df}}
\newcommand{\state}[1]{\text{\rm \em \large q}_{#1}}

\newcommand{\subformulas}{\ensuremath{\mathit{sub}}}

\allowdisplaybreaks

\usepackage{tikz}
\usepackage[edges]{forest}

\usetikzlibrary{automata,arrows,positioning}
\renewcommand{\next}{\text{\rm \raisebox{-.5pt}{\Large\textopenbullet}}}

\makeatletter
\newcommand{\manuallabel}[2]{\def\@currentlabel{#2}\label{#1}}
\makeatother

\begin{document}

\title{Linear-Time Temporal Answer Set Programming}

\author[F. Aguado et al.]{FELICIDAD AGUADO
  and
  PEDRO CABALAR
  \\
  University of Corunna, Spain
  \and
  MART\'{I}N DI\'EGUEZ
  \\
  Universit\'e d'Angers, France
  \and
  GILBERTO P\'EREZ
  \\
  University of Corunna, Spain
  \and
  TORSTEN SCHAUB and ANNA SCHUHMANN
  \\
  University of Potsdam, Germany
  \and
  CONCEPCI\'ON VIDAL
  \\
  University of Corunna, Spain}

\submitted{[n/a]}
\revised{[n/a]}
\accepted{[n/a]}

\maketitle

\begin{abstract}
In this survey, we present an overview on (Modal) Temporal Logic Programming in view of its application to Knowledge Representation and Declarative Problem Solving.
The syntax of this extension of logic programs is the result of combining usual rules with temporal modal operators,
as in \emph{Linear-time Temporal Logic} (\LTL).
In the paper, we focus on the main recent results of the non-monotonic formalism called \emph{Temporal Equilibrium Logic} (\TEL)
that is defined for the full syntax of \LTL\ but involves a model selection criterion based on \emph{Equilibrium Logic},
a well known logical characterization of Answer Set Programming (ASP).
As a result, we obtain a proper extension of the stable models semantics for the general case of temporal formulas in the syntax of \LTL.
We recall the basic definitions for \TEL\ and its monotonic basis, the temporal logic of Here-and-There (\THT), and study the differences between finite and infinite trace length.
We also provide further useful results, such as the translation into other formalisms like Quantified Equilibrium Logic and Second-order \LTL,
and some techniques for computing temporal stable models based on automata constructions.
In the remainder of the paper, we focus on practical aspects, defining a syntactic fragment called  \emph{(modal) temporal logic programs} closer to ASP, and
explaining how this has been exploited in the construction of the solver \telingo,
a temporal extension of the well-known ASP solver \clingo\ that uses its incremental solving capabilities.
\end{abstract}
 \section{Introduction}\label{sec:introduction}

The term \emph{Temporal Logic Programming\footnote{In this paper, we use this original meaning of the term, that is, the extension of \emph{logic programming} with \emph{temporal logic}. However, nowadays, temporal logic programming can be understood as the more general discipline of combining \emph{logic programming} with \emph{temporal} reasoning, possibly in other different ways. See related work in Section~\ref{sec:related} for a more detailed explanation of this distinction.}} was introduced in the late 1980s to refer to an extension of logic programs that incorporates modal temporal operators, usually from Linear-time Temporal Logic or LTL~\cite{kamp68a,pnueli77a}.
These \LTL\ extensions of logic programming include proposals by~\citeN{moszkowski86a}, \citeN{fukotamo86a}, \citeN{gabbay87b}, \citeN{abaman89a}, \citeN{baudinet92a} or \citeN{orgwad92a}.
However, the initial interest gradually cooled down over time and \LTL\
extensions have neither become a common feature nor had a substantial impact in Prolog.
With the appearance of \emph{Answer Set Programming} or ASP~\cite{lifschitz02a} and
its use as a formalism for practical knowledge representation, the interest in the specification of dynamic systems was renewed.
ASP has been commonly used for temporal reasoning and the representation of action theories, following the methodology proposed by~\citeN{gellif93a}.
This methodology represents transition systems by adding to all time-dependent predicates an extra integer parameter:
the time point $T$ ranging over the finite interval $0,1, \dots, n-1$ where $n \in \Nat$ is some fixed length
that can be iteratively increased.
A major disadvantage of this representation is that it provides neither special language constructs nor specific inference
methods for temporal reasoning, as available in \LTL.
As a result, it is infeasible to represent properties of reactive systems (with infinite runs or \emph{traces}) such as
\emph{safety} (``Is some particular state reachable?''),
\emph{liveness} (``Does some condition happen infinitely often?''),
or to conclude that a given planning problem has no solution at all.
Even if we restrict ourselves to finite traces,\footnote{Note that, even for finite traces, \LTL-satisfiability is still a \textsc{PSPACE}-complete problem, while ASP-satisfiability lies in $\Sigma_2^P$ in the most general case.}
we cannot exploit in non-temporal ASP the succinctness of \LTL\ expressions for characterizing plans of interest
(say, sentences like ``never execute $A$ if $B \vee C$ has held since $D$'', etc) or
the interpretation of such expressions by their well-known translations into automata.

In this paper, we focus on the temporal extension of ASP with operators from \LTL.
Combinations of ASP with other temporal logics are outside the scope of this paper,
although an overall comparison is given in Section~\ref{sec:related}.
In particular, we give an overview of the main definitions and results for the (modal) Temporal Logic Programming formalism called \emph{Temporal Equilibrium Logic} or TEL~\cite{cabper07a}.
As suggested by its name,
it combines \LTL\ with the logical characterization of ASP based on \emph{Equilibrium Logic}~\cite{pearce96a}.
Although \TEL\ was introduced more than a decade ago and a first survey was already presented by~\citeN{agcadipevi13a},
several important advances have taken place since then, providing significant breakthroughs.
On the theoretical side,
apart from new results on complexity and expressiveness~\cite{bozpea15a,baldie16,agcadipevi17a},
the most important feature has arguably been the redefinition of \TEL\ to cope not only with infinite traces,
as in its original version, but also with finite traces~\cite{cakascsc18a},
following similar steps to the same variation introduced for \LTL\ by~\citeN{giavar13a}.
On the practical side,
finite traces are much better suited to the way in which ASP is used to solve planning problems and
have paved the way for the introduction of temporal operators into the ASP solver \clingo~\cite{gekakasc17a},
giving birth to the first full-fledged temporal ASP solver called \telingo~\cite{cakamosc19a}.
In the paper at hand,
we
give a revised definition of \TEL\ that incorporates the new advances,
present the most general version of the logic (which neither imposes nor forbids finiteness of traces) and then
specify the two variants:
\TELo\ for infinite traces, as first defined by~\citeN{cabper07a}, and
\TELf\ for finite ones, as recently introduced by~\citeN{cakascsc18a}.
We revisit previous results for \TELo\ under the more general umbrella of \TEL\ and then
compare their effects to the case of \TELf,
providing a homogeneous presentation of the main achievements in the topic.

The rest of the paper is organized as follows.
In the next section,
we introduce the semantic framework of \TEL, whose basic definition is made in terms of a models selection criterion
on top of the monotonic basis provided by the logic of \emph{Temporal Here-and-There} (\THT).
We also explain the relation of \TEL\ to standard ASP by proving that the former can be seen as a fragment of Quantified Equilibrium Logic with monadic predicates and a linear order relation.
In Section~\ref{sec:tht}, we study the monotonic basis, \THT, in full detail, providing some properties and relations to other formalisms.
Section~\ref{sec:tel2aut} focuses on different techniques to compute temporal equilibrium models via an automata construction.
This is done in two steps: first defining the temporal equilibrium models of a temporal theory in terms of a
second-order formula expressed in the logic of \emph{Quantified LTL} (\QLTL), and
then building the automata from this \QLTL\ formula.
In the next section, we study a normal form for temporal theories under \TEL\ semantics: what we call \emph{(modal) temporal logic programs}.
This normal form is used as the basis for a pair of translations from temporal logic programs (for finite traces) into standard ASP programs.
We also define a syntactic fragment, called \emph{present-centered} rules, that facilitates the application of incremental reasoning and, eventually, has led to the construction of the tool \telingo, an extension of \clingo\ with temporal operators.
Finally, Section~\ref{sec:related} discusses related work and Section~\ref{sec:discussion} concludes the paper.

 \section{Temporal Equilibrium Logic}
\label{sec:tel}

The definition of \emph{(Linear-time) Temporal Equilibrium Logic} (\TEL) is done in two steps.
First, we define a monotonic logic called \emph{(Linear-time) Temporal Here-and-There} (\THT), a temporal extension of the intermediate logic of Here-and-There~\cite{heyting30a}.
In a second step, we select some models from \THT\ that are said to be \emph{in equilibrium},
obtaining in this way a non-monotonic entailment relation.

\subsection{Monotonic basis: Temporal Here-and-There}
\label{subsec:tht}

The syntax of \THT\ (and \TEL) is the same as for \LTL\ with past operators.
In particular, in this paper, we use the following notation.
Given a (countable, possibly infinite) set \PV\ of propositional variables (called \emph{alphabet}),
\emph{temporal formulas} $\varphi$ are defined by the grammar:
\[
\varphi ::= a \mid \bot \mid \varphi_1 \otimes \varphi_2 \mid \previous\varphi \mid \varphi_1 \since \varphi_2 \mid \varphi_1 \trigger \varphi_2 \mid \next \varphi \mid \varphi_1 \until \varphi_2 \mid \varphi_1 \release \varphi_2 \mid \varphi_1 \while \varphi_2
\]
where $a\in\PV$ is an atom and $\otimes$ is any binary Boolean connective $\otimes \in \{\to,\wedge,\vee\}$.
The last six cases correspond to the temporal connectives whose names are listed below:
\[
\begin{array}[t]{r|cl}
\mathit{Past} & \previous & \text{for \emph{previous}}\\
              & \since    & \text{for \emph{since}}   \\
              & \trigger  & \text{for \emph{trigger}}
\end{array}
\qquad\qquad\qquad
\begin{array}[t]{r|cl}
\mathit{Future} & \next    & \text{for \emph{next}}\\
                & \until   & \text{for \emph{until}}\\
                & \release & \text{for \emph{release}}\\
                & \while   & \text{for \emph{while}}
\end{array}
\]

the intended meaning of the previous temporal operators is the following:
$\previous \varphi$ (resp.\ $\next \varphi$) means that $\varphi$ is true at the previous (resp.\ next) time point.
$\varphi \until \psi$ means that $\varphi$ is true until $\psi$ is true,
while $\varphi \since \psi$ can be read as $\varphi$ is true since $\psi$ was true.
For $\varphi \release \psi$ and $\varphi \trigger \psi$ the meaning is not as direct as for the previous operators.
$\varphi \release \psi$ means that $\psi$ is true until both $\varphi$ and $\psi$ become true simultaneously or $\psi$ is true forever.
$\varphi \trigger \psi$ means that $\psi$ is true since both $\varphi$ and $\psi$ became true simultaneously or $\psi$ has been true from the beginning.

We also define several common derived operators like the Boolean connectives
\(
\top \eqdef \neg \bot
\),
\(
\neg \varphi \eqdef  \varphi \to \bot
\),
\(
\varphi \leftrightarrow \psi \eqdef (\varphi \to \psi) \wedge (\psi \to \varphi)
\),
and the following temporal operators:
\[
\begin{array}{rcll}
       \alwaysP \varphi  & \eqdef & \bot \trigger \varphi             & \text{\emph{always before}} \\
   \eventuallyP \varphi  & \eqdef & \top \since \varphi               & \text{\emph{eventually before}} \\
             \initially\,& \eqdef & \neg \previous \top               & \text{\emph{initial}}\\
     \wprevious \varphi  & \eqdef & \previous \varphi \vee \initially & \text{\emph{weak previous}}
\end{array}
\qquad
\begin{array}{rcll}
       \alwaysF \varphi  & \eqdef & \bot \release \varphi             & \text{\emph{always afterward}}\\
   \eventuallyF \varphi  & \eqdef & \top \until \varphi               & \text{\emph{eventually afterward}}\\
               \finally  & \eqdef & \neg \next \top                   & \text{\emph{final}}\\
         \wnext \varphi  & \eqdef & \next \varphi \vee \finally       & \text{\emph{weak next}}
\end{array}
\]
A \emph{(temporal) theory} is a (possibly infinite) set of temporal formulas.
Note that we use solid operators to refer to the past, while future-time operators are denoted by outlined symbols.

Although \THT\ and \LTL\ share the same syntax, they have a different semantics,
the former being a weaker logic than the latter.
The semantics of \LTL\ relies on the concept of a \emph{trace},
a (possibly infinite) sequence of \emph{states},
each of which is a set of atoms.
For defining traces, we start by introducing some notation to deal with intervals of integer time points.
Given $a \in \mathbb{N}$ and $b \in \mathbb{N} \cup \{\omega\}$,
we let $\intervc{a}{b}$ stand for the set $\{i \in \mathbb{N} \mid a \leq i \leq b\}$, $\intervo{a}{b}$ for $\{i \in \mathbb{N} \mid a \leq i < b\}$ and $\ointerv{a}{b}$ for $\{i \in \mathbb{N} \mid a < i \leq b\}$.
In \LTL, a \emph{trace} \T\ of length $\lambda$ over alphabet \PV\ is
a sequence $\T=(T_i)_{\rangeo{i}{0}{\lambda}}$ of sets $T_i\subseteq\PV$.
We sometimes use the notation $|\T| \eqdef \lambda$ to stand for the length of the trace.
\label{page:lambda}
We say that \T\ is \emph{infinite} if $|\T|=\omega$ and \emph{finite} if $|\T| \in \mathbb{N}$.
To represent a given trace, we write a sequence of sets of atoms concatenated with `$\cdot$'.
For instance, the finite trace $\{a\} \cdot \emptyset \cdot \{a\} \cdot \emptyset$ has length 4 and makes $a$ true at even time points and false at odd ones.
For infinite traces, we sometimes use $\omega$-regular expressions like, for instance,
in the infinite trace $(\{a\} \cdot \emptyset)^\omega$ where all even positions make $a$ true and all odd positions make it false.

At each state $T_i$ in a trace, an atom $a$  can only be true, viz.\ $a \in T_i$, or false, $a \not\in T_i$.
The logic \THT\ weakens this truth assignment, following the same intuitions as the (non-temporal) logic of \HT.
In \THT, an atom can have one of three truth-values in each state, namely,
\emph{false}, \emph{assumed} (or true by default) or \emph{proven} (or certainly true).
Anything proved has to be assumed, but the opposite does not necessarily hold.
Following this idea,
a state $i$ is represented as a pair of sets of atoms $\tuple{H_i,T_i}$ with $H_i \subseteq T_i \subseteq \PV$ where
$H_i$ (standing for ``here'') contains the proven atoms, whereas $T_i$ (standing for ``there'') contains the assumed atoms.
On the other hand, false atoms are just the ones not assumed, captured by $\PV \setminus T_i$.
Accordingly,
a \emph{Here-and-There trace} (for short \emph{\HT-trace}) of length $\lambda$ over alphabet \PV\ is
a sequence of pairs
\(
(\tuple{H_i,T_i})_{\rangeo{i}{0}{\lambda}}
\)
with $H_i\subseteq T_i$ for any $\rangeo{i}{0}{\lambda}$.
For convenience, we usually represent the \HT-trace as the pair $\tuple{\H,\T}$ of traces $\H = (H_i)_{\rangeo{i}{0}{\lambda}}$ and $\T = (T_i)_{\rangeo{i}{0}{\lambda}}$.
Given $\M=\tuple{\H,\T}$, we also denote its length as $|\M| \eqdef |\H|=|\T|=\lambda$.
Note that the two traces \H, \T\ must satisfy a kind of order relation, since $H_i \subseteq T_i$ for each time point $i$.
Formally, we define the ordering $\H \leq \T$ between two traces of the same length $\lambda$ as $H_i\subseteq T_i$ for each $\rangeo{i}{0}{\lambda}$.
Furthermore, we define $\H<\T$ as both $\H\leq\T$ and $\H\neq\T$.
Thus, an \HT-trace can also be defined as any pair $\tuple{\H,\T}$ of traces such that $\H \leq \T$.
The particular type of \HT-traces satisfying $\H=\T$ are called \emph{total}.

We proceed by generalizing the extension of \HT\ with temporal operators,
called \THT~\cite{agcadipevi13a},
to \HT-traces of fixed length in order to integrate finite as well as infinite traces.
Given any \HT-trace $\M=\tuple{\H,\T}$,
we define the \THT\ satisfaction of formulas as follows.
\begin{definition}[\THT-satisfaction;~\citeNP{agcadipevi13a,cakascsc18a}\;\footnotemark]\label{def:dht:satisfaction}\footnotetext{The while operator \while\ was introduced by~\citeN{agcafapevi20a}.}
  An \HT-trace $\M=\tuple{\H,\T}$ of length $\lambda$ over alphabet \PV\
  \emph{satisfies} a temporal formula $\varphi$ at time point $\rangeo{k}{0}{\lambda}$,
  written \mbox{$\M,k \models \varphi$}, if the following conditions hold:
  \begin{enumerate}
  \item $\M,k \models \top$ and  $\M,k \not\models \bot$
  \item $\M,k \models \myatom$ if $\myatom \in H_k$ for any atom $\myatom \in \PV$
  \item $\M, k \models \varphi \wedge \psi$
    iff
    $\M, k \models \varphi$
    and
    $\M, k \models \psi$
  \item $\M, k \models \varphi \vee \psi$
    iff
    $\M, k \models \varphi$
    or
    $\M, k \models \psi$
  \item $\M, k \models \varphi \to \psi$
    iff
    $\langle \mathbf{H}', \mathbf{T} \rangle, k \not \models \varphi$
    or
    $\langle \mathbf{H}', \mathbf{T} \rangle, k \models  \psi$, for all $\mathbf{H'} \in \{ \mathbf{H}, \mathbf{T} \}$
  \item $\M, k \models \previous \varphi$
    iff
    $k>0$ and $\M, k{-}1 \models \varphi$
  \item $\M, k \models \varphi \, \since \, \psi$
    iff
    for some $\rangec{j}{0}{k}$, we have
    $\M, j \models \psi$
    and
    $\M, i \models \varphi$ for all $\orange{i}{j}{k}$
  \item $\M, k \models \varphi \trigger \psi$
    iff
    for all $\rangec{j}{0}{k}$, we have
    $\M, j \models \psi$
    or
    $\M, i \models \varphi$ for some $\orange{i}{j}{k}$
  \item $\M, k \models \next \varphi$
    iff
    $k+1<\lambda$ and $\M, k{+}1 \models \varphi$
  \item $\M, k \models \varphi \until \psi$
    iff
    for some $\rangeo{j}{k}{\lambda}$, we have
    $\M, j \models \psi$
    and
    $\M, i \models \varphi$ for all $\rangeo{i}{k}{j}$
  \item $\M, k \models \varphi \release \psi$
    iff
    for all $\rangeo{j}{k}{\lambda}$, we have
    $\M, j \models \psi$
    or
    $\M, i \models \varphi$ for some $\rangeo{i}{k}{j}$
\item $\M, k \models \varphi \while \psi$
    iff
    for all $\rangeo{j}{k}{\lambda}$, we have
    $\tuple{\H',\T},j \models \varphi$ or $\tuple{\H',\T},i \not\models \psi$ for some $\rangeo{i}{k}{j}$ and for all $\H' \in \{\H,\T\}$
  \end{enumerate}
 \qed
\end{definition}

In general, these conditions inherit the interpretation of connectives from \LTL{} (with past operators) with just a few differences.
A first minor variation is that we allow traces of arbitrary length $\lambda$,
including both infinite ($\lambda=\omega$) and finite ($\lambda \in \Nat$) traces.
The most significant difference, however, has to do with the treatment of implication,
which is inherited from the intermediate logic of \HT.
It requires that the implication is satisfied in ``both dimensions'' $\H$ (here) and $\T$ (there) of the trace, using $\tuple{\H,\T}$ (as in the other connectives) but also $\tuple{\T,\T}$.
Finally, one last difference with respect to \LTL\ is the new connective $\varphi \while \psi$
which is also a kind of temporally-iterated \HT\ implication.
Its intuitive reading\footnote{A dual, past operator could also be easily defined, but it would not have a natural, intuitive reading and its purpose would be unclear.} is ``keep doing $\varphi$ while condition $\psi$ holds.''
In \LTL, $\varphi \while \psi$ would just amount to $\neg \psi \release \varphi$,
but under \HT\ semantics both formulas have a different meaning,
as the latter may provide evidence for $\varphi$ even though the condition $\psi$ does not hold.

An \HT-trace $\M$ is a \emph{model} of a temporal theory $\Gamma$ if $\M,0 \models \varphi$ for all $\varphi \in \Gamma$.
We write $\THT(\Gamma,\lambda)$ to stand for the set of \THT-models of length $\lambda$ of a theory $\Gamma$,
and define $\THT(\Gamma) \eqdef \THT(\Gamma,\omega) \cup \bigcup_{\lambda\in\mathbb{N}}\THT(\Gamma,\lambda)$.
That is, $\THT(\Gamma)$ is the whole set of models of $\Gamma$ of any length.
For $\Gamma=\{\varphi\}$, we just write $\THT(\varphi,\lambda)$ and $\THT(\varphi)$.
We can analogously define $\LTL(\Gamma,\lambda)$, that is, the set of traces of length $\lambda$ that satisfy theory $\Gamma$, and $\LTL(\Gamma)$, that is, the \LTL-models of $\Gamma$ any length.
We omit specifying \LTL\ satisfaction since it coincides with \THT\ when \HT-traces are total.
\begin{proposition}[\citeNP{agcadipevi13a,cakascsc18a}]\label{prop:total}
Let \T\ be a trace of length $\lambda$, $\varphi$ a temporal formula, and $k\in\intervo{0}{\lambda}$ a time point.

Then,
$\T,k \models \varphi$ in \LTL\ iff $\tuple{\T,\T},k \models \varphi$.\qed
\end{proposition}
In fact, total models can be forced by adding the following set of \emph{excluded middle} axioms:
\begin{align}
\alwaysF (a \vee \neg a) \tag{\EM}
\qquad
\text{ for each atom }
a \in \PV\
\text{ in the alphabet.}
\end{align}
\begin{proposition}[\citeNP{agcadipevi13a,cakascsc18a}]\label{prop:emtotal}
Let $\tuple{\H,\T}$ be an \HT-trace and (EM) the theory containing all excluded middle axioms for every atom $a \in \PV$.
Then, $\tuple{\H,\T}$ is a model of (EM) iff $\H=\T$.\qed
\end{proposition}

Satisfaction of derived operators can be easily deduced, as shown next.
\begin{proposition}[\citeNP{agcadipevi13a,cakascsc18a}]\label{prop:satisfaction:tel}
  Let $\M=\tuple{\H,\T}$ be an \HT-trace of length $\lambda$ over \PV.
  Given the respective definitions of derived operators, we get the following satisfaction conditions:
  \begin{enumerate} \setcounter{enumi}{11}
  \item $\M, k \models \initially$
    iff
    $k =0$
  \item $\M, k \models \wprevious\varphi$
    iff
    $k =0$ or
    $\M, k{-}1 \models \varphi$
  \item $\M, k \models \eventuallyP\varphi$
    iff
    $\M, i \models \varphi$ for some $\rangec{i}{0}{k}$
  \item $\M, k \models \alwaysP\varphi$
    iff
    $\M, i \models \varphi$ for all $\rangec{i}{0}{k}$
  \item $\M, k \models \finally$
    iff
    $k+1 = \lambda$
  \item $\M, k \models \wnext\varphi$
    iff
    $k + 1=\lambda$ or
    $\M, k{+}1 \models \varphi$
  \item $\M, k \models \eventuallyF\varphi$
    iff
    $\M, i \models \varphi$ for some $\rangeo{i}{k}{\lambda}$
  \item $\M, k \models \alwaysF\varphi$
    iff
    $\M, i \models \varphi$ for all $\rangeo{i}{k}{\lambda}$
  \end{enumerate}
  \qed
\end{proposition}
A formula $\varphi$ is a \emph{tautology} (or is \emph{valid}), written $\models \varphi$,
iff $\M,k \models \varphi$ for any \HT-trace \M\ and any $\rangeo{k}{0}{\lambda}$.
We call the logic induced by the set of all tautologies \emph{Temporal logic of Here-and-There} (\THT\ for short).

Several types of equivalence can be defined in \THT.
In this sense, it is important to observe that being equivalent is something generally stronger than simply having the same set of models.
Two formulas $\varphi, \psi$ are (globally) \emph{equivalent}, written $\varphi \gequiv \psi$, iff
$\models \varphi \leftrightarrow \psi$, that is, $\M,k \models \varphi \leftrightarrow \psi$
for any \HT-trace $\M$ of length $\lambda$ and any $\rangeo{k}{0}{\lambda}$.
\THT-equivalence satisfies the rule of substitution of equivalents:
\begin{proposition}[Substitution of equivalents]
Let $\gamma[\varphi]$ be a formula with some occurrence of subformula $\varphi$ and $\gamma[\psi]$ denote the replacement of that occurrence by $\psi$ in $\gamma[\varphi]$. If $\varphi \equiv \psi$ then $\gamma[\varphi] \equiv \gamma[\psi]$.
\end{proposition}
On the other hand, we say that $\varphi, \psi$ are just \emph{initially equivalent}, written $\varphi \iequiv \psi$, if they have the same models $\THT(\varphi)=\THT(\psi)$,
that is, $\M,0 \models \varphi$ iff $\M,0 \models \psi$, for any \HT-trace $\M$.
Obviously, $\varphi \gequiv \psi$ implies $\varphi \iequiv \psi$ but not vice versa.
For example, note that $\previous a \iequiv \bot$, since $\previous a$ is always false at the initial situation,
whereas in the general case $\previous a \not\gequiv \bot$ or, otherwise, we could always replace $\previous a$ by $\bot$ in any context.

One important remark is that the finiteness of a trace $\tuple{\H,\T}$ only affects the satisfaction of formulas dealing with future-time operators.
In particular, if $\tuple{\H,\T}$ has some finite length $\lambda =n$,
then, in the semantics for $\until$, $\release$ and $\while$ in Definition~\ref{def:dht:satisfaction},
index $j$ ranges over the finite interval $\{k, \dotsc, n-1\}$.
Besides, if $\lambda =n$, the satisfaction of $\next$ forces $k < n$,
which implies that there does exist a next state $k+1$.
As a result, the formula $\next \top$ is not always satisfied, since it is false whenever $k = n = \lambda$.

Operators $\initially$ and $\finally$ exclusively depend on the value of time point $k$,
so that the valuation of atoms in $\tuple{\H,\T}$ is irrelevant to them.
As a result, they behave ``classically'' and satisfy the law of the excluded middle, that is,
$\initially \vee \neg \initially$ and $\finally \vee \neg \finally$ are \THT\ tautologies.
Besides, operator \finally\ can only be true in finite traces.
This implies that the inclusion of axiom $\eventuallyF \finally$ in any theory forces its models to be finite traces,
while including its negation $\neg \eventuallyF \finally$ causes the opposite effect,
that is, all models of the theory are infinite traces.

Several logics stronger than \THT\ can be obtained by the addition of axioms (or the corresponding restriction on sets of traces).
For instance, \THTo\ is defined as $\THT$ plus $\neg \eventuallyF \finally$, that is, \THT\ where we exclusively consider infinite \HT-traces.\footnote{This corresponds to the (stronger) version of \THT\ considered previously by~\citeN{agcadipevi13a}.}
\THTf, the finite-trace version, corresponds to $\THT$ plus $\eventuallyF \finally$. Linear Temporal Logic for possibly infinite traces, \LTL, can be obtained as $\THT$ plus $(\EM)$, that is, \THT\ with total \HT-traces, \LTLo\ is captured by $\THTo$ plus $(\EM)$, i.e.\ infinite and total \HT-traces, and finally \LTLf\ can be obtained as $\THTf$ plus $(\EM)$, that is, \LTL\ on finite traces~\cite{giavar13a}.

We study more properties and results about \THT\ later on, but we proceed next to define its non-monotonic extension, \TEL.

\subsection{Non-monotonic extension: Temporal Equilibrium Logic}
\label{subsec:tel}

Given a set of \THT-models, we define the ones in equilibrium as follows.
\begin{definition}[Temporal Equilibrium/Stable Model]\label{def:tem}
Let $\mathfrak{S}$ be some set of \HT-traces.

A total \HT-trace $\tuple{\T,\T} \in\mathfrak{S}$ is a \emph{temporal equilibrium model} of $\mathfrak{S}$ iff
there is no other $\H < \T$ such that $\tuple{\H,\T} \in\mathfrak{S}$.

The trace \T\ is called a \emph{temporal stable model} (\TS-model) of $\mathfrak{S}$.
\qed
\end{definition}
We further talk about temporal equilibrium or temporal stable models of a theory $\Gamma$ when $\mathfrak{S}=\THT(\Gamma)$, respectively.
Moreover, we write $\TEL(\Gamma,\lambda)$ and $\TEL(\Gamma)$ to stand for the temporal equilibrium models of $\THT(\Gamma,\lambda)$ and $\THT(\Gamma)$ respectively.
The corresponding sets of \TS-models are denoted as $\TSM(\Gamma,\lambda)$ and $\TSM(\Gamma)$ respectively.
One interesting observation is that, since temporal equilibrium models are total models $\tuple{\T,\T}$, due to Proposition~\ref{prop:total}, we obtain $\TSM(\Gamma,\lambda) \subseteq \LTL(\Gamma,\lambda)$ that is, temporal stable models are a subset of \LTL-models.

Since the ordering relation among traces is only defined for a fixed $\lambda$, the following can be easily observed:
\begin{proposition}[\citeNP{cakascsc18a}]
The set of temporal equilibrium models of $\Gamma$ can be partitioned by the trace length $\lambda$, that is,
$\bigcup_{\lambda=0}^\omega \TEL(\Gamma,\lambda) = \TEL(\Gamma)$. \qed
\end{proposition}

\emph{Temporal Equilibrium Logic} (\TEL) is the (non-monotonic) logic induced by temporal equilibrium models.
We can also define the variants \TELo\ and \TELf\ by applying the corresponding restriction to infinite and finite traces, respectively.

As an example of non-monotonicity,
consider the formula
\begin{align}\label{f:ys4b}
 \alwaysF (\previous \loaded \wedge \neg \unloaded & \to  \loaded)
\end{align}
that corresponds to the inertia for \loaded, together with the fact \loaded, describing the initial state for that fluent.
Without entering into too much detail,
this formula behaves as the logic program with the rules:
\begin{lstlisting}[numbers=none,belowskip=2pt,aboveskip=2pt,basicstyle=\ttfamily]
 loaded(0).
 loaded(T) :- loaded(T-1), not unloaded(T).
\end{lstlisting}
for any time point \lstinline{T>0}.
As expected, for some fixed $\lambda$, we get a unique temporal stable model of the form $\{\loaded\}^\lambda$.
This entails that $\loaded$ is always true, viz.\ $\alwaysF \loaded$, as there is no reason for \unloaded\ to become true.
Note that in the most general case of \TEL, we actually get one stable model per each possible $\lambda$, including $\lambda=\omega$.
Now, consider formula \eqref{f:ys4b} along with
\(
\loaded \wedge \next \next \unloaded
\)
which amounts to adding the fact \lstinline{unloaded(2)}.
As expected, for each $\lambda$,
the only temporal stable model now is $\T=\{\loaded\} \cdot \{\loaded\} \cdot \{\unloaded\} \cdot \emptyset^{\alpha}$ where $\alpha$ can be $*$ or $\omega$.
Note that by making $\next \next \unloaded$ true,
we are also forcing $|\T|  \geq 3$, that is, there are no temporal stable models (nor even \THT-models) of length smaller than three.
Thus, by adding the new information $\next \next \unloaded$ some conclusions that could be derived before,
such as $\alwaysF \loaded$, are not derivable any more.

As an example emphasizing the behavior of finite traces, take the formula
\begin{eqnarray}
\alwaysF (\neg a \to \next a) \label{f:defnext}
\end{eqnarray}
which can be seen as a program rule
``\lstinline{a(T+1) :- not a(T)}'' for any natural number \lstinline{T}.
As expected, temporal stable models make $a$ false in even states and true in odd ones.
However, we cannot take finite traces making $a$ false at the final state $\lambda-1$,
since the rule would force $\next a$ and this implies the existence of a successor state.
As a result, the temporal stable models of this formula have the form $(\emptyset \cdot \{a\})^+$ for finite traces in $\TELf$, or the infinite trace $(\emptyset \cdot \{a\})^\omega$ in \TELo.

Another interesting example is the temporal formula
\[
  \alwaysF (\neg \next a \to a) \wedge \alwaysF (\next a \to a).
\]
The corresponding rules
``\lstinline{a(T) :- not a(T+1)}''
and
``\lstinline{a(T) :- a(T+1)}''
have no stable model~\cite{fages94a} when grounded for all natural numbers \lstinline{T}.
This is because there is no way to build a finite proof for any \lstinline{a(T)}, as it depends on infinitely many next states to be evaluated.
The same happens in \TELo, that is, we get no infinite temporal stable model.
However in \TELf, we can use the fact that $\next a$ is always false in the last state.
Then, $\alwaysF(\neg \next a \to a)$ supports $a$ in that state and therewith $\alwaysF (\next a \to a)$ inductively supports $a$ everywhere.

As an example of a temporal expression not so close to logic programming,
consider the formula $\alwaysF\eventuallyF a$,
which is normally used in \LTLo\ to assert that $a$ occurs infinitely often.
As discussed by~\citeN{giavar13a}, if we assume finite traces, then the formula collapses to $\alwaysF (\finally \to a)$ in \LTLf,
that is, $a$ is true at the final state (and either true or false everywhere else).
The same behavior is obtained in \THTo\ and \THTf, respectively.
However, if we move to \TEL, a truth minimization is additionally required.
As a result, in \TELf,
we obtain a unique temporal stable model for each fixed $\lambda \in \mathbb{N}$,
in which $a$ is true at the last state, and false everywhere else.
Unlike this, \TELo\ yields no temporal stable model at all.
This is because for any \T\ with an infinite number of $a$'s we can always take some \H\
from which we remove $a$ at some state, and still have an infinite number of $a$'s in \H.
Thus, for any total \THTo-model $\tuple{\T,\T}$ of $\alwaysF \eventuallyF a$ there always exists some model $\tuple{\H,\T}$ with strictly smaller $\H < \T$.
Note that we can still specify infinite traces with an infinite number of occurrences of $a$, but at the price of \emph{removing the truth minimization} for that atom.
This can be done, for instance, by adding the excluded middle axiom (\EM) for atom $a$.
In this way, infinite traces satisfying $\alwaysF \eventuallyF a \wedge \alwaysF (a \vee \neg a)$ are those that contain an infinite number of $a$'s.
In fact, if we add the excluded middle axiom for all atoms, \TEL\ collapses into \LTL, as stated below.
\begin{proposition}[\citeNP{agcadipevi13a,cakascsc18a}]\label{prop:em}
  Let $\Gamma$ be a temporal theory over \PV\ and (EM) be the set of all excluded middle axioms for all atoms in \PV.

  Then, $\TSM(\ \Gamma \cup {\rm(EM)}\ )=\LTL(\Gamma)$.\qed
\end{proposition}

\subsection{Relation to ASP}
\label{subsec:kamp}

As we have seen in the above examples,
there seems to be a connection between temporal formulas over propositional atoms like $a$ and
logic programs with atoms for (monadic) predicates, viz.\ \texttt{a(T)},
with an integer argument \texttt{T} representing a time point.
This connection is strongly related to Kamp's well-known theorem~\cite{kamp68a} that
allows for translating \LTL\ into Monadic First-Order Logic with a linear order $<$ relation, \MFO\ for short.
In this section,
we show how Kamp's translation is also applicable to \TEL,
so that the latter can also be reduced to (Monadic) Quantified Equilibrium Logic,
which essentially covers ASP for predicates with one argument.
This connection reinforces the adequacy of \TEL\ as a suitable temporal extension of ASP.

We begin by revisiting the definition of \emph{Quantified Equilibrium Logic} or QEL~\citeN{peaval06a}.
This logic allows for first-order logic programs in ASP to be partially simplified before grounding.
Moreover, its monotonic basis, \emph{Quantified Here-and-There} (\QHT), can be used to check the property of strong equivalence,
analogous to \HT\ in the propositional case.
The definition of \QHT\ is based on a first-order language denoted by $\tuple{C,F,P}$,
where $C$, $F$ and $P$ are three disjoint sets representing constants, functions and predicates, respectively.
Given a set of constants $D$, we define:
\begin{itemize}
\item $\Terms(D,F)$ as the set of all ground terms that can be built with functions in $F$ and constants in $D$, and
\item $\Atoms(D,P)$ as the set of all ground atomic sentences that can be formed with  predicates in $P$ and constants in $D$.
\end{itemize}
In its most general version,
\QHT\ allows for different domains in the here and there worlds
and the interpretation of equality can also be varied in each world.
In this paper, though, we use the most common option when applied to ASP, that is,
\QHT\ with so-called \emph{static domains and decidable equality},
dealing with a common universe and a fixed interpretation of equality.
In this setting, a \QHT-interpretation is a tuple $\mathcal{M} = \tuple{\left(D,\sigma\right),H,T}$ such that:
\begin{itemize}
\item  $D$ is a (possibly infinite) non-empty set of constant names identifying each element in the universe. For simplicity, we use the same name for the constant and the universe element.
\item $\sigma: \Terms(C \cup D,F) \rightarrow D$ is a mapping from ground terms to elements of $D$
  satisfying $\sigma(d) = d$ for all $d\in D$ and
  structural recursion $\sigma(f(t_1,\dots,t_n))=\sigma(f(\sigma(t_1),\dots,\sigma(t_n)))$.
\item $H$ and $T$ are sets of atomic sentences satisfying $H \subseteq T \subseteq \Atoms(D,P)$.
\end{itemize}

Given two \QHT-interpretations, $\mathcal{M} = \tuple{(D,\sigma),H,T}$ and $\mathcal{M'} = \tuple{(D',\sigma'),H',T'}$,
we say that $\mathcal{M} \le \mathcal{M'}$ iff $D=D'$, $\sigma = \sigma'$, $T=T'$ and $H\subseteq H'$.
If, additionally, $H\subset H'$ we say that the relation is strict and denote it by $\mathcal{M} < \mathcal{M'}$.
\begin{definition}[\QHT-satisfaction;~\citeNP{peaval06a}]
  A \QHT-interpretation $\mathcal{M}=\tuple{(D,\sigma),H,T}$ satisfies a first-order formula $\alpha$,
  written $\mathcal{M}\models \alpha$, if the following conditions hold:

\begin{itemize}
\item $\mathcal{M} \models \top$ and $\mathcal{M} \not \models \bot$
\item $\mathcal{M}  \models p(\tau_1, \cdots, \tau_n)$ iff $p(\sigma(\tau_1), \cdots,\sigma(\tau_n)) \in H$
\item $\mathcal{M}  \models \tau=\tau'$ iff $ \sigma(\tau) = \sigma(\tau')$
\item $\mathcal{M}  \models \varphi \wedge \psi $ iff $\mathcal{M}  \models\varphi \hbox{ and } \mathcal{M}  \models\psi$
\item $\mathcal{M}  \models \varphi \vee \psi $ iff $\mathcal{M}  \models\varphi \hbox{ or } \mathcal{M}  \models\psi$
\item $\mathcal{M}  \models \varphi \rightarrow \psi $ iff $\tuple{(D,\sigma), X, T}  \not \models \varphi \hbox{ or } \tuple{(D,\sigma), X, T} \models\psi$, for $X \in \{H,T\}$
\item $\mathcal{M}  \models \forall\, x\ \varphi(x) $ iff $\mathcal{M} \models \varphi(d), \hbox{ for all } d \in D$
\item $\mathcal{M}  \models \exists\, x\ \varphi(x) $ iff $\mathcal{M} \models \varphi(d), \hbox{ for some } d \in D$ \qed
\end{itemize}
\end{definition}

Equilibrium models for first-order theories are defined as follows.
\begin{definition}[Quantified Equilibrium Model;~\citeNP{peaval06a}]
Let $\varphi$ be a first-order formula. A total \QHT-interpretation $\mathcal{M}=\tuple{(D,\sigma),T,T}$ is a first-order equilibrium model of $\varphi$ if $\mathcal{M} \models \varphi$ and there is no model $\mathcal{M'} < \mathcal{M}$ of $\varphi$.
\qed
\end{definition}

We now focus on a particular fragment of \QHT, called \MHT, by imposing the following restrictions:
\begin{enumerate}
\item $C=\intervo{0}{\lambda}$ where $\lambda \in \Nat$ or $\lambda=\omega$ and $D=\{\u\} \cup C$ where $\u$ stands for ``undefined.''
\item We only allow for (unary) functions ``$+1$'' and ``$-1$'' with the expected meaning:
\begin{eqnarray*}
\sigma(\tau+1) \eqdef \left\{
\begin{array}{cl}
\sigma(\tau)+1 & \text{if } \sigma(\tau)\neq \u \text{ and }  \sigma(\tau)+1<\lambda\\
\u & \text{otherwise}
\end{array}
\right.\\
\sigma(\tau-1) \eqdef \left\{
\begin{array}{cl}
\sigma(\tau)-1 & \text{if } \sigma(\tau)\neq \u \text{ and }  \sigma(\tau)-1\geq 0\\
\u & \text{otherwise}
\end{array}
\right.
\end{eqnarray*}
\item All predicates are unary, except binary predicates $=$ and $<$, interpreted as:
\begin{enumerate}
\item $\mathcal{M} \models \tau = \tau'$ if $\sigma(\tau)=\sigma(\tau')\neq \u$.
\item $\mathcal{M} \models \tau < \tau'$ if $\sigma(\tau)<\sigma(\tau')$ and both $\sigma(\tau)\neq \u$ and $\sigma(\tau')\neq \u$.
\end{enumerate}
\end{enumerate}
Note that the interpretation for equality requires now that both terms are different from $\u$.
We define the abbreviation $x \leq y$ as $x < y \vee x=y$.
Given these restrictions, we can simply represent an \MHT\ interpretation as $\mathcal{M}=\tuple{\lambda,H,T}$.
Moreover, it is easy to see that we can establish a one-to-one mapping between the latter and an \HT-trace $\M=\tuple{\H,\T}$ with $\lambda=|\M|$ so that $H=\{a(i) \mid a \in H_i, i \in \intervo{0}{\lambda}\}$ and $T=\{a(i) \mid a \in T_i, i \in \intervo{0}{\lambda}\}$.
When this happens, we say that $\M$ and $\mathcal{M}$ are  \emph{corresponding} interpretations.
\begin{example}\label{ex:mht}
The \HT-trace $\tuple{\H,\T}$ with $\H=\{a\} \cdot \emptyset \cdot \{b\}$ and $\T=\{a,b\} \cdot \{a\} \cdot \{b\}$ corresponds to the \MHT\ interpretation $\tuple{3,H,T}$ where $H=\{a(0),b(2)\}$ and $T=\{a(0),b(0),a(1),b(2)\}$.
\end{example}

We proceed to adapt now Kamp's translation to our setting in the following way.
\begin{definition}[Kamp's translation] \label{tel:trans:qht}
  Let $\varphi$ be a temporal formula over \PV.
  Kamp's translation of $\varphi$ for some time point $k\in \Nat$, denoted by $\tr{\varphi}_k$, is defined as follows:
\begin{eqnarray*}
\tr{\bot}_k & \eqdef & \bot\\
\tr{a}_k & \eqdef & a(k), \textrm{ with $a \in \PV$}\\
\tr{\alpha \otimes \beta}_k & \eqdef & \tr{\alpha}_k \otimes \tr{\beta}_k \quad \text{for any connective } \otimes \in \{\wedge,\vee,\to\}\\
\tr{\next \alpha}_k & \eqdef & \exists x \ \big(x=k+1 \wedge \tr{\alpha}_{x}\big )\\
\tr{\alpha \until \beta}_k & \eqdef & \exists x \; \big(k \le x \wedge \tr{\beta}_x \wedge \forall y\;  (k \le y \wedge y < x \rightarrow \tr{\alpha}_y) \big)\\
\tr{\alpha \release \beta}_k & \eqdef & \forall x\; \big(k \le x \rightarrow \tr{\beta}_x \vee \exists y\;  (k \le y \wedge y < x \wedge \tr{\alpha}_y)\big)\\
\tr{\alpha \while \beta}_k & \eqdef & \forall x\; \big(k \le x \wedge \forall y \; (k \leq y \wedge y<x \to \tr{\beta}_y) \to \tr{\alpha}_x \big)\\
\tr{\previous \alpha}_k & \eqdef & \exists x \ \big(x=k-1 \wedge \tr{\alpha}_{x}\big)\\
\tr{\alpha \since \beta}_k & \eqdef & \exists x \; \big(x \le k \wedge \tr{\beta}_x \wedge \forall y\;  (x < y \wedge y \le k \rightarrow \tr{\alpha}_y)\big)\\
\tr{\alpha \trigger \beta}_k & \eqdef & \forall x\; \big(x \le k \rightarrow \tr{\beta}_x \vee \exists y\;  (x < y \wedge y \le k \wedge \tr{\alpha}_y)\big)
\end{eqnarray*}\qed
\end{definition}
We now prove that, when considering the model correspondence between \MHT\ and \THT, Kamp's translation is sound.
\begin{theorem} \label{tel:th:sat}
Let $\varphi$ be a temporal formula over \PV, $\M=\tuple{\H,\T}$ a \THT-interpretation over \PV\ and
$\mathcal{M}=\tuple{(D,\sigma),H,T}$ its corresponding \MHT-interpretation.

Then, $\M,k \models \varphi$ in \THT\ iff  $\mathcal{M} \models \; \tr{\varphi}_k$ in \MHT. \qed
\end{theorem}
\begin{corollary}
A total \THT-interpretation $\M=\tuple{\T,\T}$ is a temporal equilibrium model of a temporal formula $\varphi$ iff its corresponding \MHT-interpretation $\mathcal{M}=\tuple{|\T|,T,T}$ is an equilibrium model of $\tr{\varphi}_0$.\qed
\end{corollary}

The translation of derived operators can be simplified in \MHT\ as follows:
\begin{eqnarray*}
\tr{\initially}_k & \equiv & \neg \exists x \ (x=k-1) \quad \equiv \quad k=0\\
\tr{\wprevious\alpha}_k & \equiv & \forall x \ (x=k-1 \to \tr{\alpha}_x )  \quad \equiv \quad k=0 \ \vee \ \tr{\alpha}_{k-1}\\
\tr{\eventuallyP\alpha}_k & \equiv & \exists x \ (x \leq k \wedge \tr{\alpha}_{x})\\
\tr{\alwaysP\alpha}_k & \equiv & \forall x \ (x \leq k \to \tr{\alpha}_{x})\\
\tr{\finally}_k & \equiv & \neg \exists x \ (x=k+1) \\
\tr{\wnext\alpha}_k & \equiv & \forall x \ (x=k+1 \to \tr{\alpha}_x ) \\
\tr{\eventuallyF\alpha}_k & \equiv & \exists x \ (k \leq x \wedge \tr{\alpha}_{x})\\
\tr{\alwaysF\alpha}_k & \equiv & \forall x \ (k \leq x \to \tr{\alpha}_{x})\\
\end{eqnarray*}

As an example, the translation of formula \eqref{f:defnext}, viz.\ $\alwaysF (\neg a \to \next a)$, for $k=0$ amounts to
\begin{eqnarray*}
& & \forall x \ (\ 0 \leq x \to (\neg a(x) \to \exists y \ (y=x+1 \wedge a(y)) ) \ )\\
&\equiv & \forall x \ (\ \neg a(x) \to \exists y \ (y=x+1 \wedge a(y)) \ )
\end{eqnarray*}
since in \MHT, $x \geq 0$ for any $x$.
For infinite traces, the existence of some $y=x+1$ is always guaranteed, and the formula above can be further simplified into
\begin{eqnarray*}
\forall x \ (\ \neg a(x) \to a(x+1) \ )
\end{eqnarray*}
which is just a first-order logic representation of the rule\footnote{Although variable \texttt{X} is unsafe, in a practical implementation, we would add a domain predicate
  \texttt{time(X)} to specify that \texttt{X} is a time point.}
``\texttt{a(X+1) :- not a(X)}'' in ASP.
Similarly, it is not difficult to check that the translation of \eqref{f:ys4b} amounts to:
\begin{eqnarray*}
& & \forall x \ (0 \leq x \wedge \exists y \; (y=x-1 \wedge \loaded(y)) \wedge \neg \unloaded(x) \to \loaded(x))\\
& \equiv & \forall x \ (0 < x \wedge \loaded(x-1) \wedge \neg \unloaded(x) \to \loaded(x))
\end{eqnarray*}
that corresponds in ASP to the rule:
\begin{lstlisting}[numbers=none,belowskip=2pt,aboveskip=2pt,basicstyle=\ttfamily]
 loaded(X) :- loaded(X-1), not unloaded(X), X>0.
\end{lstlisting}
 \section{Foundations of Temporal Here-and-There}\label{sec:tht}

In this section, we explore some of the fundamental properties of \THT, the monotonic basis of \TEL.
The importance of \THT\ with respect to \TEL\ is analogous to the relevance of \HT\ for Equilibrium Logic and ASP.
In particular, \THT\ is a suitable framework to study the \TEL-equivalence of two alternative representations.
In what follows, we prove that \THT-equivalence is a necessary and sufficient condition for \emph{strong equivalence}, we provide several interesting equivalences in \THT\ and we also present an alternative three-valued characterization of this logic.
Besides, we explain how \THT\ can be translated to \LTL\ adding auxiliary atoms, something that allows reducing the strong equivalence problem to \LTL-satisfiability checking.
The section concludes with some results for \THT\ for infinite traces, including properties about inter-definability of operators and an axiomatization.

\subsection{Strong Equivalence}
As happens in ASP, given that \TEL\ is a non-monotonic formalism, it may be the case that two different temporal formulas $\alpha$ and $\beta$ share the same temporal equilibrium models but behave differently when a common context $\gamma$ is added.
For this reason, it is usual to consider the notion of \emph{strong equivalence} instead.
Two temporal formulas $\alpha, \beta$ are \emph{strongly equivalent} iff
$\TEL(\alpha \wedge \gamma)= \TEL(\beta \wedge \gamma)$ for any arbitrary temporal formula $\gamma$.
As expected,
the \THT-equivalence of $\alpha \equiv \beta$ is a sufficient condition for strong equivalence,
since temporal equilibrium models are the result of a selection among \THT-models.
\begin{proposition}[\citeNP{agcadipevi13a,cakascsc18a}]\label{prop:suff}
If two temporal formulas $\alpha$ and $\beta$ satisfy $\alpha \equiv \beta$ then they are strongly equivalent.\qed
\end{proposition}
The interest of \THT\ is that equivalence in that logic is also a \emph{necessary} condition for strong equivalence,
as we prove next.
\begin{lemma}\label{lem:1}
  Let $\alpha$ and $\beta$ be two \LTL-equivalent formulas and let $\gamma$ be the theory containing a formula $\beta \to \alwaysF (a \vee \neg a)$ for each atom $a \in \PV$.

  Then, the following conditions are equivalent:
  \begin{enumerate}
  \item There exists some $\H < \T$ such that $\tuple{\H,\T} \not\models \alpha \rightarrow \beta$;
  \item \T\ is a \TS-model of $\{\beta\} \cup \gamma$ but not a \TS-model of $\{\alpha\} \cup \gamma$.
    \qed
\end{enumerate}
\end{lemma}
\begin{proposition}\label{prop:nec}
If two temporal formulas $\alpha$ and $\beta$ are strongly equivalent then $\alpha \equiv \beta$.\qed
\end{proposition}
\begin{proof}
We prove that if $\alpha \not\equiv \beta$ then there is some context theory $\Gamma$ for which $\{\alpha\} \cup \Gamma$
and $\{\beta\} \cup \Gamma$ have different \TS-models. Assume first that $\alpha$ and $\beta$ have different total models,
i.e., different \LTL-models.
Then, take the set $\Gamma=(\EM)$ of excluded middle axioms for every $a \in \PV$.
The \LTL-models of $\{\alpha\} \cup (\EM)$ and $\{\beta\} \cup (\EM)$ also differ (since $(\EM)$ is a set of \LTL\
tautologies).
But by Proposition~\ref{prop:total}, \LTL-models of these theories are exactly their \TS-models, and so, they also differ.

Suppose now that $\alpha$ and $\beta$ are \LTL-equivalent but still, $\alpha \not\equiv \beta$.
Then, there is some \THT-countermodel $\tuple{\H,\T}$ of either $(\alpha \rightarrow \beta)$ or $(\beta \rightarrow
\alpha) $, and given \LTL-equivalence of $\alpha$ and $\beta$, the countermodel is non-total, $\H < \T$.
Without loss of generality, assume $\tuple{\H,\T} \not\models \alpha \rightarrow \beta$.
By Lemma~\ref{lem:1}, taking the theory $\Gamma$ consisting of an implication $\beta \to \alwaysF (a \vee \neg a)$ for each atom $a \in \PV$, we get that \T\ is \TS-model of $\{\beta\} \cup \Gamma$ but not \TS-model of $\{\alpha\} \cup \Gamma$.\end{proof}

As a consequence of Propositions~\ref{prop:suff} and \ref{prop:nec}, we obtain the following characterization:
\begin{theorem}[Strong Equivalence characterisation]\label{th:strongeq}
Two temporal formulas $\alpha$ and $\beta$ are strongly equivalent iff $\alpha \equiv \beta$.\qed
\end{theorem}

\subsection{Some interesting properties of \THT}

\THT-equivalences have a crucial role for deciding how to rewrite a formula without caring about the possible context in which it is included.
We analyze next several useful \THT-equivalences.
The following are some De~Morgan laws satisfied by negation and other operators:
\begin{eqnarray}
\neg (\varphi \wedge \psi) & \equiv & \neg \varphi \vee \neg \psi \label{f:dm1}\\
\neg (\varphi \vee \psi) & \equiv & \neg \varphi \wedge \neg \psi \label{f:dm2}\\
\neg (\varphi \until \psi) & \equiv & \neg \varphi \release \neg \psi \label{f:dm3}\\
\neg (\varphi \release \psi) & \equiv & \neg \varphi \until \neg \psi \label{f:dm4}\\
\neg (\varphi \while \psi) & \equiv & \neg \neg \psi \until \neg \varphi \label{f:dm4b}\\
\neg (\next \varphi) & \equiv & \wnext \neg \varphi \label{f:dm5}\\
\neg (\wnext \varphi) & \equiv & \next \neg \varphi \label{f:dm6}
\end{eqnarray}

The next operators distribute over disjuction and conjunction so that:
\begin{align}
 \next( \varphi \oplus \psi) & \equiv  \next \varphi \oplus  \next \psi\\
 \wnext( \varphi \oplus \psi) & \equiv  \wnext \varphi \oplus  \wnext \psi
\end{align}
for $\oplus \in \{\vee, \wedge\}$ but for implication and temporal operators, only these distributive equivalences are valid:
\begin{eqnarray}
\wnext (\varphi \to \psi) & \equiv & \wnext \varphi  \to \wnext \psi\\
\next (\varphi \until \psi) & \equiv & \next \varphi  \until \next \psi\\
\next \eventuallyF \varphi & \equiv & \eventuallyF \next \varphi\\
\wnext (\varphi \release \psi) & \equiv & \wnext \varphi \release \wnext \psi\\
\wnext (\varphi \while \psi) & \equiv & \wnext \varphi \while \wnext \psi\\
\wnext \alwaysF \varphi & \equiv & \alwaysF \wnext \varphi
\end{eqnarray}
All these properties are symmetrically satisfied by the
past-oriented versions of the operators above.
For infinite traces, however, we obtain the additional relation:
\begin{eqnarray}
\wnext \varphi & \equiv & \next \varphi
\end{eqnarray}
so both next operators coincide. This equivalence never  holds for previous operators, $\previous$ and $\wprevious$, since there always exists an initial state.

Also, \THT\ satisfies \emph{persistence}, a characteristic property of \HT-based logics.
\begin{proposition}[Persistence;~\citeNP{agcadipevi13a,cakascsc18a}]\label{prop:persistance}
  Let $\tuple{\H,\T}$ be an \HT-trace of length $\lambda$ and $\varphi$ be a temporal formula.

  Then, for any $\rangeo{k}{0}{\lambda}$,
  if $\tuple{\H,\T}, k \models \varphi$ then $\tuple{\T,\T}, k \models \varphi$ (or, if preferred, $\T,k \models \varphi$).\qed
\end{proposition}
As a corollary, we have that
$\langle\mathbf{H},\mathbf{T}\rangle \models \neg \varphi$ iff ${\T} \not\models \varphi$ in \LTL.
All \THT\ tautologies are \LTL\ tautologies but not vice versa.
However, they coincide for some types of equivalences, as we show next.
\begin{proposition}[\citeNP{agcadipevi13a,cakascsc18a}]\label{prop:nonimpl}Let $\varphi$ and $\psi$ be temporal formulas without implications (and so, without negations either).

  Then, $\varphi \equiv \psi$ in \LTL\ iff $\varphi \equiv \psi$ in \THT.\qed
\end{proposition}

As an example, the usual inductive definition of the until operator from \LTL\
\begin{align}
  \varphi \until \psi & \equiv \psi \vee ( \varphi \wedge \next (\varphi \until \psi) )\label{f:induntil}
\end{align}
is also valid in \THT\ due to Proposition~\ref{prop:nonimpl}.
In fact, by De~Morgan laws, \LTL\ satisfies a kind of duality guaranteeing, for instance, that \eqref{f:induntil} iff
\begin{align}
  \varphi \release \psi & \equiv  \psi \wedge ( \varphi \vee \wnext (\varphi \release \psi) ) \label{f:indrelease}
\end{align}
and, by Proposition~\ref{prop:nonimpl} again, this is also a valid equivalence in \THT.

If we define all the pairs of dual connectives as follows: $\wedge/\vee$, $\top/\bot$, $\until/\release$, $\next/\wnext$, $\alwaysF/\eventuallyF$, $\since/\trigger$, $\previous/\wprevious$, $\alwaysP/\eventuallyP$, we can extend this to any formula $\varphi$ without implications and define $\delta(\varphi)$ as the result of replacing each connective by its dual operator.
Then, we get the following corollary of Proposition~\ref{prop:nonimpl}.
\begin{corollary}[Boolean Duality;~\citeNP{cakascsc18a}]\label{BDT}
  Let $\varphi$ and $\psi$ be formulas without implication.

  Then, \THT\ satisfies: $\varphi \gequiv \psi$  iff $\delta(\varphi) \gequiv \delta(\psi)$.\qed
\end{corollary}

In a similar manner, the temporal symmetry in the system can be exploited in order to switch the temporal direction of operators to conclude, for instance, that \eqref{f:induntil} iff $\varphi \since \psi \equiv \psi \vee ( \varphi \wedge \previous (\varphi \since \psi) )$.
However, this duality has some obvious limitations when we allow for infinite traces. For instance, the past has a beginning $\eventuallyP \initially \equiv \top$ but the future may have no end $\eventuallyF \finally \not\equiv \top$.
If we restrict ourselves to finite traces, we get the following result.

To this end,
let $\until/\since$, $\release/\trigger$, $\next/\previous$, $\wnext/\wprevious$, $\alwaysF/\alwaysP$, and $\eventuallyF/\eventuallyP$ denote all pairs of swapped-time connectives and let $\sigma(\varphi)$ denote the replacement in $\varphi$ of each connective by its swapped-time version.
\begin{lemma}[\citeNP{cakascsc18a}]\label{TDT}
There exists a mapping $\varrho$ on finite \HT-traces of the same length $\lambda$ such that for any $\rangeo{k}{0}{\lambda}$,
$\M,k \models \varphi$ iff $\varrho(\M),\lambda\!-\!1\!-\!k \models \sigma(\varphi)$.\qed
\end{lemma}\begin{theorem}[Temporal Duality Theorem;~\citeNP{cakascsc18a}]
  A temporal formula $\varphi$ is a \THTf-tautology iff $\sigma(\varphi)$ is a \THTf-tautology.\qed
\end{theorem}
\subsection{Three-valued Characterization of \THT}
\label{sec:3val}
As mentioned in Section~\ref{sec:tel},
\HT-traces can be seen as three-valued models where each atom can be false, assumed or proven (in a given state).
This stems from the fact that the intermediate logic of \HT\ actually corresponds to G\"odel's three-valued logic $G_3$~\cite{goedel32a}.
In particular, the \HT-satisfaction of formulas can be replaced by a three-valued function that assigns,
to each formula, a value $0$, $1$, or $2$, standing for
``false'' (it does not hold in $T$),
``assumed'' (it holds in $T$ but not in $H$) and
``proven'' (it holds in $H$, and so, in $T$, too),
respectively.
The interest of a multi-valued truth assignment $\trival{k}{\alpha}$ of a formula $\alpha$ at time point $k$ is that
\THT-equivalence $\alpha \equiv \beta$ can be reduced to a comparison of truth values such as
$\trival{k}{\alpha}=\trival{k}{\beta}$.
This has an important application when introducing an auxiliary atom $a$ to represent $\alpha$ so that we can safely replace $\alpha$ by $a$ provided that we add formulas to guarantee $\trival{k}{a}=\trival{k}{\alpha}$.
Following these ideas, we proceed with the following formal definitions.

For convenience, the evaluation of implications is captured by the following conditional operation on truth values:  $\text{imp}(x,y)=2 $ if $x\leq y$; otherwise $\text{imp}(x,y) = y$.
Given an \HT-trace $\tuple{\H,\T}$ of length $\lambda$,
we define its associated truth valuation as a function $\trival{k}{\varphi}$ that assigns
a truth value in $\{0,1,2\}$ to formula $\varphi$ at time point $k \in \intervo{0}{\lambda}$
according to the following rules:
\begin{eqnarray*}
	\trival{k}{\bot} & \eqdef & 0 \\
	\trival{k}{a} & \eqdef &
	\begin{cases}
		0 & \text{if} \ a \not\in T_k \\
		1 & \text{if} \ a \in T_k\setminus H_k \\
		2 & \text{if} \ a \in H_k
	\end{cases}\hspace{20pt} \text{for any atom } a\\
	\trival{k}{\varphi \wedge \psi} & \eqdef &
	\min(\trival{k}{\varphi},\trival{k}{\psi})\\
	\trival{k}{\varphi \vee   \psi} & \eqdef &
	\max(\trival{k}{\varphi},\trival{k}{\psi})\\
	\trival{k}{\varphi \to \psi} & \eqdef & \text{imp}(\trival{k}{\varphi},\trival{k}{\psi})\\
	\trival{k}{\previous \varphi} & \eqdef &
	\begin{cases}
		0                      & \text{if } k=0 \\
		\trival{k-1}{\varphi}  & \text{if } k>0
	\end{cases}\\
	\trival{k}{\varphi \since \psi} & \eqdef &
	\max\{\min(\trival{j}{\psi},\min\{\trival{i}{\varphi}\mid j<i\leq k\} ) \mid 0\leq j\leq k\}\\
	\trival{k}{\varphi \trigger \psi} & \eqdef &
	\min\{\max(\trival{j}{\psi},\max\{\trival{i}{\varphi}\mid j<i\leq k\} ) \mid 0\leq j\leq k\}\\
	\trival{k}{\next \varphi} & \eqdef &
	\begin{cases}
		0                      & \text{if } k+1=\lambda \ (\neq \omega)\\
		\trival{k+1}{\varphi}  & \text{if } k+1<\lambda
	\end{cases}\\
	\trival{k}{\varphi \until \psi} & \eqdef &
	\max\{\min(\trival{j}{\psi},\min\{\trival{i}{\varphi}\mid k\leq i < j\} )
	\mid k\leq j< \lambda\}\\
	\trival{k}{\varphi \release \psi} & \eqdef &
	\min\{\max(\trival{j}{\psi},\max\{\trival{i}{\varphi}\mid k\leq i < j\} )
	\mid k\leq j< \lambda\}\\
    \trival{k}{\varphi \while \psi} & \eqdef &
	\min\{\text{imp}(\min\{\trival{i}{\psi} \mid k\leq i<j\},\trival{j}{\varphi})
	\mid k\leq j< \lambda\}
\end{eqnarray*}

The valuation of derived operators can be easily concluded from their definitions:

\begin{eqnarray*}
	\trival{k}{\top} & = & 2\\
	\trival{k}{\neg \varphi} & \eqdef &
	\begin{cases}
		2  & \text{if } \trival{k}{\varphi}=0 \\
		0  & \text{otherwise}
	\end{cases}\\
	\trival{k}{\initially} & \eqdef &
	\begin{cases}
		2 & \text{if } k=0 \\
		0 & \text{if } k>0
	\end{cases}\\
	\trival{k}{\wprevious \varphi} & \eqdef &
	\begin{cases}
		2                      & \text{if } k=0 \\
		\trival{k-1}{\varphi}  & \text{if } k>0
	\end{cases}\\
	\trival{k}{\alwaysP \varphi} & = &
	\min \{\trival{i}{\varphi} \mid 0\leq i\leq k\}\\
	\trival{k}{\eventuallyP \varphi} & = &
	\max \{\trival{i}{\varphi} \mid 0\leq i\leq k\}\\
	\trival{k}{\finally} & \eqdef &
	\begin{cases}
		2 & \text{if } k=\lambda \neq \omega \\
		0 & \text{if } k<\lambda
	\end{cases}\\
	\trival{k}{\wnext \varphi} & \eqdef &
	\begin{cases}
		2                      & \text{if } k+1=\lambda \ (\neq \omega) \\
		\trival{k+1}{\varphi}  & \text{if } k+1< \lambda
	\end{cases}\\\trival{k}{\alwaysF \varphi} & = &
	\min \{\trival{i}{\varphi} \mid k\leq i<\lambda \}\\
	\trival{k}{\eventuallyF \varphi} & = &
	\max \{\trival{i}{\varphi} \mid k\leq i< \lambda\}
\end{eqnarray*}

For the restriction to \THTf, it suffices to impose the condition $\lambda \neq \omega$.
\begin{proposition}
  Let $\tuple{\H,\T}$ be a \HT-trace of length $\lambda$, ${\bm m}$ its associated valuation and
  $k\in\intervo{0}{\lambda}$.

  Then, we have that
  \begin{itemize}
  \item $\tuple{\H,\T},k \models \varphi$ iff $\trival{k}{\varphi}=2$
  \item $\tuple{\T,\T},k \models \varphi$ iff $\trival{k}{\varphi}\neq 0$\qed
  \end{itemize}
\end{proposition}

As an illustration, consider the \HT-trace from Example~\ref{ex:mht}.
This trace corresponds to the three-valued assignment
\[
  \trival{0}{a}=2, \trival{0}{b}=1, \trival{1}{a}=1, \trival{1}{b}=0, \trival{2}{a}=0\text{ and }\trival{2}{b}=2.
\]
The formula $\eventuallyF b$ is proven at $k=0$, that is, $\trival{0}{\eventuallyF b}=2$, because $\trival{0}{b}=2$ and $2$ is the maximum possible value.
The formula $\alwaysF b$ is false at $k=0$, $\trival{0}{\alwaysF b}=0$ because $\trival{1}{b}=0$ and $0$ is the minimum value.
On the other hand, the formula $\alwaysF (a \vee b)$ is just assumed but not proven at $0$.
To see why, note that the value $\trival{k}{a \vee b}$ is the maximum of both disjuncts and that, in our trace, this is equal to $2$ for time points $k=0$ and $k=2$ but is equal to $1$ for $k=1$.
Since $\alwaysF (a \vee b)$ takes the minimum of all these values for $k\geq 0$,
we get $\trival{0}{\alwaysF (a \vee b)}=1$.
 \subsection{From \THT\ to \LTL}
\label{sec:tht2ltl}

Propositions~\ref{prop:total} and \ref{prop:emtotal} tell us that the addition of excluded middle axioms (\EM)
allows for an easy encoding of \LTL\ into \THT.
In fact, Proposition~\ref{prop:em} even guarantees that the addition of these axioms
makes also the non-monotonic formalism of \TEL\ collapse into \LTL.
An interesting question is whether the other direction is possible, namely, whether \THT\ can be encoded into \LTL.
In fact, this can be done by adapting the encoding of (non-temporal) \HT\ into classical propositional logic, as presented, for instance, by~\citeN{petowo01a}.

Given a propositional signature $\PV$, let us define a new propositional signature in \LTL\ by
\(
\PV^*=\PV \cup \{a' \ | \ a \in \PV\}
\).
For any temporal formula $\varphi$, we define its translation $\varphi^*$ as follows:
\begin{enumerate}
 \item $\bot^* \eqdef \bot$
 \item $a^* \eqdef a'$ for any $a\in \PV$
 \item $(\odot \varphi)^* \eqdef \odot \varphi^*$,
 if  $\odot \in \{\next, \previous \}$
 \item $(\varphi \odot \psi)^* \eqdef \varphi^* \odot \psi^*$,
  when $\odot \in \{\wedge, \vee, \until, \release, \since, \trigger \}$
 \item $(\varphi \rightarrow \psi)^*\eqdef
  (\varphi \rightarrow \psi)\wedge(\varphi^* \rightarrow \psi^*)$
\end{enumerate}

We associate with any \HT-trace $\tuple{\H,\T}$ of length $\lambda$
the trace $\T^*$ in \LTL\ defined as the sequence of sets of atoms
\(
T^*_i=T_i \cup \{a' \mid a\in H_i\}
\)
for any $i \in \intervo{0}{\lambda}$.
Informally speaking, $\T^*$ considers a new primed atom $a'$ per each $a \in \PV$ in the original signature.
In the trace, the primed atom $a'$ represents the fact that $a$ occurs at some point in the \H\ components,
whereas the original symbol $a$ is used to represent an atom in \T.
As an \HT-trace satisfies $H_i \subseteq T_i$ by construction, we may have traces that do not correspond to any \HT-trace,
since the set of primed atoms must be a subset of the non-primed ones.
To force this structural condition on traces, we can just add the axiom:
\begin{eqnarray}
\nec (a' \rightarrow a) \label{f:ax} \text{ for any atom }a \in \PV.
\end{eqnarray}
By including this axiom, we obtain a one-to-one correspondence between \HT-traces and traces for the extended signature.
In particular, if we take any arbitrary trace $\T^*$ for signature $\PV^*$ satisfying $\eqref{f:ax}$,
we can now define its corresponding \HT-trace $\tuple{\H,\T}$ as
$T_i \eqdef T^*_i \cap \PV$ and $H_i\eqdef \{a \mid a' \in \T^*_i\}$.
\begin{theorem}\label{th:trans}\label{translation}
  Let $\tuple{\H,\T}$ be an \HT-trace of length $\lambda$ for alphabet \PV{}, $k \in \intervo{0}{\lambda}$ and $\varphi$ a temporal formula over \PV.
  Then, $\tuple{\H,\T},k \models \varphi$ iff $\T^*,k \models \varphi^*  \wedge \eqref{f:ax}$ in \LTL.\qed
\end{theorem}
\begin{example}\label{ex2}
  Take the \HT-trace $\tuple{\H,\T}$ from Example~\ref{ex:mht}.
  Its corresponding trace is $\T^*= \{a,b,a'\} \cdot \{a\} \cdot \{b,b'\}$.
  Axiom \eqref{f:ax} in this case is the formula:
  \begin{eqnarray}
    \alwaysF(a' \to a) \wedge \alwaysF (b' \to b) \label{f:ex2.1}
  \end{eqnarray}
  Now, for instance, given the formula:
  \begin{eqnarray}
    \alwaysF (\neg b \to a) \label{f:ex2.2}
  \end{eqnarray}
  its translation $\eqref{f:ex2.2}^*$ corresponds to:
  \begin{eqnarray}
    & & \alwaysF \big( (\neg b \to a) \wedge ((\neg b)^* \to a^*) \big) \nonumber\\
    & = & \alwaysF \big( (\neg b \to a) \wedge (\neg b \wedge \neg b' \to a') \big)\nonumber\\
    & \equiv & \alwaysF (\neg b \to a) \wedge \alwaysF (\neg b \to a') \label{f:starex}
  \end{eqnarray}
  It is easy to see that $\T^*$ satisfies $\eqref{f:ex2.2}^* \wedge \eqref{f:ex2.1}$ in \LTL.
\end{example}

To conclude this comparison to \LTL, we include a brief comment on the complexity of \THT.
As shown above,
\LTL-satisfaction of a temporal formula $\varphi$ can be reduced to \THT-satisfaction with the simple addition of (\EM) axioms.
On the other hand, reducing \THT-satisfaction of $\varphi$ to \LTL-satisfaction consists of adding axiom~\eqref{f:ax} and
applying translation $\varphi^*$, which can be easily proved to be quadratic in the worse case.
As a result, \LTL- and \THT-satisfaction share the same complexity, regardless of the trace length under consideration.
In particular, it is well-known that \LTLo-satisfiability is \textsc{PSpace}-complete~\cite{siscla85a}.
Also, \LTLf-satisfiability was also proved to be \textsc{PSpace}-complete by~\citeN{giavar13a}.
As a result, both \THTo- and \THTf-satisfiability have the same complexity, too.
 \subsection{Automata-based Checking of Strong Equivalence}\label{sec:automSE}

The translation from \THT\ to \LTL\ can be used to reduce strong equivalence of temporal theories to \LTL-satisfiability checking.
This feature was exploited by the tool \abstem~\cite{cabdie14a} that provides us with several functionalities for temporal theories under \TELo{} semantics.

For checking whether two temporal formulas $\varphi$ and $\psi$ are \LTL-equivalent, \abstem\ checks the validity of $\varphi \leftrightarrow \psi$, which is reduced to the unsatisfiability of $\neg \left(\varphi \leftrightarrow \psi\right)$.
This is done by translating $\neg \left(\varphi \leftrightarrow \psi\right)$ into a B\"uchi automaton $\mathfrak{A}$ by means of the \LTL-model checker \texttt{SPoT}~\cite{dulefamirexu16a}.
As explained in Section~\ref{sec:tel2aut}, the language accepted by $\mathfrak{A}$ corresponds to the \LTL-models of $\neg \left(\varphi\leftrightarrow \psi\right)$.
Therefore, if the accepting language of $\mathfrak{A}$ is empty, it means that $\neg \left(\varphi\leftrightarrow \psi\right)$ is not satisfiable and  $\varphi$ and $\psi$ are \LTL-equivalent.
Otherwise, any word accepted by $\mathfrak{A}$ can be seen as a counterexample of $\varphi \leftrightarrow \psi$.

The method used by~\citeN{cabdie14a} to determine whether $\varphi$ and $\psi$ are strongly equivalent consists in checking the validity of both $\varphi \rightarrow \psi$  and $\psi \rightarrow \varphi$ in \THT.
Let us consider, without loss of generality, that we want to check the validity of $\varphi \rightarrow \psi$.
This is equivalent to checking the satisfiability of
\begin{equation}
\neg \left(   \left(\mbox{$\bigwedge\limits_{a \in \PV}$}  \nec\left(a' \rightarrow a\right)\right) \rightarrow \left( \varphi \rightarrow \psi \right)^*\right)\label{eq:abstem}
\end{equation}
in \LTL\ and it can be done by using the procedure explained above.
\abstem\ obtains a B\"uchi automaton $\mathfrak{A}_{\eqref{eq:abstem}}$ that accepts the \LTL-models of~\eqref{eq:abstem}.
If~\eqref{eq:abstem} is \LTL-satisfiable then \abstem\ filters all the variables of the type $a'$ from the accepting language of $\mathfrak{A}_{\eqref{eq:abstem}}$ (as explained in Section~\ref{sec:tel2aut}).
The resulting automaton, denoted by $h(\mathfrak{A}_{\eqref{eq:abstem}})$, captures the \TEL-models of $\varphi \wedge \gamma$ which are not models of $\psi \wedge \gamma$, where
\[\textstyle
  \gamma \eqdef \bigwedge_{a \in \PV} \psi \rightarrow \alwaysF(a \vee \neg a)
\]
for every atom $a$ in the signature (assuming a finite alphabet).

With minor modifications, this technique could be adapted to the \TELf\ case.
The only required change is the target automata:
for \TELo\ we need to use B\"uchi automata since we are dealing with infinite computations
while, in the finite case, we need to use automata on finite words.
The rest of the algorithm would remain the same.

\subsection{Definability of Temporal Operators in \THT}

Propositional connectives are not inter-definable in intuitionistic propositional logic~\cite{dalen86a}.
However, when it comes to propositional \HT, some inter-definability results can be obtained.
For instance, the \HT-valid formula~\cite{lukasiewicz41a}
\begin{displaymath}
p \vee q \leftrightarrow \left(\left(p \rightarrow q\right)\rightarrow q\right) \wedge \left(\left(q \rightarrow p\right)\rightarrow p\right).
\end{displaymath}
allows us to determine that disjunction in \HT\ can be defined in terms of the remaining propositional connectives.
Unfortunately, this is the only definable operator~\cite{baldie16,agcapepevi15a}.
When considering \QHT, the equivalence
\begin{displaymath}
\exists x\; p(x) \leftrightarrow \left(\forall y\; \forall x\; \left(p(x) \rightarrow p(y)\right)\rightarrow p(y)\right)
\end{displaymath}
is valid, so existential quantifiers can be reformulated in terms of the remaining connectives as well~\cite{mints10a}.
Since \THT\ can be seen as a syntactic subclass of \QHT, a similar result could be expected.
In fact, \citeN{babodife19a} have proved that the connectives $\eventuallyF p$ can be defined in terms of a $\until$-free formula, as stated in the following lemma.
\begin{lemma}[\citeNP{babodife19a}]\label{lem:def:dia}
  The formulas $\eventuallyF p$ and
  \begin{equation}
    (\alwaysF (p\to \alwaysF (p\vee \neg p))\wedge \alwaysF (\next \alwaysF (p\vee \neg p) \to p\vee \neg p\vee \next \alwaysF \neg p))\to (\alwaysF(p\vee \neg p)\wedge \neg \alwaysF \neg p) \label{eq:def:eventually}
  \end{equation}
  are equivalent in \THT.
\end{lemma}
Broadly speaking the equivalence in Lemma~\ref{lem:def:dia} captures the three different ways of satisfying $\eventuallyF p$ on any \THT-model $\tuple{\H,\T}$ at $i \ge 0$:
\begin{enumerate}
\item $\tuple{\H,\T}, i \models \alwaysF \left(p \vee \neg p \right)$ holds, and in this case $\tuple{\H,\T}$ behaves
  classically after $i$; at least for formulas whose only variable is $p$.
  In this case $\eventuallyF p$ would behave as $\neg \alwaysF \neg p$.
\item If $\tuple{\H,\T}, i \not \models \alwaysF \left(p \vee \neg p \right)$, then $\tuple{\H,\T}$ does not behave classically after $i$.
  Therefore, for some $j \ge i$, $\tuple{\H,\T}, i \not \models p \vee \neg p$.
  In order to make the right part of the equivalence true, either $\tuple{\H,\T}, i \not\models \alwaysF \left(p \rightarrow \alwaysF \left( p \vee \neg p \right)\right)$ or $\tuple{\H,\T}, i \not\models \alwaysF (\next \alwaysF (p\vee \neg p) \to p\vee \neg p\vee \next \alwaysF \neg p)) $ holds.
  The former formula fails when there is $k \ge i$ such that $\tuple{\H,\T}, j \models p$ (so $\tuple{\H,\T}, i \models
  \eventuallyF p$) and $\tuple{\H,\T}$ does not behave classically after $k$.
\item Similarly, $\alwaysF (\next \alwaysF (p\vee \neg p) \to p\vee \neg p\vee \next \alwaysF \neg p))$  fails exactly
  where there is $k \ge i$ satisfying $p$ but $\tuple{\H,\T}$ behaves classically after $k$.
  In other words $\tuple{\H,\T}$ falsifies $p \vee \neg p$ only for $i < k$.
  In this case, $\next \alwaysF (p \vee  \neg p) \rightarrow p \vee \neg p \vee \next \alwaysF \neg p$ is falsified exactly at the greatest such $k$.
\end{enumerate}
From this equivalence the following corollary follows.
\begin{corollary}[\citeNP{babodife19a}]\label{def:until:tht}
	$p\until q$ is $\until$-free definable by using the equivalence $q\until p \leftrightarrow (p \release (q\vee p))\wedge\eqref{eq:def:eventually}$.\end{corollary}
Lemma~\ref{lem:def:dia} and Corollary~\ref{def:until:tht} were proved for infinite traces, but in fact, these equivalences also hold for finite traces \THTf, as stated below.
\begin{corollary}\label{cor:definableTHTf}
The equivalences~\eqref{eq:def:eventually} and $q\until p \leftrightarrow (p \release (q\vee p))\wedge\eqref{eq:def:eventually}$ are valid in \THTf{}.
\end{corollary}
\citeN{babodife19a} also give a negative answer for the definability of $\alwaysF p$ or $p \release q$ as a $\release$-free formula. To do so, they consider the following \THT\ model: let $n \in \mathbb{N}$, we define the \THT{} models $\tuple{\H,\T}_n$ as follows:
\begin{enumerate*}[label=(\arabic*)]
	\item $H_i = \lbrace p \rbrace$ for all $0 \le i < n+1$;
	\item $H_{n+1} = \emptyset$;
	\item $T_i = \lbrace p \rbrace$, for all $0 \le i \le n+1$;
	\item\label{definability:model1} $T_i = T_{i \text{ mod } n + 2}$ and $H_i = H_{i \text{ mod } n + 2}$,  for all $i > n+1$;
\end{enumerate*}
It is evident that $\tuple{\H,\T}, 0 \not \models \alwaysF p$ and $\tuple{\T,\T}, 0 \models \alwaysF p$.
By using a \emph{bisimulation} argument, \citeN{babodife19a} proved the following lemma.

\begin{lemma}[\citeNP{babodife19a}]\label{lemma:bisimulation}
	Let us consider the \THT\ model defined above. The models $\tuple{\H,\T}_n$ and $\tuple{\T,\T}_n$ satisfy the same $\release$-free formulas with at most $n$ connectives at time $0$.
\end{lemma}
This lemma is proved by showing that $\tuple{\H,\T}_n$ and $\tuple{\T,\T}_n$ are \emph{bisimilar} at time $0$ with respect to the temporal language excluding the $\release$ connective. This means both models cannot be distinguish with respect to the set of $\release$-free formulas.

As a consequence, let us assume by contradiction that $\alwaysF p$ is definable in terms of a $\release$-free formula $\psi$ and let us fix $n$ as the number of connectives of $\psi$.
Since $\tuple{\H,\T}_n,0 \not \models \alwaysF p$ then $\tuple{\H,\T}_n,0 \not \models \psi$.
Since $\tuple{\T,\T}_n,0 \models \alwaysF p$ then $\tuple{\T,\T}_n,0 \models \psi$, which contradicts Lemma~\ref{lemma:bisimulation}.
As a corollary, if we could define the formula $p \release q$ in terms of a $\release$-free formula, we could also define  $\alwaysF p$, since the latter is equivalent to $p \release \bot$.

If we remove condition~\ref{definability:model1} from the model $\tuple{\H,\T}_n$ we obtain a finite model with $\lambda = n+2$ satisfying Lemma~\ref{lemma:bisimulation}.
Therefore, we can reformulate the previous discussion in terms of \THTf{}.

Another interesting observation regarding expressiveness of operators is the presence of the newly introduced temporal operator \mbox{$\varphi \while \psi$} read as ``repeat $\varphi$ while $\psi$.''
Another interesting observation regarding expressiveness of operators is the presence of the newly introduced temporal
operator \mbox{$\varphi \while \psi$}, read as ``repeat $\varphi$ while $\psi$.''
This operator is not usual in \LTL\ because, in that logic, it can be easily defined in terms of $\until$ or $\release$:

\begin{proposition}
In LTL, we have the following equivalences:
\begin{eqnarray*}
\alpha \ \while \ \beta & \equiv & \neg ( \beta \ \until \ \neg \alpha)              \quad \equiv\quad \neg \beta \ \release \ \alpha\\
\alpha \ \until \ \beta & \equiv & \neg \neg (\alpha \ \until \ \neg \neg \beta) \\
                        & \equiv & \neg (\neg \beta \ \while \ \alpha)
\end{eqnarray*}\qed
\end{proposition}
In \THT, however, these equivalences do not hold.
Unlike $\until$ and $\release$, operator $\while$ has an implicational nature.
As a consequence, there is no clear way to represent $\while$ in terms of the other operators: we conjecture that this may be, in fact, a fundamental operator in \THT.\label{page:whileconjecture}
The following equivalences describe these three operators, $\until$, $\release$ and~$\while$, by an inductive unfolding.
\begin{proposition}\label{prop:WRU}
In \THT, we have the following equivalences:
\begin{eqnarray*}
\alpha \until \beta & \equiv &  \beta \vee (\alpha \wedge \next (\alpha \until \beta))\\
\alpha \release \beta & \equiv &  \beta \wedge (\alpha \vee \next (\alpha \release \beta))\\
\alpha \while \beta & \equiv &  \alpha \wedge (\beta \to \next (\alpha \while \beta))
\end{eqnarray*}
\qed
\end{proposition}
The first two equivalences are well-known from \LTL\ and also preserved in \THT.
The third equivalence shows that the $\while$ operator is formed by a repeated application of implications.
This is even clearer in the following expansion.
\begin{proposition}\label{prop:unfold}
The formula $(\alpha \while \beta)$ is \THT-equivalent to the (possibly infinite) conjunction of rules
\[\textstyle
  \big( \bigwedge_{j=0}^{i-1} \next^j \beta \big)  \to  \next^i \alpha
  \qquad
  \text{ for all }i\geq 0.
\]
\qed
\end{proposition}

To see how this expression works, consider the formula $(w \while f)$, where $w$ means ``pouring water'' and $f$ means ``fire.''
Notice that, for $i=0$, the rule body becomes an empty conjunction ($\top$) and so this expands to $\top \to \next^0 w$ which is just equivalent to fact $w$.
For instance, for $i \in [0,3]$ we get the rules:
\begin{eqnarray*}
& & w \\
f & \to & \next w \\
f \wedge \next f & \to & \next^2 w \\
f \wedge \next f \wedge \next^2 f & \to & \next^3 w \\
& \vdots &
\end{eqnarray*}
That is, we start pouring water at situation $0$ and, if we have fire, we keep pouring water at $1$ and check fire again, and so on.
As we can see, the effect of each $f$ test is placed at the next situation.
Thus, the reading of $(w \while f)$ is like a procedural program ``{\bf do} pour-water {\bf while} fire''.
If we want to move the test to the same situation of its effect,
we would write instead $\big((f \to w) \while f \big)$ whose reading would be ``{\bf while} fire {\bf do} pour-water'' and whose expansion becomes:
\begin{eqnarray*}
f & \to & w \\
f \wedge \next f & \to & \next w \\
f \wedge \next f \wedge \next^2 f & \to & \next^2 w \\
\vdots
\end{eqnarray*}

The expansion of $(w \, \while f)$ as a set of rules reveals its behavior from a logic programming point of view.
For instance, a theory only containing the formula $(w \while f)$ has a unique temporal stable model (per each length $\lambda>0$) in which $w$ is only true at the initial state, whereas $f$ is always false, since there is no additional evidence about fire.
In \LTL, the formula $(w \while f)$ is equivalent to $(\neg f \release w)$, that is, $\alwaysF w \vee (w \until (w \wedge \neg f))$.
In our formalism, however, this last formula produces temporal stable models with arbitrary prefix sequences of $w$, and even the case in which $w$ holds in all the states of the trace.
In other words, water may be poured arbitrarily many times, even though fire is false all over the trace.

\subsection{Axiomatization of \THTo}\label{sec:axiomatisation}

The axiomatic system of the future fragment of \THTo~\cite{baldie16} is obtained by combining the axioms of
intuitionistic modal logic~K~\cite{simpson94a} (\IK), the Hosoi~Axiom~\cite{hosoi66a} and the future \LTLo\ axiomatization
of \citeN{goldblatt92a}.
Such an axiomatic system is described below.

\begin{enumerate}
\item \textbf{Axioms of Intuitionistic Propositional Calculus}

\item \textbf{Hosoi axiom}:\quad  \mbox{$ p\vee( p\rightarrow q)\vee\neg q$}

\item \textbf{\IK\ axioms  for $\next$ and $\wnext$}~\cite{simpson94a}:
  \begin{multicols}{2}
    \begin{enumerate}
    \item\label{ax:next:1} \mbox{$\wnext p \leftrightarrow \next p$}
    \item\label{ax:next:2} \mbox{$\wnext \left(p \rightarrow q\right)\rightarrow \left(\wnext p \rightarrow \wnext q\right)$}
    \item\label{ax:next:3} \mbox{$\wnext \left(p \rightarrow q\right)\rightarrow \left(\next p \rightarrow \next q\right)$}
    \item\label{ax:next:4} \mbox{$\next\left( p \vee q \right) \leftrightarrow \next p \vee \next q$}
    \item\label{ax:next:5} \mbox{$\left(\next p \rightarrow \wnext q \right) \rightarrow \wnext\left( p \rightarrow q \right)$}
\item\label{ax:next:7} \mbox{$\neg \next \bot$}
    \end{enumerate}
  \end{multicols}

\item \textbf{\IK\ axioms for $\alwaysF$ and $\eventuallyF$}:
  \begin{multicols}{2}
    \begin{enumerate}[start=7]
    \item\label{ax:fs:2} \mbox{$\alwaysF( p\rightarrow q)\rightarrow(\alwaysF p\rightarrow\alwaysF q)$}
    \item\label{ax:fs:3} \mbox{$\alwaysF( p\rightarrow q)\rightarrow(\eventuallyF p\rightarrow\eventuallyF q)$}
    \item\label{ax:fs:4} \mbox{$\eventuallyF( p\vee q)\rightarrow\eventuallyF p\vee\eventuallyF q$}
    \item\label{ax:fs:5} \mbox{$(\eventuallyF p\rightarrow\alwaysF q)\rightarrow\alwaysF( p\rightarrow q)$}
    \item\label{ax:fs:1} \mbox{$\neg \eventuallyF \bot$}
    \end{enumerate}
  \end{multicols}

\item \textbf{Axioms combining $\next$, $\wnext$, $\alwaysF$ and $\eventuallyF$}:
  \begin{multicols}{2}
    \begin{enumerate}[start=12]
    \item\label{ax:comb:1}\mbox{$\alwaysF p\rightarrow p \wedge \wnext \alwaysF  p$}
    \item\label{ax:comb:2}\mbox{$ p \vee \next \eventuallyF  p\rightarrow\eventuallyF p$}
    \end{enumerate}
  \end{multicols}

\item \textbf{Induction}:
  \begin{multicols}{2}
  \begin{enumerate}[start=14]
  \item\label{ax:ind:1} $\dfrac{  p \rightarrow \wnext  p }{  p \rightarrow \alwaysF  p }$
  \item\label{ax:ind:2} $\dfrac{ \next  p \rightarrow  p }{ \eventuallyF  p \rightarrow  p }$
  \end{enumerate}
\end{multicols}

\item \textbf{Axioms for $\until$ and $\release$}:
  \begin{multicols}{2}
    \begin{enumerate}[start=16]
    \item\label{ax:until:1} $p \until q \rightarrow \eventuallyF q$
    \item\label{ax:until:2} $p \until q \leftrightarrow \left(q \vee \left( p \wedge \next \left(p \until q \right)\right)\right)$
    \item\label{ax:release:1} $\alwaysF q \rightarrow p \release q$
    \item\label{ax:release:2} $p \release q \leftrightarrow \left(q \wedge \left( p \vee \wnext \left(p \release q \right)\right)\right)$
    \end{enumerate}
  \end{multicols}

\item \manuallabel{ax:mp}{Modus ponens}\textbf{Modus ponens}:\quad $\dfrac{ p \rightarrow q,\ p }{ q }$
\item \manuallabel{ax:nec}{Necessitation}\textbf{Necessitation}:\quad $\dfrac{p}{ \wnext p }$

\end{enumerate}

Due to the intuitionistic basis of \THTo\ all the axioms need to be duplicated in order to cope with the potential non-definability of modal operators~\cite{simpson94a}.
It should be also noticed that the induction axiom is replaced by the corresponding inference rule. These inference rules can be derived from the axioms and vice versa.
The corresponding proof of soundness and completeness is achieved by adapting Goldblatt's proof for classical \LTLo~\cite{goldblatt92a} to the here-and-there case:
soundness is proved by checking that all axioms are valid in \THTo\ and that the inference rules preserve validity.
The proof of completeness relies on variations of well-known methods for proving completeness:
a canonical model construction for bi-relational models~\cite{simpson94a}, filtration and unwinding~\cite{baldie16}.

From the axiomatic point of view, the main difference between \THTo\ and \THTf\ relies on Axiom~\eqref{ax:next:1}, which guarantees that every time instant has a successor.
From the point of view of the completeness proof, this axiom makes possible the unwinding of the filtrated model by~\citeN{baldie16}.
Since Axiom~\eqref{ax:next:1} is not valid in \THTf\ because it fails at the end of the trace, this unwinding may not be directly applied to obtain a sound and complete axiomatization of \THTf.

  \section{From \TEL\ to automata}\label{sec:tel2aut}

The relation between traditional temporal logics and automata is well-known and has been thoroughly studied in the
literature~\cite{degola16a}.
For infinite traces, \LTLo\ unsatisfiability can be reduced to checking the emptiness of a language with a given automaton that accepts \emph{infinite words}.
For this, we can use, for instance, the well-known translation of~\citeN{gepevawo95a} mapping a temporal formula $\varphi$ into a \emph{B\"uchi automaton}~\cite{buechi62a} that accepts the $\omega$-regular language corresponding to the (infinite trace) models of $\varphi$.
When we jump to finite traces, the type of automata required is simpler, as they just need to cover languages with words of finite length.
Thus, in the works by~\citeN{zhpuva19a} and~\citeN{giavar13a}, translations from \LTLf\ into different types of automata such as {\em Non-deterministic finite automata} (NFA) and {\em Alternating Automata on Words} or AFW~\cite{chkost81a} have been proposed.

The use of automata for deciding \THT-satisfiability can naturally follow the same steps, provided that the translation
from Section~\ref{sec:tht2ltl} allows us to reduce the problem to standard \LTL-satisfiability, regardless of the trace length.
This fact was used by~\citeN{cabdie14a} to decide strong equivalence of two temporal formulas $\alpha$ and $\beta$ in \TELo\ as follows.
By Theorem~\ref{th:strongeq}, strong equivalence amounts to \THTo-equivalence, that is, checking the validity of $\varphi:=\alpha \leftrightarrow \beta$ in \THTo.
But then, by Theorem~\ref{th:trans}, this amounts to checking the unsatisfiability of $\eqref{f:ax} \wedge \neg (\varphi^*)$ in \LTLo\ and the latter is then done by the standard B\"uchi-automata construction method.
Similar steps can be followed for \THTf, replacing B\"uchi-automata for \LTLo\ by NFA or AFW for \LTLf.

The use of automata for a \emph{non-monotonic} temporal formalism like \TEL\ becomes a more involved process, since temporal equilibrium models are obtained by a model selection (or minimization) process.
The complexity of \TELo, for instance, suggests that there is no efficient method to reduce \TELo- to \LTLo-satisfiability.
\begin{theorem}[\citeNP{bozpea15a}]
The problem of deciding whether a temporal formula has some (infinite) temporal equilibrium model in \TELo\ is \textsc{ExpSpace}-complete.\qed
\end{theorem}
A possible way of overcoming this difficulty is to observe that Definition~\ref{def:tem} (of temporal equilibrium
models) involves a kind of \emph{quantification}.
That is, we must find a total \HT-model $\tuple{\T,\T}$ of some formula such that
\emph{there is no smaller} $\H < \T$ such that $\tuple{\H,\T}$ is also a model of the formula.
This quantification seems to have a second-order nature, as it has to do with different configurations of the truth of propositions.
In fact, in the non-temporal case, equilibrium models of a formula have been captured using Quantified Boolean
Formulas~\cite{petowo01a} and subsequently generalized by using a syntactic operator $\SM{\cdot}$~\cite{feleli07a}.
This gives a second-order formula that uses the translation from Section~\ref{sec:tht2ltl} and
quantifies over primed propositions.
This approach opens the natural possibility of capturing \TEL\ through \emph{Quantified} \LTL\ \cite{sistla83a},
an extension of \LTL\ in which we can use second-order quantifiers over propositions, and
define a temporal version of the $\SM{\cdot}$ operator as a \QLTL\ formula.
The advantage of having a quantified temporal formula $\SM{\varphi}$ capturing the \TS-models of $\varphi$ is that there
exist several methods for reducing that formula to an automaton, as we see below.
Given that our interest is focused on finite automata, in the rest of this section, we assume that the alphabet \PV\ is finite.

\subsection{\TS-models in terms of \QLTL: the \SMD\ operator}

We start by recalling the syntax and semantics of \emph{Quantified Linear-Time Temporal Logic} or \QLTL~\cite{sistla83a,sivawo87a,degola16a}.
The syntax of \QLTL\ extends the one of temporal formulas by quantifiers on propositional variables,
that is, formulas of the type $\exists a\, \varphi$ and $\forall a\, \varphi$ are added.

Interpreting the additional formulas requires the concept of a variant trace.
\begin{definition}[$X$-variant trace;~\citeNP{frerey02a}]\label{def:xvar}
  Given two traces \T\ and $\T'$ of length $\lambda$ and $X \subseteq \PV$,
  we say that $\T'$ is an \emph{$X$-variant trace} of \T\
  if  $T_i \setminus X = T_i' \setminus X$ for all $i \in \intervo{0}{\lambda}$.
\end{definition}
With this,
the semantics for \QLTL\ is obtained from that of \LTL\ by extending its satisfaction relation with the cases of
quantified formulas as follows:
\begin{itemize}
\item $\T, i \models \forall a\, \varphi$ iff $\T', i \models \varphi$ for all  $\lbrace a \rbrace$-variant traces $\T'$ of \T,
\item $\T, i \models \exists a\, \varphi$ iff $\T', i \models \varphi$ for some $\lbrace a \rbrace$-variant traces $\T'$ of \T.
\end{itemize}

It is known that \QLTLo\ is as expressive as B\"uchi automata~\cite{sivawo87a}, and so, more expressive than \LTLo.
For instance, the indicative ``even states'' property~\cite{wolper83a} can be expressed by the quantified temporal
formula
\(
\exists q\, \left(\neg q \wedge \alwaysF\left( q \leftrightarrow \next \neg q\right)\right)\wedge \alwaysF\left(q \rightarrow p\right)
\).

We also recall that a quantified temporal formula is in \emph{prenex normal form} if it is of the form $Q_1\; x_1, Q_2\; x_2 \cdots Q_n\; x_n\; \varphi$ where each $Q_i \in \lbrace \exists, \forall \rbrace$ and each $x_i$ is a propositional variable occurring in the \emph{ quantifier-free} formula $\varphi$.
An alternation in a quantified prefix is a sequence of $\exists y \forall x$ or $\forall x \exists y$ occurring in a prefix.
Given a quantified temporal formula $\varphi$ in prenex normal form, the \emph{alternation depth} of $\varphi$ is $k$ if the prefix associated with $\varphi$ contains $k$ alternations.
Also, given a quantified temporal formula $\varphi$, we denote by $\mathit{free}(\varphi)$ the number of propositional variables occurring in $\varphi$ which are not quantified.

The complexity  of the satisfiability problem  in \QLTL\ turns to be non-elementary~\cite{sivawo87a}.
However, for the case of formulas in prenex normal form with alternation depth $k$ the following result has been obtained.

\begin{theorem}[\citeNP{sistla83a}]\label{th:sistla} The \QLTL\ satisfiability problem of a quantified temporal formula $\varphi$ in prenex normal form with alternation depth $k$ is $k$-\textsc{ExpSpace}-complete.\footnote{By convention, $0$-\textsc{ExpSpace} stands for \textsc{PSpace}.}
\end{theorem}

Now, in order to encode \TS-models of a temporal formula $\varphi$ using a quantified temporal formula,
we use translation $\varphi^*$ from Section~\ref{sec:tht2ltl} and define the following notation.
Let $\vtuple{a}$ and $\vtuple{a'}$ stand for the tuples $\tuple{a_1, \dots, a_n}$ and $\tuple{a'_1,\dots,a'_n}$ of $n$ atoms, respectively.
We define the following expressions:
\[
  \vtuple{a'} \le \vtuple{a} \eqdef \bigwedge \limits_{i=1}^n  \alwaysF(a_i' \rightarrow a_i) \qquad\qquad
  \vtuple{a'} < \vtuple{a} \eqdef \vtuple{a'} \le \vtuple{a} \wedge  \bigvee \limits_{i=1}^n  \eventuallyF (\neg a_i' \wedge a_i).
\]
Note that, whenever $\vtuple{a}$ coincides with the original propositional signature $\PV$, $\vtuple{a'} \le \vtuple{a}$
amounts to axiom \eqref{f:ax} and is used to guarantee that atoms in trace \H\ are also included in trace \T.
Expression $\vtuple{a'} < \vtuple{a}$ further forces $\H < \T$ as stated by the following proposition.
\begin{proposition}\label{prop:primedatoms}
Let $\T^*$ be  a trace over signature $\PV^* = \PV \cup \lbrace a' \mid a \in \PV\rbrace$ and let \H\ and \T\ be such that $H_i \eqdef \{a \mid a' \in \T^*_i\}$ and $T_i \eqdef T^*_i \cap \PV$.
Also, let $\vtuple{a}$ be a tuple containing all propositional variables on \PV\ and let $\vtuple{a'}$ be of the same
length as $\vtuple{a}$.
Then, we have:
\begin{enumerate}
\item $\T^*,0\models \vtuple{a'} \le \vtuple{a}$ \ iff \ $\H \le \T$,
\item $\T^*,0\models \vtuple{a'}  <  \vtuple{a}$ \ iff \ $\H  <  \T$.\qed
\end{enumerate}
\end{proposition}

We can now define the temporal version of the $\SMD$ operator by~\citeN{feleli07a}:
Let $\varphi$ be a temporal formula over signature $\PV$ and let $\vtuple{a}$ be a tuple with all propositions in $\PV$.
Then, we define the \QLTL\ expression:
\begin{equation}
\SM{\varphi} \eqdef \varphi \wedge \neg \exists \vtuple{a'}\; \big( \vtuple{a'} < \vtuple{a} \wedge \varphi^* \big).\label{eq:tsm}
\end{equation}
The fact that $\SM{\varphi}$ captures the \TS-models of $\varphi$ is established by the following theorem.
\begin{theorem}
  A trace \T\ is a \TS-model of a temporal formula $\varphi$ iff \T\ is a \QLTL-model of $\SM{\varphi}$.\qed
\end{theorem}

Note that a simple analysis of $\SM{\varphi}$ allows us to match the upper bound complexity of the satisfiability problem of temporal equilibrium logic obtained by~\citeN{cabdem11a}.

\begin{corollary} The satisfiability problem of \TELo\ is at most \textsc{ExpSpace}-complete.\end{corollary}

\begin{proof}
  Note that $\SM{\varphi}$ can be rewritten into
  $\forall\vtuple{a'}\; \left(\varphi \wedge  \left( \vtuple{a'} < \vtuple{a} \rightarrow  \neg \varphi^* \right)\right)$
  by applying the De Morgan laws and pushing $\varphi$, whose variables are free, inside of the scope of the second-order quantifiers.
  Since checking whether the previous formula is satisfiable iff
  $\exists\vtuple{a}\;\forall\vtuple{a'}\; \left(\varphi \wedge  \left( \vtuple{a'} < \vtuple{a} \rightarrow  \neg \varphi^* \right)\right)$
  is satisfiable~\cite{degper21a}, we can apply Theorem~\ref{th:sistla} to obtain such \textsc{ExpSpace}-complete upper bound.
\end{proof}

Note that for finite traces this amounts to a characterization in terms of \emph{weak} \QLTL.

\begin{example}[Example~\ref{ex2} continued]
Consider again formula $\varphi=\alwaysF (\neg b \to a)$ in \eqref{f:ex2.2}.
As mentioned, formula $\vtuple{a} \leq \vtuple{a'}$ amounts to axiom \eqref{f:ax} for all atoms in the signature which, in this case, corresponds to formula \eqref{f:ex2.1}.
Now $\vtuple{a} < \vtuple{a'}$ is stronger, as it requires not only $\H \leq \T$ but also $\H \neq \T$.
In this case, we obtain the expression:
\begin{eqnarray}
\eqref{f:ex2.1} \wedge \big( \ \eventuallyF (\neg a' \wedge a) \vee \eventuallyF (\neg b' \wedge b) \ \big)
\end{eqnarray}
whereas $\varphi^*$ corresponds to \eqref{f:starex}, as illustrated above.
So, as a result, $\SM{\varphi}$ is finally unfolded into:
\begin{eqnarray*}
\SM{\varphi} & = & \SM{\alwaysF (\neg b \to a)} \\
& = & \nec(\neg b \to a ) \wedge \neg \exists\ a' \ b' \ \big(\\
& & \quad \quad \alwaysF (a' \to a) \wedge \alwaysF (b' \to b) \wedge \big( \ \eventuallyF (\neg a' \wedge a) \vee \eventuallyF (\neg b' \wedge b) \ \big) \\
& & \quad \quad \wedge \ \alwaysF (\neg b \to a) \wedge \alwaysF (\neg b \to a') \ \ \big)
\end{eqnarray*}
The \LTL-models of $\alwaysF (\neg b \to a)$ are traces \T\ where each state satisfies $T_i \neq \emptyset$.
Take any \T\ where there is some $T_i$ including atom $b$.
Then, we can extend \T\ to $\T^*$ for signature $\{a,b,a',b'\}$ repeating the truth of $a'$ and $b'$ with respect to $a$ and $b$ in all states, except in state $T^*_i$ where we just leave $b'$ false.
It is not difficult to see that this extended interpretation $\T^*$ satisfies $\vtuple{a} < \vtuple{a'} \wedge \varphi^*$ and so, \T\ cannot be a model of $\SM{\varphi}$.
Therefore, \T\ must make $b$ false at all states and the only remaining model for $\SM{\varphi}$ is the trace \T\ where $T_i=\{a\}$ for all states, which coincides with the only \TS-model of $\varphi$.
\end{example}

\subsection{From \QLTL\ to automata}

Once the \TS-models of $\varphi$ are captured by a quantified temporal formula, there are several possibilities to obtain the corresponding automaton.
For instance, \citeN{cabdem11a} used the following approach\footnote{In fact, \citeN{cabdem11a} never constructed the formula $\SM{\varphi}$ as an intermediate step, but the description we provide here is equivalent.} in the case of \TELo:
First, we build a B\"uchi automaton $\mathfrak{A}_1$ that yields the \LTLo-models of $\varphi$ as usual.
Then, we build a second automaton, $\mathfrak{A}_2$, corresponding to the temporal formula $\vtuple{a} < \vtuple{a'} \wedge \varphi^*$ for the extended signature $\PV^*$.
On this last construction, we perform a filter operation to obtain the new automaton $h(\mathfrak{A}_2)$ that removes the primed atoms from the signature.
In other words, $h(\mathfrak{A}_2)$ captures the models of the quantified temporal formula $\exists \vtuple{a'} (\vtuple{a} < \vtuple{a'} \wedge \varphi^*)$ and, in terms of \THT, this corresponds to the \T\ traces for which there exists some $\H < \T$.
The final step is using the operations of complement and intersection of B\"uchi automata to obtain $\mathfrak{A}_1 \cap \overline{h(\mathfrak{A}_2)}$ that naturally corresponds to the formula \eqref{eq:tsm} of the \SMD\ operator.
This method was implemented later on in the tool \abstem~\cite{cabdie14a} that allows for both computing \TS-models and
checking strong equivalence of temporal formulas in \TELo.

In what follows, we present an analogous method for \TELf\
yet by reducing \SMD\ formulas to automata over finite words.
More precisely, we consider two different types of finite automata $\mathfrak{A}$,
namely, NFAs and AFWs.

A NFA is a structure $\left(\Sigma, Q, Q_0, \delta, F\right)$ where
\begin{enumerate}
\item $\Sigma$ is the alphabet,
\item $Q$ is a set of states,
\item $Q_0\subseteq Q$ is a set of initial states,
\item $\delta: Q \times \Sigma \mapsto 2^Q$ is a transition function and
\item $F\subseteq Q$ is a set of final states.
\end{enumerate}

A run of a NFA for an input word $w=(X_0 \dots X_{n-1})$ of length $n>0$ is a finite sequence of states $q_0,\dots, q_n$ such that $q_0\in Q_0$ and $\delta(q_i,X_i) = q_{i+1}$ for all $0\le i < n$.
The run is said to be \emph{accepting} if $q_n \in F$.
By $\mathcal{L}({\mathfrak{A}})$ we denote the set of runs accepted by $\mathfrak{A}$.
The following theorem shows the relation between \LTLf\ and NFAs.
\begin{theorem}[\citeNP{zhtalipuva17a,cabamumc18a}]\label{th:ltl2nfa}
  A given temporal formula $\varphi$ over \PV\ can be translated into a NFA $\mathfrak{A}_{\varphi}$ such that
  the set of \LTLf-models of $\varphi$ corresponds to $\mathcal{L}(\mathfrak{A}_\varphi)$.\qed
\end{theorem}
Note that the alphabet of $\mathfrak{A}_{\varphi}$ consists of $2^{\PV}$,
which amounts to the set of interpretations over \PV.

To give an example of how the corresponding algorithm works, let us consider the formula
\(
\varphi = \alwaysF(\neg a \rightarrow \next a)
\)
in \eqref{f:defnext}.
In \LTLf, $\varphi$ is equivalent to $\alwaysF( a \vee \next a)$, which means that every \LTLf-model of
this formula satisfies $a$ at a state $i$ or $i+1$, for all $0\le i < \lambda$.
The set of \LTLf-models of $\varphi$ is captured by the NFA in Figure~\ref{fig:ltlm}.
We see that every run in which $a$ is false in two consecutive states is disregarded.
\begin{figure}[h!]\centering
	\begin{tikzpicture}[->,>=stealth',minimum width=.05cm,font=\footnotesize, shorten >=0.5pt,auto,node distance=2cm,semithick]
	\node[state](s1) {\centering\scriptsize\textbf{init}};
	\node[state](s2)[right of = s1]{};
	\node[state,accepting](s3) [left of = s1]{};
	\node[state](s4) [right of = s2]{};

	\path[->]
	(s1) edge [] node[] {$\neg a$} (s2)
	(s1) edge [] node[] {$a$} (s3)
	(s3) edge [loop left] node {$a$} (s3)
	(s3) edge [bend left] node[pos=.8] {$\neg a$} (s2)
	(s2) edge [bend left] node[pos=.8] {$a$} (s3)
	(s2) edge [] node[] {$\neg a$} (s4)
	(s4) edge [loop above] node {$a$} (s4)
	(s4) edge [loop right] node {$\neg a$} (s4);
	\end{tikzpicture}
	\caption{NFA accepting the \LTLf-models of $\alwaysF(\neg a \rightarrow \next a)$.}
	\label{fig:ltlm}
\end{figure}

On the other hand,
an AFW is a structure $\left(\Sigma, Q, q_0, \delta, F\right)$ where:
\begin{enumerate}
\item $\Sigma$ is the alphabet,
\item $Q$ is a set of states,
\item $q_0\in Q$ is the initial state,
\item $\delta: Q \times \Sigma \mapsto \mathbb{B}^+(Q)$ is a transition function,
 where $\mathbb{B}^+(Q)$ is built from the elements of $Q$, conjunction, disjunction, $\top$ and $\bot$,
\item $F\subseteq Q$ is a set of final states.
\end{enumerate}

\label{def:AFW} Given an input word $w=(X_0 \dots X_{n-1})$ of length $n>0$, a run of an AFW is a tree labeled by states of the AFW such that
\begin{enumerate}
\item the root is labeled by $q_0$,
\item if a tree node $z$ at level $i$ is labeled by a state $q$ and $\delta(q,X_i) = \Theta$ then
  either $\Theta = \top$ or
  some $P \subseteq Q$ satisfies $\Theta$ and $z$ has a child for each element in $P$, and
\item the run is accepting if all leaves at depth $n$ are labeled by states in $F$.
\end{enumerate}
Thus, a branch in an accepting run has to hit the $\top$ transition or hit an accepting state after reading all the input word $w$.

The correspondence between \LTLf\ and AFW is stated in the following theorem
\begin{theorem}[\citeNP{giavar13a}]\label{th:ltl2afw}
  A temporal formula $\varphi$ can be translated into an AFW
  \(
  \mathfrak{A}_\varphi = (2^{\PV \cup \lbrace \mathit{last} \rbrace} , Q, q_{\varphi}, \delta, \emptyset)
  \),
  where
  \begin{enumerate}
  \item $\mathit{last}$ corresponds a fresh atom indicating the last state of the trace,
  \item $Q$ corresponds to the negation-closed closure of $\varphi$ due to~\citeN{fislad79a},
  \item $q_{\varphi}$ is the initial state, which corresponds to the state labeled with the formula $\varphi$,
\item $\delta: Q \times 2^{\PV\cup \lbrace \mathit{last} \rbrace} \mapsto \mathbb{B}^+(Q)$
    is the transition function defined as follows,

    where $X\subseteq {\PV\cup \lbrace \mathit{last} \rbrace}$ is a state
    \begin{enumerate}
    \item $\delta(q_\top,X) = \top$
    \item $\delta(q_\bot,X) = \bot$
    \item $\delta(q_a, X)= \begin{cases} \top & \hbox{ if } a \in X;\\ \bot & \hbox{ otherwise } \end{cases}$
    \item $\delta(q_{\neg \varphi}, X) = \overline{\delta(q_\varphi,X)}$, where $\overline{\delta(q_\varphi,X)}$ is
      obtained from $\delta(q_\varphi,X)$
      by switching $\wedge$ and $\vee$,
      by switching $\top$ and $\bot$ and, in addition,
      by negating subformulas in $X$.
    \item $\delta(q_{\varphi \wedge \psi}, X)=  \delta(q_\varphi, X) \wedge \delta(q_\psi, X)$
    \item $\delta(q_{\varphi \vee \psi}, X)= \delta(q_\varphi, X) \vee \delta(q_\psi, X)$
    \item $\delta(q_{\varphi \rightarrow \psi}, X)= \delta(q_{\neg \varphi}, X) \vee \delta(q_\psi,X)$
    \item $\delta(q_{\next \varphi}, X)=  \begin{cases}  q_\varphi & \hbox{ if } \mathit{last} \not \in X\\ \bot & \hbox{ otherwise }     \end{cases} $
    \item $\delta(q_{\wnext \varphi}, X)= \begin{cases}  q_\varphi & \hbox{ if } \mathit{last} \not \in X\\ \top & \hbox{ otherwise }    \end{cases} $
\item $\delta(q_{\varphi \until \psi}, X)= \begin{cases} \delta(q_\psi, X) & \hbox{ if } \mathit{last} \in X\\
    												       \delta(q_\psi, X) \vee \left( \delta(q_\varphi,X) \wedge  q_{\varphi \until \psi}\right) & \hbox{ otherwise }
    									      \end{cases}$
    \item $\delta(q_{\varphi \release \psi}, X)= \begin{cases}\delta(q_\psi, X)  & \hbox{ if }  \mathit{last} \in X\\
														     \delta(q_\psi, X) \wedge \left( \delta(q_\varphi,X) \vee q_{\varphi \release \psi}\right)& \hbox{ otherwise }\end{cases}$\qed
\end{enumerate}
  \end{enumerate}
  Moreover, as in the case of NFAs, $\mathcal{L}(\mathfrak{A}_\varphi)$ corresponds to the set of \LTLf-models of $\varphi$.
\end{theorem}
As above,
the formula $\alwaysF(\neg a \rightarrow \next a) \equiv \alwaysF\left( a \vee \next a\right)$ can be translated into the AFW of Figure~\ref{fig:ltlafw}.
Here, the alternating automaton is represented as a NFA extended with a new node type, labeled as ``$\forall$'' and splitting some incoming edge into a (possibly empty) set of outgoing, unlabeled edges.
As in the NFA, if multiple outgoing edges from a state share the same label $X$, they are understood as an
\emph{existential} constraint on paths (a run must follow at least one of them).
As also happens in a NFA, when there is no outgoing edge for some label $X$, it is understood as a transition to $\bot$ (the ``empty'' existential constraint), i.e., no accepting path can be built using $X$.
As it can be imagined, the new $\forall$-nodes represent the dual, \emph{universal} constraints on paths: a run reaching $\forall$ must follow \emph{all} the outgoing edges.
Analogously, when the $\forall$-node has no outgoing edges (the ``empty'' universal constraint $\top$) the run is trivially accepted.

\begin{figure}[h!]\centering
	\begin{tikzpicture}[->,>=stealth',minimum width=width=.05cm,font=\footnotesize, shorten >=0.5pt,auto,node distance=2.3cm,semithick]
	\node[state](a0) {${\state{\alwaysF(a \vee  \circ a)}}$};
	\node[] (a1) [right of = a0] {$\forall$};
	\node[state] (a2) [right of = a1] {$\state{a}$};
	\node[] (a3) [above of = a1] {$\forall$};
	\node[] (a7) [right of = a2] {$\forall$};

\path[->]
	(a0) edge [loop above] node[above] {$a \wedge \neg \mathit{last}$} (a0)
	(a0) edge  node[above] {$\neg a \wedge \neg \mathit{last}$} (a1)
	(a1) edge [] node {} (a2)
	(a1) edge [bend left] node {} (a0)
	(a2) edge [] node[] {$a$} (a7)
(a0) edge [] node[right] {$a \wedge \mathit{last}$} (a3);

\end{tikzpicture}
	\caption{AFW accepting the \LTLf-models of $\alwaysF(a \vee \next a)$.}
	\label{fig:ltlafw}
\end{figure}

Figure~\ref{fig:afwrun} shows an accepting run of the AFW of Figure~\ref{fig:ltlafw} with the input $\emptyset \cdot \lbrace a \rbrace \cdot \lbrace a, \mathit{last}\rbrace$.
For sake of clarity, we have labeled each tree edge with the input symbol $X_i$ that fires the transition and, when applicable, we also show the $\forall$ constraints that are reached along the run.
As we can see, since $X_0=\emptyset$ satisfies $\neg a \wedge \neg \mathit{last}$, it splits the run into two branches, but both of them eventually reach the empty $\forall$ (i.e.\ the accepting $\top$ transition).

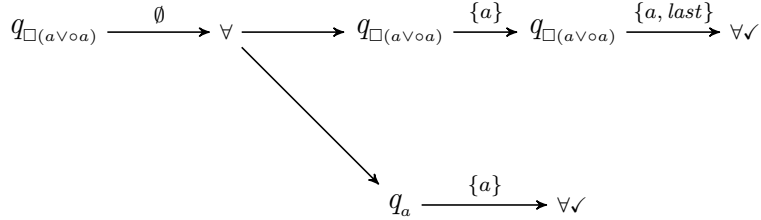
\begin{figure}[h!]\centering
	\begin{tikzpicture}[->,>=stealth',minimum width=width=.05cm,font=\footnotesize, shorten >=0.5pt,auto,node distance=2.3cm,semithick]
	\node[](a0) {$\state{\alwaysF(a \vee  \circ a)}$};
	\node[](a1) [right of = a0] {$\forall$};
	\node[](a2) [right of = a1] {$\state{\alwaysF(a \vee  \circ a)}$};
	\node[](a3) [below of = a2] {$\state{a}$};
	\node[](a4) [right of = a2] {$\state{\alwaysF(a \vee  \circ a)}$};
	\node[](a5) [right of = a4] {$\forall \checkmark$};
	\node[](a6) [below of = a4] {$\forall \checkmark$};
	\path[->]
	(a0) edge  node[above] {$\emptyset$} (a1)
	(a1) edge [] node {} (a2)
	(a1) edge [] node {} (a3)
	(a2) edge [] node[above] {$\{a\}$} (a4)
	(a4) edge [] node[above] {$\{a,\mathit{last}\}$} (a5)
	(a3) edge [] node[above] {$\{a\}$} (a6)
	;
	\end{tikzpicture}
\caption{A branching run of the AFW of Figure~\ref{fig:ltlafw} accepting the input $\emptyset \cdot \lbrace a \rbrace \cdot \lbrace a, \mathit{last}\rbrace$.}
\label{fig:afwrun}
\end{figure}

To illustrate a non-accepting run, suppose we take the input $\emptyset \cdot \lbrace a \rbrace \cdot \lbrace \mathit{last}\rbrace$ instead, so that the last interpretation $X_2$ does not satisfy atom $a$ now.
Figure~\ref{fig:afwrun2} shows the corresponding run obtained from the AFW in Figure~\ref{fig:ltlafw}.
As before, $X_0=\emptyset$ splits the run but, this time, the path followed by the top branch is not accepted since no
outgoing edge from $\state{\alwaysF(a \vee  \circ a)}$ has a formula satisfied by $X_2=\{\mathit{last}\}$.
This implicitly means $\delta(\state{\alwaysF(a \vee  \circ a)},\{\mathit{last}\})=\bot$.

\begin{figure}[h!]\centering
	\begin{tikzpicture}[->,>=stealth',minimum width=width=.05cm,font=\footnotesize, shorten >=0.5pt,auto,node distance=2.3cm,semithick]
	\node[](a0) {$\state{\alwaysF(a \vee  \circ a)}$};
	\node[](a1) [right of = a0] {$\forall$};
	\node[](a2) [right of = a1] {$\state{\alwaysF(a \vee  \circ a)}$};
	\node[](a3) [below of = a2] {$\state{a}$};
	\node[](a4) [right of = a2] {$\state{\alwaysF(a \vee  \circ a)}$};
	\node[](a5) [right of = a4] {$\bot$};
	\node[](a6) [below of = a4] {$\forall \checkmark$};
	\path[->]
	(a0) edge  node[above] {$\emptyset$} (a1)
	(a1) edge [] node {} (a2)
	(a1) edge [] node {} (a3)
	(a2) edge [] node[above] {$\{a\}$} (a4)
	(a4) edge [] node[above] {$\{\mathit{last}\}$} (a5)
	(a3) edge [] node[above] {$\{a\}$} (a6)
	;
	\end{tikzpicture}
	\caption{A run of the AFW of Figure~\ref{fig:ltlafw} rejecting the input $\emptyset \cdot \lbrace a \rbrace \cdot \lbrace \mathit{last}\rbrace$.}
	\label{fig:afwrun2}
\end{figure}

The rest of this section is devoted to show how Theorem~\ref{th:ltl2nfa} and~\ref{th:ltl2afw} can be used to obtain the
following result.
\begin{theorem}
A temporal formula $\varphi$ can be translated into an AFW/NFA $\mathfrak{A}_\varphi$ such that $\mathcal{L}(\mathfrak{A}_\varphi)$ corresponds to the set of \TELf-models of $\varphi$.
\end{theorem}

Our construction relies on several intermediate automata transformations
corresponding to the ones by~\citeN{cabdem11a}.
Each of them accepts the model of one specific part of the \SMD\ expression.
The involved automata and their correspondence with \SM{\varphi} are depicted in~\eqref{eq:automata}.
\begin{equation}
  \SM{\varphi} = \underbrace{\underbrace{\varphi}_{\mathfrak{A}_1} \wedge \underbrace{\neg \underbrace{\exists \vtuple{a'}\; \underbrace{\vtuple{a'} < \vtuple{a} \wedge (\varphi)^*}_{\mathfrak{A}_2}}_{h(\mathfrak{A}_2)}}_{\overline{h(\mathfrak{A}_2)}}}_{ \mathfrak{A}_1 \cap \overline{h({\mathfrak{A}_2})}}.\label{eq:automata}
\end{equation}
The computation of the final automaton $\mathfrak{A}_1 \cap \overline{h({\mathfrak{A}_2})}$
accepting the \TELf-models of $\varphi$  is done in a compositional manner,
analogous to the translation of \QLTL\ or {\em Weak Monadic Second Order Theory of one Successor} into
NFA/AFW~\cite{bomapi15a,hejejoklparasa95a}.
The two initial automata $\mathfrak{A}_1$ and $\mathfrak{A}_2$ are computed by using Lemma~\ref{th:ltl2afw} or~\ref{th:ltl2nfa},
depending on whether we are interested in computing an NFA or an AFW.
Note that the formulas recognized by $\mathfrak{A}_1$ and $\mathfrak{A}_2$ are temporal formulas,
which justifies the application of these lemmas.
$h(\mathfrak{A}_2)$ is obtained from $\mathfrak{A}_2$ by \emph{projecting} the atoms of the type $a'$. In language-theoretic terms, a \emph{projection operation}  $h(\mathcal{L}(\mathfrak{A}))$ is defined as follows.
\begin{displaymath}
h(\mathcal{L}(\mathfrak{A})) \eqdef \lbrace w \mid \exists\;  w' \hbox{ s.t. } w \hbox{ is identical to $w'$ except for all atoms of the type } a'\rbrace.
\end{displaymath}
Since $h(\mathcal{L}(\mathfrak{A})) = \mathcal{L}(h(\mathfrak{A}))$~\cite{cabdem11a},
the projection can be applied directly over ${\mathfrak{A}_2}$ and, broadly speaking,
the resulting automaton accepts all \LTLf-models $\T$ for which there exists a trace $\H$ satisfying $\H < \T$.
$\overline{h(\mathfrak{A}_2)}$ is obtained by complementing $h(\mathfrak{A}_2)$;
it accepts all traces \T\ that are either
no \LTLf-models of $\varphi$ or
for which exists no trace \H\ satisfying $\H < \T$.
Automata complementation is the usual way to obtain automata accepting the \LTLf\ models of negated formulas.
Finally, $\mathfrak{A}_1 \cap \overline{h(\mathfrak{A}_2)}$ is obtained by intersecting $\mathfrak{A}_1$ and $\overline{h(\mathfrak{A}_2)}$.
This operation is equivalent to selecting all traces \T\ being
\LTLf-models of $\varphi$
for which there is no trace $\H < \T$ satisfying $\tuple{\H,\T}, 0 \models \varphi$.
In this case, $\tuple{\T,\T}$ is an equilibrium model of $\varphi$.

If we analyze
\(
\varphi = \alwaysF(\neg a \rightarrow \next a)
\)
in \TELf,
$a$ is false at the initial state since it cannot be provable.
Since $\neg a \rightarrow \next a$ is satisfied in the initial state, $a$ is true at time point $1$,
so the rule $\neg a\rightarrow \next a$ cannot be applied to prove $a$ at time point $2$ and the cycle starts again.
Hence, \TELf-models of $\varphi$ are of the form $\emptyset\cdot \lbrace a\rbrace\cdot \emptyset\cdot \lbrace a \rbrace \cdots$.
They are captured by the NFA in Figure~\ref{fig:telfm}, whose accepting runs correspond to sequences of this type.
\begin{figure}[h!]\centering
\begin{tikzpicture}[->,>=stealth',minimum width=.05cm,font=\footnotesize, shorten >=0.5pt,auto,node distance=2cm,semithick]
\node[state](s0)[]{\centering\textbf{init}};
\node[state](s1)[right of=s0] {};
\node[state,accepting](s4)[right of = s1]{};
\node[state](s3)[right of = s4]{};
\path[->]
(s0) edge [] node {$\neg a$} (s1)
(s0) edge [bend right] node {$a$} (s3)
(s2) edge [bend right] node {$a$} (s4)
(s4) edge [bend right] node {$\neg a$} (s2)
(s4) edge [] node {$a$} (s3)
(s2) edge [bend left] node {$\neg a$} (s3)
(s3) edge [loop right] node {$a$} (s3)
(s3) edge [loop above] node {$\neg a$} (s3)
;
\end{tikzpicture}
\caption{NFA accepting the \TELf-models of $\alwaysF(\neg a \rightarrow \next a)$.}
\label{fig:telfm}
\end{figure}
 \section{(Modal) Temporal Logic Programs}\label{sec:normal_forms}

In the previous section,
we have seen a first method to compute \TS-models and check \TEL-satisfiability, relying first on a quantified \LTL\ expression, and then, deriving the construction of different types of automata: B\"uchi for \TELo\ and NFA or AFW for \TELf.
Automata-based techniques are interesting for the analysis of \emph{the set of traces} (\TS-models) induced by a given temporal formula.
This is useful, for instance, to check several types of equivalence between two different temporal representations, or
to check the existence of a solution for a given planning problem.
However, in many practical applications, rather than exploring all possibilities, our interest is more focused on the generation of one, or a few \TS-models that encode solutions to a given temporal problem.
This is, in fact, the usual model-based problem solving orientation of ASP,
where stable models encode solutions to a particular problem, encoded by the rules in the logic program.
In this way, if we want to solve a temporal diagnosis scenario,
we are interested in finding at least one \TS-model that explains the observations.
Similarly, for a given planning problem, we look for some plan that reaches the goal in a finite number of steps.

In this section, we show how the generation of \TS-models can be done using ASP technology by considering temporal theories with a syntax closer to logic programming.
As we see next, both \TELo\ and \TELf\ can be reduced to a normal form that we call \emph{(modal) temporal logic programs}.
This normal form is useful for ASP-based computation.
For the finite case, there are two translations of this normal form into ASP.
One allows any temporal formula and provides a translation for a given trace length.
The other translation imposes some restrictions on the formula but allows for incremental solving.
There are tools supporting these translations, most notably \telingo~\cite{cakamosc19a} for the incremental approach.

We provide a generic normal form for temporal programs and a translation for temporal formulas into this normal form. Both the finite and the infinite case are special cases of this normal form.
\begin{definition}[Temporal literal, rule, and program]\label{def:temporal:rule}
Given alphabet \PV, we define the set of \emph{temporal literals} as
$\{a, \neg a, \previous a, \neg \previous  a \mid a \in \PV\}$.

A \emph{temporal rule} is either:
\begin{itemize}
\item an \emph{initial rule} of the form \ $\Bd\to\Hd$
\item a  \emph{dynamic rule} of the form \ $\wnext\,\alwaysF (\Bd\to \Hd)$
\item a  \emph{fulfillment rule} of the form \ $\alwaysF(\,\alwaysF \, p \to q \,)$ or $\alwaysF(\, p $ $\to \eventuallyF \, q \,)$
\item a  \emph{final rule} of the form \ $\alwaysF(\,\finally \to (\Bd \to \Hd )\,)$
\end{itemize}
where $\Bd=b_1 \wedge \dots \wedge b_n$ with $n\geq 0$, $\Hd=a_1 \vee \dots \vee a_m$
with $m \geq 0$ and the $b_i$ and $a_j$ are temporal literals for dynamic rules and
regular literals $\{a, \neg a \mid a \in \PV\}$ for initial and final rules,
and $p$ and $q$ are atoms.
We call \emph{temporal logic program} to a set of temporal rules.
\end{definition}
We normally allow some straightforward extensions of this form. For instance, an \emph{always}-rule has the form $\alwaysF (B \to A)$ and stands for the conjunction $(B \to A) \wedge \wnext \alwaysF( B \to A)$ of an initial and a dynamic rule, respectively (therefore, $B$ and $A$ only contain regular literals).
Moreover, an \emph{extended dynamic rule} allows body literals $b_i$ to be the atomic formulas $\initially$ or $\finally$ (since they can be encoded as auxiliary atoms). When we allow this last extension, final rules can be reduced to constraints (empty head) of the form $\alwaysF (\finally \to (B \to \bot))$.

\begin{theorem}[Normal form;~\citeNP{agcadipevi13a,cakascsc18a}]\label{thm:normalform}
Every temporal formula $\varphi$ can be converted into
a temporal program being \THTf-equivalent to $\varphi$
and into one being \THTo-equivalent to $\varphi$, in both cases, modulo auxiliary atoms.
\end{theorem}
Moreover, the reduction by~\citeN{cakascsc18a} further guarantees the following property.
\begin{corollary}\label{cor:finalrules}
For finite traces $\lambda \in \Nat$, every temporal formula can be converted into a \THTf-equivalent temporal logic program $P$ consisting of initial, dynamic, and final rules only.
\end{corollary}
We denote these three sets of rules as $\initial{P}, \dynamic{P}$ and $\final{P}$, respectively.\qed

In what follows, we introduce a more general reduction from temporal formulas to temporal logic programs that is
applicable to \emph{both} finite and infinite traces.
It uses an extended alphabet $\PV^+\supseteq\PV$
that additionally contains a new atom \Lab{\varphi} (a.k.a.\ label) for each formula $\varphi$ in the original language over \PV\
along with the label $\Lab{\neg\eventuallyF\finally}$.
This label represents the value of the formula $\neg\eventuallyF\finally$, which means that
it can be used to restrict ourselves to finite traces by including
the formula $\alwaysF (\neg \Lab{\neg\eventuallyF\finally})$.
For convenience, we use
$\Lab{\varphi} \overset{\mathit{def}}{=} \varphi $ if $\varphi$ is $\top , \bot$ or an atom $a \in \PV$.
For any non-atomic formula $\mu$ over \PV,
we define the translation \df\ given in Tables~\ref{tab:translation:1} and~\ref{tab:translation:2},
and call \df\ the \emph{definition} of $\mu$.
Note that the translation of the \emph{next} operator $\next$
includes the label $\Lab{\neg\eventuallyF\finally}$.
This ensures that if the formula $\alwaysF (\neg \Lab{\neg\eventuallyF\finally})$ is included and therefore only finite traces are considered,
then $\next \varphi$ cannot be satisfied in the final state.

\begin{table}[th]
\begin{minipage}[t]{0.49\textwidth}
\centering
\caption{Definition of formulas}
\label{tab:translation:1}
\medskip
\(
\begin{array}{|c|c|} \cline{1-2}
\mu &\df({\mu}) \\ \cline{1-2}
\varphi \wedge \psi
    	&
    	\alwaysF (\Lab{\mu} \leftrightarrow \Lab{\varphi} \wedge \Lab{\psi})
    \\ \cline{1-2}
\varphi \vee \psi
    	&
    	\alwaysF (\Lab{\mu} \leftrightarrow \Lab{\varphi} \vee \Lab{\psi})
    \\ \cline{1-2}
\varphi \to \psi
    	&
    	\alwaysF (\Lab{\mu} \leftrightarrow \Lab{\varphi} \to \Lab{\psi})
    \\ \cline{1-2}
\multirow{2}*{$\next \varphi$}
    	&\raisebox{11pt}{}{\,}
    	\wnext\alwaysF (\previous\Lab{\mu} \leftrightarrow \Lab{\varphi})\\
    	& \alwaysF ( \alwaysF  \Lab{\mu} \to \Lab{\neg\eventuallyF\finally})
    \\ \cline{1-2}
\multirow{2}*{$\previous \varphi$}
		&	\neg \Lab{\mu} \\
    	&
    	\wnext\alwaysF (\Lab{\mu} \leftrightarrow \previous\Lab{\varphi})
    \\ \cline{1-2}
\end{array}
\)
\end{minipage}
\begin{minipage}[t]{0.49\textwidth}
\centering
\caption{Definition of formulas}
\label{tab:translation:2}
\medskip
\(
\begin{array}{|c|c|} \cline{1-2}
\mu &\df({\mu})\\ \cline{1-2}
\multirow{3}*{$ \varphi \until \psi$}
    	&\raisebox{11pt}{}{\,}
    	\wnext\alwaysF (\previous\Lab{\mu} \leftrightarrow \previous \Lab{\psi} \vee (\previous\Lab{\varphi} \wedge \Lab{\mu}))\\
    	& \alwaysF ( \Lab{\mu} \to \eventuallyF \Lab{\psi})
    \\
	& \alwaysF  (\Lab{\psi} \to \eventuallyF \Lab{\mu} )
    \\ \cline{1-2}
\multirow{3}*{$ \varphi \release \psi$}
    	&\raisebox{11pt}{}{\,}
    	\wnext\alwaysF (\previous\Lab{\mu} \leftrightarrow \previous \Lab{\psi} \wedge (\previous\Lab{\varphi} \vee \Lab{\mu}))\\
    	& \alwaysF  ( \alwaysF \Lab{\psi} \to \Lab{\mu})
    \\
	& \alwaysF ( \alwaysF \Lab{\mu} \to \Lab{\psi} )
    \\ \cline{1-2}
\multirow{2}*{$ \varphi \since \psi$}
    	&
    	\Lab{\mu} \leftrightarrow \Lab{\psi}
    		\\
    	&
    	\wnext\alwaysF (\Lab{\mu} \leftrightarrow  \Lab{\psi} \vee (\Lab{\varphi}
    	\wedge \previous\Lab{\mu}))
    \\ \cline{1-2}
\multirow{2}*{$ \varphi \trigger \psi$}
    	& \Lab{\mu} \leftrightarrow \Lab{\psi}
    		\\
    	&
    	\wnext\alwaysF (\Lab{\mu} \leftrightarrow  \Lab{\psi} \wedge (\Lab{\varphi} \vee
    	\previous \Lab{\mu}))
    \\ \cline{1-2}
\end{array}
\)
\end{minipage}
\end{table}

These definitions of temporal formulas are not yet in normal form,
but they contain some double implications that can easily be transformed
into the format of temporal logic programs by simple, non-modal transformations in the propositional logic of \HT.
This translation of the definition into normal form was presented by~\citeN{cakascsc18a} for the finite case and
by~\citeN{agcadipevi13a} for the infinite case.
We denote the resulting set of formulas in normal form by $\df^\star$.
Given a theory $\Gamma$, we define $\subformulas(\Gamma)$ as the set of all subformulas of all formulas in $\Gamma$.
\begin{definition}[\citeNP{agcadipevi13a,cakascsc18a}]
We define the translation $\sigma$ as the following temporal logic program:
\begin{align*}
  \sigma(\Gamma) & = \{ \Lab{\gamma} \mid \gamma \in \Gamma \}
                   \cup
                   \{ \df^\star(\mu) \mid \mu \in \subformulas(\Gamma) \}
                   \qedmath
\end{align*}
\end{definition}
\begin{theorem}[\citeNP{agcadipevi13a,cakascsc18a}]
For any theory $\Gamma$ over \PV, we have
$
\left\lbrace
\M \mid \M \models \Gamma
\right\rbrace
=
\left\lbrace
\M' |  _{\PV} \mid \M' \models \sigma(\Gamma)
\right\rbrace
$.
\qed
\end{theorem}
This correspondence between models of $\Gamma$ and $\sigma(\Gamma)$ is, in fact, one-to-one, since the satisfaction of each label $\Lab{\gamma}$ is equivalent to the satisfaction of its associated formula $\gamma$.
The next result shows that the computation of $\sigma(\Gamma)$ has a polynomial complexity on the size of $\Gamma$.
\begin{theorem}
Translation $\sigma$ is linear.\qed
\end{theorem}

As an example, consider the theory $\Gamma_1 = \alwaysF (\neg p \to q \until p)$.
Its translation $\sigma(\Gamma_1)$ can be determined
using Table \ref{tab:example:translation}, independently of the finiteness of the trace.
On finite traces, the fulfillment rules \eqref{tab:example:translation:08},
\eqref{tab:example:translation:09}, \eqref{tab:example:translation:21} and
\eqref{tab:example:translation:22} can be replaced by the corresponding final rules.
On infinite traces, the initial rules
\eqref{tab:example:translation:01}, \eqref{tab:example:translation:02},
\eqref{tab:example:translation:10}, \eqref{tab:example:translation:11},
\eqref{tab:example:translation:12}, \eqref{tab:example:translation:13}
can be combined with dynamic rules
\eqref{tab:example:translation:03}, \eqref{tab:example:translation:04},
\eqref{tab:example:translation:14}, \eqref{tab:example:translation:15},
\eqref{tab:example:translation:16}, \eqref{tab:example:translation:17}
into formulas of the form $\alwaysF (r)$ where \textit{r} is the given initial rule.
Rules \eqref{tab:example:translation:09} and \eqref{tab:example:translation:22} are
superfluous in the infinite case.

\begin{table}[th]
\centering
\caption{Translation of $\Gamma_1$ into normal form}
\label{tab:example:translation}
\(
\begin{array}[t]{|c|cc|} \cline{1-3}
\mu &\df^\star (\mu) & \\ \cline{1-3}
\multirow{4}*{$\neg p$}
		& 	\Lab{1} \wedge p \to \bot &
		\refstepcounter{equation}(\theequation)\label{tab:example:translation:01} \\
    	& \neg p \to \Lab{1}  &
    	\refstepcounter{equation}(\theequation)\label{tab:example:translation:02} \\
    	& 	\wnext \alwaysF (\Lab{1} \wedge p \to \bot ) &
    	\refstepcounter{equation}(\theequation)\label{tab:example:translation:03} \\
    	& 	\wnext \alwaysF (\neg p \to \Lab{1})  &
    	\refstepcounter{equation}(\theequation)\label{tab:example:translation:04}
        \\ \cline{1-3}
\multirow{5}*{$ q \until p$}
    	& \wnext\alwaysF(\previous\Lab{2}\to\previous p \vee \previous q \wedge \Lab{2})\raisebox{11pt}{}{\,}
    	& \refstepcounter{equation}(\theequation)\label{tab:example:translation:05} \\
    	& \wnext\alwaysF(\previous q \wedge \Lab{2} \to \previous \Lab{2}) &
    	\refstepcounter{equation}(\theequation)\label{tab:example:translation:06} \\
    	& \wnext\alwaysF(\previous p \to \previous \Lab{2})&
    	\refstepcounter{equation}(\theequation)\label{tab:example:translation:07} \\
    	& \alwaysF (\Lab{2} \to \eventuallyF p) &
    	\refstepcounter{equation}(\theequation)\label{tab:example:translation:08} \\
    	& \alwaysF (p \to \eventuallyF \Lab{2}) &
    	\refstepcounter{equation}(\theequation)\label{tab:example:translation:09}
        \\ \cline{1-3}
\end{array}
\qquad
\begin{array}[t]{|c|cc|} \cline{1-3}
\mu &\df^\star(\mu) &\\ \cline{1-3}
\multirow{8}*{$ \neg p \to q \until p$}
    	& \Lab{3} \wedge \Lab{1} \to \Lab{2}
		&\refstepcounter{equation}(\theequation)\label{tab:example:translation:10} \\
    	& \neg \Lab{1} \to \Lab{3}
		&\refstepcounter{equation}(\theequation)\label{tab:example:translation:11} \\
    	& \Lab{2} \to \Lab{3}
		&\refstepcounter{equation}(\theequation)\label{tab:example:translation:12} \\
    	& \Lab{1} \vee \neg \Lab{2} \vee \Lab{3}
		&\refstepcounter{equation}(\theequation)\label{tab:example:translation:13} \\
    	& \wnext \alwaysF ( \Lab{2} \to \Lab{3} )
		&\refstepcounter{equation}(\theequation)\label{tab:example:translation:14} \\
    	& \wnext \alwaysF ( \Lab{3} \wedge \Lab{1} \to \Lab{2} )
		&\refstepcounter{equation}(\theequation)\label{tab:example:translation:15} \\
    	& \wnext \alwaysF ( \neg \Lab{1} \to \Lab{3} )
		&\refstepcounter{equation}(\theequation)\label{tab:example:translation:16} \\
    	& \wnext \alwaysF ( \top \to \Lab{1} \vee \neg \Lab{2} \vee \Lab{3} )
		&\refstepcounter{equation}(\theequation)\label{tab:example:translation:17} \\
	\cline{1-3}
\multirow{5}*{$\alwaysF  (\neg p \to q \until p)$}
		& \wnext\alwaysF ( \previous\Lab{3} \wedge \Lab{4} \to \previous\Lab{4} )\raisebox{11pt}{}{\,}
		&\refstepcounter{equation}(\theequation)\label{tab:example:translation:18} \\
		&\wnext\alwaysF ( \previous\Lab{4 }\to \previous\Lab{3} )
		&\refstepcounter{equation}(\theequation)\label{tab:example:translation:19} \\
		& \wnext\alwaysF ( \previous\Lab{4} \to  \Lab{4})
		&\refstepcounter{equation}(\theequation)\label{tab:example:translation:20} \\
		& \alwaysF  ( \alwaysF \Lab{3}\to \Lab{4} )
		&\refstepcounter{equation}(\theequation)\label{tab:example:translation:21} \\
		& \alwaysF  ( \alwaysF \Lab{4}\to \Lab{3} )
		&\refstepcounter{equation}(\theequation)\label{tab:example:translation:22} \\
		\cline{1-3}
\end{array}
\)
\end{table}

The interest of temporal logic programs is that, as suggested by their name, they have a syntax closer to logic programming.
This suggests that, at least for finite traces as in \TELf, computing \TS-models can be delegated to an ASP solver in some way.
In the rest of the section, we explain two different methods to achieve this goal.

\subsection{Bounded translation of temporal programs over finite traces to ASP}
\label{sec:bounded}

A first way to encode a temporal program $P$ into standard ASP is by assuming that we know the (finite) trace length $\lambda$ beforehand.
In that case, we can compute all models in $\TEL(P,\lambda)$ by a simple translation of $P$ into a regular program.
For this, we let
\(
\PV_k=\{\Stamp{a}{k}\mid a\in\PV\}
\)
be a time stamped copy of alphabet \PV\ for each time point $k \in \intervo{0}{\lambda}$.
\begin{definition}[Bounded translation;~\citeNP{cakascsc18a}]\label{def:temporal:bounded:translation}
  We define the translation $\tau$ of a temporal literal at time point $k$ as:
  \begin{align*}
    \tau_k(          a) &\eqdef \Stamp{a}{k    } &\tau_k(\neg            a) &\eqdef \Stamp{\neg a}{k    } & &\text{ for }a\in\PV\\
    \tau_k(\previous a) &\eqdef \Stamp{a}{k{-}1} &\tau_k(\neg \previous  a) &\eqdef \Stamp{\neg a}{k{-}1} & &\text{ for }a\in\PV
  \end{align*}
  We define the translation at time point $k$ of any (non-fulfillment) rule $r$ as those in Definition~\ref{def:temporal:rule} as:
  \begin{align*}
    \tau_k(r) &\eqdef \tau_k(a_1) \vee \dots \vee \tau_k(a_m)\leftarrow\tau_k(b_1) \wedge \dots \wedge \tau_k(b_n)
  \end{align*}
  We define the translation of a temporal program $P$ bounded by finite length $\lambda$ as:
  \[\textstyle
  \tau_\lambda(P)\eqdef \{ \tau_0(r)\mid r\in\initial{P}\}\ \cup \ \{\tau_k(r)\mid r\in\dynamic{P}, k \in \intervo{1}{\lambda}\}\ \cup\ \{\tau_{\lambda-1}(r)\mid r\in\final{P}\}
  \]
  where $\initial{P}, \dynamic{P}$ and $\final{P}$ are the initial, dynamic and final rules in $P$.
  \qed
\end{definition}
Note that the translation of temporal rules is similar in just considering the implication $B \to A$ in Definition~\ref{def:temporal:rule};
their difference manifests itself in their instantiation in $\tau_\lambda(P)$.

As an example, let program $P$ be the set of temporal rules:
\begin{align}
\{\
                                                 \to a      ,\quad
\label{ex:program} \wnext\,\alwaysF (\previous a \to b)     ,\quad
                   \alwaysF (\finally\to (\neg b \to\bot))
\ \}
\end{align}
This program has a single finite temporal stable model of length 2, viz.\ $\{a\}\cdot\{b\}$.
Applying translation $\tau$ for some bound $\lambda$ to our temporal program $P$ in \eqref{ex:program}
yields regular logic programs of the following form.
\[
\tau_\lambda(P)=
\{ a_0 \leftarrow {} \}
\ \cup\
\{ b_k \leftarrow a_{k-1}\mid k \in \intervo{1}{\lambda}\}
\ \cup\
\{ \bot  \leftarrow \neg b_{\lambda-1} \}
\]
Program $\tau_1(P)$ has the stable model $\{a_0,b_1\}$ but all $\tau_\lambda(P)$ for $\lambda>2$ are unsatisfiable (have no stable models).

\begin{theorem}[\citeNP{cakascsc18a}]\label{thm:temporal:bounded:translation}
Let $P$ be a temporal program over \PV.
Let $\T=(T_i)_{\rangeo{i}{0}{\lambda}}$ be a trace of finite length $\lambda$ over \PV\
and
$X$ a set of atoms over $\bigcup_{\rangeo{i}{0}{\lambda}}\PV_i$
such that
\(
a\in T_i\text{ iff }a_i\in X\text{ for } i \in \intervo{0}{\lambda}
\).

Then,
\T\ is a temporal stable model of $P$
iff
$X$ is a stable model of $\tau_\lambda(P)$.\qed
\end{theorem}
Applied to our example,
this result confirms that the temporal stable model $\{a\}\cdot\{b\}$ of $P$
corresponds to the stable model $\{a_0,b_1\}$ of $\tau_2(P)$.

Using this translation we have implemented a system, \tel\footnote{\url{https://github.com/potassco/tel}}, that
takes a propositional theory $\Gamma$ of arbitrary temporal formulas and a bound $\lambda$ and returns the regular logic program $\tau_\lambda(P)$,
where $P$ is the intermediate normal form of $\Gamma$ left implicit.
The resulting program $\tau_\lambda(P)$ can then be solved by any off-the-shelf ASP system.
For illustration,
consider the representation of our example temporal program in~\eqref{ex:program} in \tel's input language.
\begin{lstlisting}[numbers=none,belowskip=2pt,aboveskip=2pt,basicstyle=\ttfamily]
  a.
  #next^ #always+ ( (#previous a) -> b).
  #always+ ( #final -> (~ b -> #false)).
\end{lstlisting}
As expected, passing the result of \tel's translation for horizon~2 to \clingo\ yields the stable model containing \lstinline{a(0)} and \lstinline{b(1)}
(suppressing auxiliary atoms).

\subsection{Pointwise translation of temporal programs over finite traces to ASP}
\label{sec:pointwise}

The bounded translation $\tau_\lambda(P)$ allows us to compute all models in $\TEL(P,\lambda)$ for a fixed bound $\lambda$.
However, in many practical problems (as in planning, for instance), $\lambda$ is unknown beforehand and
the crucial task consists in finding a representation of $\TEL(P,k)$ that is easily obtained from that of $\TEL(P,k\text{-}1)$.
In ASP, this can be accomplished via incremental solving techniques that rely upon the composition of logic program modules~\cite{oikjan06a}.
The idea is then to associate the knowledge at each time point with a module and to successively add modules corresponding to increasing time points
(while leaving all previous modules unchanged).
A stable model obtained after $k$ compositions then corresponds to a \TELf-model of length $k$.
This technique of modular computation, however, is only applicable when modules are \emph{compositional} (positive loops
cannot be formed across modules), something that cannot always be guaranteed for arbitrary temporal programs.
Still,~\citeN{cakascsc18a} identified a quite general syntactic fragment\footnote{In order to compute loop formulas for \TELo, \citeN{cabdie11a} used a similar fragment (\emph{splittable
    programs}) where rules cannot derive information from the future to the past.}
that implies compositionality.
We say that a temporal rule as in Definition~\ref{def:temporal:rule} is \emph{present-centered},
whenever all the literals $a_1,\dots,a_m$ in its head $\Hd$ are regular.
Accordingly, a set of such rules is a present-centered temporal program.
In fact,
such programs are sufficient to capture common action languages.
For illustration, we show how to encode action language $\mathcal{BC}$~\cite{leliya13a} into present-centered temporal
programs in \TELf\ in Section~\ref{sec:action_languages}.

Following these ideas, we provide next a ``\emph{point-wise}'' variant of our translation that allows for defining one module per time point and is compositional for the case of present-centered temporal programs.
We begin with some definitions.
A \emph{module}~\module{P} is a triple
\(
(P,I,O)
\)
consisting of
a logic program~$P$ over alphabet $\PV_P$
and sets~$I$ and~$O$
of \emph{input} and \emph{output}
atoms such that
\begin{enumerate}
\item $I\cap\nolinebreak O=\nolinebreak\emptyset$,
\item $\PV_P\subseteq I\cup O$, and
\item $\Head{P}\subseteq O$,
\end{enumerate}
where \Head{P} gives all atoms occurring in rule heads in $P$.
Whenever clear from context, we associate \module{P} with $(P,I,O)$.
In our setting,
a set $X$ of atoms is a stable model of \module{P},
if $X$ is a stable model of logic program
\(
P \).\footnote{Note that the default value assigned to input atoms is \emph{false} in multi-shot solving~\cite{gekakasc17a};
  this differs from the original definition~\cite{oikjan06a} where a choice rule is used.}
Two modules~$\module{P}_1$ and~$\module{P}_2$ are
\emph{compositional}, if
$O_1\cap O_2=\emptyset$
and
$O_1\cap C=\emptyset$ or $O_2\cap C=\emptyset$
for every strongly connected component~$C$ of the positive dependency graph of the logic program $P_1\cup P_2$.
In other words,
all rules defining an atom must belong to the same module, and no positive recursion is allowed among modules.
Whenever $\module{P}_1$ and~$\module{P}_2$ are compositional,
their \emph{join} is defined as the module
\(
\module{P}_1\sqcup\module{P}_2
=
(P_1\cup P_2,(I_1\setminus O_2)\cup (I_2\setminus O_1), O_1\cup O_2)
\).
The module theorem~\cite{oikjan06a} ensures that compatible stable models of $\module{P}_1$ and~$\module{P}_2$ can be combined to one of
$\module{P}_1\sqcup\module{P}_2$,
and vice versa.

For literals and rules, the point-wise translation $\tau^*$ coincides with $\tau$ up to final rules.
\begin{definition}[Point-wise translation: Temporal rules;~\citeNP{cakascsc18a}]\label{def:temporal:pointwise:translation}
We define the translation of a final rule $r$ as in Definition~\ref{def:temporal:rule} at time point $k$ as
  \begin{align}\label{eq:temporal:pointwise:translation:tri}
    \tau^*_k(r) &\eqdef \tau_k(a_1) \vee \dots \vee \tau_k(a_m)\leftarrow\tau_k(b_1) \wedge \dots \wedge \tau_k(b_n)\wedge\neg\Stamp{q}{k+1}
  \end{align}
  for a new atom $q\notin{\PV}$
  and of an initial or dynamic rule $r$ as $\tau^*_k(r)\eqdef\tau_k(r)$.
\end{definition}
The new atoms $q_{k+1}$ in \eqref{eq:temporal:pointwise:translation:tri} are used to deactivate instances of final rules.
This allows us to implement operator \finally\ by using $\neg q_{k+1}$ and therefore to enable the actual final rule unless $q_{k+1}$ is derivable.
The idea is then to make sure that at each horizon $k$ the atom $q_{k+1}$ is false, while $q_1,\dots,q_k$ are true.
In this way, only $\tau^*_k(r)$ is potentially applicable, while all rules $\tau^*_i(r)$ are disabled at earlier time points $i \in \intervo{1}{k}$.

Translation $\tau^*$ is then used to define modules for each time point as follows.
\begin{definition}[Point-wise translation: Modules]\label{def:temporal:module}
Let $P$ be a present-centered temporal program over \PV.
We define the module $\module{P}_k$ corresponding to $P$ at time point $k$ as:
\begin{align*}
  \module{P}_0 &\eqdef (P_0,                     \{\Stamp{q}{  1}\},{\PV_0}                    )&
  \module{P}_k &\eqdef (P_k,\PV_{k-1}\cup\{\Stamp{q}{k+1}\},{\PV_k}\cup\{\Stamp{q}{k}\})
                  \quad\text{ for }k>0
\end{align*}
where
\begin{align*}
  P_0 &\eqdef \{\tau^*_0(r)\mid r\in\;\initial{P}\} \cup\{\tau^*_0(r)\mid r\in\final{P}\}
  \\
  P_k &\eqdef \{\tau^*_k(r)\mid r\in  \dynamic{P}\} \cup\{\tau^*_k(r)\mid r\in\final{P}\}\cup\{q_k \leftarrow \}
\end{align*}
\qed
\end{definition}
Each module $\module{P}_k$ defines what holds at time point $k$.
The underlying present-centeredness warrants that
modules only incorporate atoms from previous time points,
as reflected by $\PV_{k-1}$ in the input of $\module{P}_k$.
The exception consists of
auxiliary atoms like $q_{k+1}$ that belong to the input of each $\module{P}_k$ for $k>0$
but only get defined in the next module $\module{P}_{k+1}$.
This corresponds to the aforementioned idea that $q_{k+1}$ is false when $\module{P}_k$ is the final module,
and is set permanently to true once the horizon is incremented by adding $\module{P}_{k+1}$.
Note that atoms like $q_{k+1}$ only occur negatively in rule bodies in $\module{P}_k$ and
hence cannot invalidate the modularity condition.
This technique allows us to capture the transience of final rules.

The point-wise translation of our present-centered example program $P$ from \eqref{ex:program} yields the following modules.
\begin{align*}
  \module{P}_0 & =
  \left(
    \{a_0\leftarrow\}\cup\{\leftarrow \neg b_0,\neg q_1\},
    \{q_1\},
    \{a_0,b_0\}
  \right)
  \\
  \module{P}_i & =
  \left(
    \{b_i\leftarrow a_{i-1}\}\cup\{\leftarrow \neg b_i,\neg q_{i+1}\}\cup \{q_i \leftarrow\},
    \{a_{i-1},b_{i-1},q_{i+1}\},
    \{a_i,b_i,q_i\}
  \right)
  \\\textstyle\hspace{-25pt}
  \bigsqcup_{i=0}^{\lambda-1}\module{P}_i
  &=\textstyle
  (
    P_0\cup\bigcup_{i=1}^{\lambda-1} P_i,
    \{q_{\lambda}\},
    \{a_i,b_i\mid \rangeo{i}{0}{\lambda}\} \cup \{q_i\mid \rangeo{i}{1}{\lambda}\}
    )
\end{align*}
As above, only the composed module for $\lambda=1$ yields a stable model, viz.\ $\{a_0,b_1,q_1\}$.

\begin{theorem}\label{thm:temporal:pointwise:translation}
Let $P$ be a present-centered temporal program over \PV.
Let $\T=(T_i)_{\rangeo{i}{0}{\lambda}}$ be a trace of finite length $\lambda$ over \PV\
and
$X$ a set of atoms over $\bigcup_{0\leq i< \lambda}\PV_i$
such that
\(
a\in T_i\text{ iff }a_i\in X\text{ for } i \in \intervo{0}{\lambda}
\).
Then, we have that

\T\ is a temporal stable model of $P$
iff
$X\cup\{q_i \mid i \in \intervo{0}{\lambda}\}$ is a stable model of
\(
\bigsqcup_{i \in \intervo{0}{\lambda}}\module{P}_i.
\)
\end{theorem}
As with Theorem~\ref{thm:temporal:bounded:translation},
this result confirms that the temporal stable model $\{a\}\cdot\{b\}$ of $P$
corresponds to the stable model $\{a_0,b_1,q_1\}$ of $\module{P}_0\sqcup\module{P}_1$.

As one might expect, not any temporal theory is reducible to a present-centered temporal program.
Hence, computing models via incremental solving imposes some limitations on the possible combinations of temporal operators.
Fortunately,~\citeN{cakascsc18a} identified again a quite natural and expressive syntactic fragment that is always reducible to present-centered programs.
We say that a temporal formula is a \emph{past-future rule} if it consists of rules as those in Definition~\ref{def:temporal:rule} where
$\Bd$ and $\Hd$ are just temporal formulas with the following restrictions:
$\Bd$ and $\Hd$ contain no implications other than negations ($\alpha \to \bot$),
$\Bd$ contains no future operators, and
$\Hd$ contains no past operators.
An example of a past-future rule is, for instance,
\begin{align}\label{f:fail}
\alwaysF ( \mathit{shoot} \wedge \previous \eventuallyP \mathit{shoot} \wedge \alwaysP \mathit{unloaded} &\to \eventuallyF \mathit{fail})
\end{align}
capturing the sentence: {\em ``If we make two shots with a gun that was never loaded, then it will eventually fail.''}
Then, we have the following result.
\begin{theorem}[Past-future reduction]\label{thm:normalform:pc}
Every past-future rule $r$ can be converted into a present-centered temporal program that is \TELf-equivalent to $r$.\qed
\end{theorem}

This past-future form is aligned with the orientation by~\citeN{gabbay87a} where past operators in the body (or antecedent) are used to declaratively check the recorded story, whereas future operators in the head (or consequent) are employed to provide imperative commands to be executed afterwards, as a result of firing the rule.

We have implemented a second system, \telingo\footnote{\url{https://github.com/potassco/telingo}}~\cite{cakamosc19a}, that
deals with present-centered temporal programs that are expressible in the full (non-ground) input language of \clingo\ extended with temporal operators.
In addition, \telingo\ offers several syntactic extensions to facilitate temporal modeling:
first, next operators can be used in singular heads and,
second, arbitrary temporal formulas can be used in integrity constraints.
All syntactic extensions beyond the normal form of Theorem~\ref{thm:normalform}
are compiled away by means of the translation used in its proof.
The resulting present-centered temporal programs are then processed according the point-wise translation.

To facilitate the use of operators \previous\ and \Next,
\telingo\ allows us to express them
by adding leading or trailing quotes to the predicate names of atoms, respectively.
For instance, the temporal literals $\previous p(a)$ and $\Next q(b)$ can be expressed by \lstinline{'p(a)} and \lstinline{q'(b)}, respectively.
For another example,
consider the representation of the sentence
\emph{``A robot cannot lift a box unless its capacity exceeds the box's weight plus that of all held objects''}:
\begin{lstlisting}[numbers=none,belowskip=2pt,aboveskip=2pt,basicstyle=\ttfamily]
 :- lift(R,B), robot(R), box(B,W),
    #sum { C : capacity(R,C); -V,O : 'holding(R,O,V) } < W.
\end{lstlisting}
Atom \lstinline{'holding(R,O,V)} expresses what the robot was holding at the \emph{previous} time point.

The distinction between different types of temporal rules is done in \telingo\ via \clingo's \lstinline{#program} directives~\cite{gekakasc17a},
which allow us to partition programs into subprograms.
More precisely,
each rule in \telingo's input language is associated with
a temporal rule $r$ of form $(b_1 \wedge \dots \wedge b_n \to a_1 \vee \dots \vee a_m)$ as in Definition~\ref{def:temporal:rule} and
interpreted as $r$, $\wnext\alwaysF{r}$, or $\alwaysF(\finally \to r)$ depending on whether it occurs
in the scope of a program declaration headed by \lstinline{initial}, \lstinline{dynamic}, or \lstinline{final}, respectively.
Additionally, \telingo\ offers \lstinline{always} for gathering rules preceded by \alwaysF\
(thus dropping \wnext\ from dynamic rules).
A rule outside any such declaration is regarded to be in the scope of \lstinline{initial}.
This allows us to represent the temporal program in~\eqref{ex:program} in the two alternative ways shown in Table~\ref{tab:enc:programs}.
\begin{table}[ht]
\centering
\begin{minipage}[t]{120pt}
\begin{lstlisting}[frame=single,numbers=none,belowskip=0pt,aboveskip=0pt,basicstyle=\ttfamily]
#program initial.
a.

#program dynamic.
b :- 'a.

#program final.
  :- not b.
\end{lstlisting}
\end{minipage}
\qquad\qquad
\begin{minipage}[t]{120pt}
\begin{lstlisting}[frame=single,numbers=none,belowskip=0pt,aboveskip=0pt,basicstyle=\ttfamily]
#program always.
a :- &initial.


b :- 'a.


  :- not b, &final.
\end{lstlisting}
\end{minipage}\caption{Two alternative \telingo\ encodings for the temporal program in~\eqref{ex:program}}
\label{tab:enc:programs}
\end{table}

As mentioned,
\telingo\ allows us to use nested temporal formulas in integrity constraints as well as in negated form in place of temporal literals within rules.
This is accomplished by encapsulating temporal formulas like $\varphi$ in expressions of the form
`\lstinline[mathescape]+&tel { $\varphi$ }+'.
To this end, the full spectrum of temporal operators is at our disposal.
They are expressed by operators built from \lstinline{<} and \lstinline{>} depending on whether they refer to the past or the future, respectively.
So,
\lstinline{<}/1, \lstinline{<?}/2, and \lstinline{<*}/2
stand for \previous, \since, and \trigger,
and
\lstinline{>}/1, \lstinline{>?}/2, \lstinline{>*}/2
for \Next, \until, \release.
Accordingly,
\lstinline{<*}/1,
\lstinline{<?}/1,
\lstinline{<:}/1
represent \alwaysP, \eventuallyP, \wprevious, and
analogously their future counterparts.
\initially\ and \finally\ are  are represented by \lstinline{&initial} and \lstinline{&final}.
This is complemented by Boolean connectives \lstinline{&}, \lstinline{|}, \lstinline{~}, etc.
For example, the integrity constraint~\footnote{Similarly, formula~\eqref{f:fail} could be represented by several present-centered rules (including auxiliary atoms).}
`$
\mathit{shoot} \wedge \alwaysP \mathit{unloaded} \wedge \previous \eventuallyP \mathit{shoot} \to \bot
$'
can be expressed as follows.
\begin{lstlisting}[numbers=none,belowskip=2pt,aboveskip=2pt,basicstyle=\ttfamily]
   :- shoot, &tel { <* unloaded & < <? shoot }.
\end{lstlisting}

Once \telingo\ has translated an (extended) temporal program into a regular one,
it is incrementally solved by \clingo's multi-shot solving engine~\cite{gekakasc17a}.

To conclude this section, we provide a larger example of \telingo\ encoding (a more detailed description of the system was presented by~\citeN{cakamosc19a}).
Listing~\ref{lst:wolf} contains an exemplary \telingo\ encoding of the \emph{Fox, Goose and Beans Puzzle} available at
\url{https://github.com/potassco/telingo/tree/master/examples/river-crossing}.
\begin{quote}
  {\em Once upon a time a farmer went to a market and purchased a fox, a goose, and a bag of beans. On his way home, the farmer came to the bank of a
    river and rented a boat. But crossing the river by boat, the farmer could carry only himself and a single one of his purchases: the fox, the
    goose, or the bag of beans.  If left unattended together, the fox would eat the goose, or the goose would eat the beans.  The farmer's challenge
    was to carry himself and his purchases to the far bank of the river, leaving each purchase intact. How did he do it?  }
  \\{\scriptsize\smallskip\mbox{}\hfill(\url{https://en.wikipedia.org/wiki/Fox,_goose_and_bag_of_beans_puzzle})}
\end{quote}
\begin{lstlisting}[float=t,captionpos=b,caption={\telingo\ encoding for the Fox, Goose and Beans Puzzle},label=lst:wolf,basicstyle=\ttfamily\small]
#program always.

item(fox;beans;goose).
route(river_bank,far_bank). route(far_bank,river_bank).
eats(fox,goose). eats(goose,beans).

#program initial.

at(farmer,river_bank).
at(X,river_bank) :- item(X).

#program dynamic.

move(farmer).
0 { move(X) : item(X) } 1.

at(X,B) :- 'at(X,A), move(X), route(A,B).
:- move(X), item(X), 'at(farmer,A), not 'at(X,A).

at(X,A) :- 'at(X,A), not move(X).

#program always.

:- at(X,A), at(X,B), A<B.
:- eats(X,Y), at(X,A), at(Y,A), not at(farmer,A).

#program final.

:- at(X,river_bank).

#show move/1.
#show at/2.
\end{lstlisting}
In Listing~\ref{lst:wolf}, lines 3-5 and 9-10 provide facts holding in all and the initial states, respectively;
this is indicated by the program directives headed by \texttt{always} and \texttt{initial}.
The dynamic rules in lines 14-22 describe the transition function.
The \texttt{farmer} moves at each time step (Line~14), and may take an item or not (Line~15).
Line~17 describes the effect of action \texttt{move/1}, Line~18 its precondition, and Line~20 the law of inertia.
The second part of the \texttt{always} rules give state constraints in Line~24 and~25.
The \texttt{final} rule in Line~29 gives the goal condition.
All in all, we obtain two shortest plans consisting of eight states in about 20 ms.
Restricted to the move predicate, \telingo\ reports the following solutions:

\medskip
\noindent \begin{minipage}{\linewidth}
\centering
\begin{tabular}{cll@{\hskip 1em}ll}\toprule
Time & \multicolumn{2}{c}{Solution 1} & \multicolumn{2}{c}{Solution 2} \\ \midrule
1 &                    &                    &              &              \\
2 & \texttt{move(farmer)} & \texttt{move(goose)}  & \texttt{move(farmer)} & \texttt{move(goose)}\\
3 & \texttt{move(farmer)} &                    & \texttt{move(farmer)} & \\
4 & \texttt{move(beans)}  & \texttt{move(farmer)} & \texttt{move(farmer)} & \texttt{move(fox)}\\
5 & \texttt{move(farmer)} & \texttt{move(goose)}  & \texttt{move(farmer)} & \texttt{move(goose)}\\
6 & \texttt{move(farmer)} & \texttt{move(fox)}    & \texttt{move(beans)}  & \texttt{move(farmer)}\\
7 & \texttt{move(farmer)} &                    & \texttt{move(farmer)} & \\
8 & \texttt{move(farmer)} & \texttt{move(goose)}  & \texttt{move(farmer)} & \texttt{move(goose)}\\ \bottomrule
\end{tabular}
\end{minipage}
\medskip

This example was also used by~\citeN{cabdie11a} to illustrate the working of
\stelp, a tool for temporal answer set programming with \TELo.
We note that \stelp\ and \telingo\ differ syntactically in describing transitions by using next or previous operators, respectively.
Since \telingo\ extends \clingo's input language, it offers a richer input language,
as witnessed by the cardinality constraints in Line~15 in Listing~\ref{lst:wolf}.
Finally, \stelp\ uses a model checker and outputs an automaton capturing all infinite traces
while \telingo\ returns finite traces corresponding to plans.
 \section{\TELf\ and the action language $\mathcal{BC}$}\label{sec:action_languages}
\newcommand{\BC}{\ensuremath{\mathcal{BC}}}
\newcommand{\BCatom}{\ensuremath{a}}

Present-centered temporal programs over finite traces are sufficient to capture common action languages.
We show this by providing a translation from action descriptions in action language \BC~\cite{leliya13a} into
present-centered temporal logic programs in \TELf.
To this end, we need the following definitions describing the major concepts of \BC.
An \emph{action signature} in \BC\ includes two finite set of symbols,
the set \textbf{F} of {fluent constants} and the set \textbf{A} of {action constants}.
Fluent constants are further divided into {regular} and {statically determined} fluent constants.
A finite set with at least two elements, called {\em domain},
is assigned to each fluent constant.
The set of all values of all domains is denoted by \textbf{V}.
In \BC, an atom is an expression of the form $f = v$, where $f$ is a fluent constant, and $v$ is an element of its domain.
If the domain of $f$ consists of truth values $\top$ and $\bot$, we say that $f$ is Boolean.
A \emph{static law} is an expression of the form
\begin{align}\label{def:static_law}
  \BCatom_0\ &\mathbf{if} ~ \BCatom_1,\dotsc,\BCatom_m ~ \mathbf{ifcons} ~ \BCatom_{m+1},\dotsc,\BCatom_n
               \qquad\text{ with } n \geq m \geq 0,
\end{align}
where each $\BCatom_i$ is an atom for $0 \leq i \leq n$.
A \emph{dynamic law} is an expression of the form
\begin{align}\label{def:dynamic_law}
  \BCatom_0\ &\mathbf{after} ~ \BCatom_1,\dotsc,\BCatom_m ~ \mathbf{ifcons} ~ \BCatom_{m+1},\dotsc,\BCatom_n
               \quad\text{ with } n \geq m \geq 0,
\end{align}
where
$\BCatom_0$ is an atom containing a regular fluent constant,
each of $\BCatom_1,\dotsc,\BCatom_m$ is an atom or an action constant, and
$\BCatom_{m+1},\dotsc,\BCatom_n$ are atoms.
Note that statically determined fluent constants are not allowed in the heads of dynamic laws,
but only in the heads of static laws (hence the term ``statically determined'').
An \emph{action description} in \BC\ consists of a finite set of static and dynamic laws.
The semantics of action descriptions in \BC\ is given by a translation into a
sequence of logic programs with nested expressions.

More precisely, given an action description $D$,
a sequence of logic programs with nested expressions $(N_i(D))_{i \geq 0}$ is defined as follows.
For $l \geq 0$, the signature of each program $N_l(D)$ is given by
\begin{align*}
  \{ i:(v=f) \mid f \in \textbf{F}, v \in \textbf{V}, 0 \leq i \leq l \}
  \cup
  \{ i:a     \mid a \in \textbf{A},                   0 \leq i \leq l \}.
\end{align*}
The logic program $N_l(D)$ contains
\begin{itemize}
\item
  for each static law from $D$ of form \eqref{def:static_law} and
  for all $0 \leq i \leq l$,
  the rule
  \begin{align*}\label{BC:rule:static_law}
    i: \BCatom_0
    \leftarrow
    i:\BCatom_1\wedge\dots\wedge i:\BCatom_m
    \wedge
    \neg\neg (i:\BCatom_{m+1})\wedge\dots\wedge \neg\neg (i:\BCatom_n)
  \end{align*}
\item
  for each dynamic law from $D$ of the form \eqref{def:dynamic_law} and
  for all $0 \leq i < l$,
  the rule
  \begin{align*}
    (i+1): \BCatom_0
    \leftarrow
    i:\BCatom_1\wedge \dots\wedge i:\BCatom_m
    \wedge
    \neg\neg ((i+1):\BCatom_{m+1})\wedge \dots\wedge \neg\neg((i+1):\BCatom_n)
  \end{align*}
\item the choice rule $\left\{ 0 : \BCatom \right\}$ for every atom $\BCatom$ containing a regular fluent constant,
\item the choice rule $\left\{ i : a \right\}$ for every action constant $a$ and every $i < l$,
\item
  for every fluent constant $f$ and every $i \leq l$,
  where $v_1, \dots, v_k$ are all elements of the domain of $f$,
  the rule
  \begin{align*}
    \leftarrow \neg i : (f = v_1 )\wedge \dots \wedge \neg i : (f = v_k )
  \end{align*}
  ensuring the existence of values,
\item
  for every fluent constant $f$, every pair of distinct elements $v, w$ of its domain, and every $i \leq l$,
  the rule
  \begin{align*}
    \leftarrow i : (f=v)\wedge i : (f=w)
  \end{align*}
  ensuring the uniqueness of values.
\end{itemize}
The transition system $T(D)$ represented by the action description $D$ is then defined as follows.
For every stable model $X$ of $N_0(D)$,
the set of atoms \BCatom\ such that $0:\BCatom$ belongs to $X$ is a state of $T(D)$.
For every state $s$ and every fluent constant $f$
there is exactly one $v$ such that $f = v$ belongs to the state $s$.
This $v$ is the value of $f$ in state $s$.
For every stable model $X$ of $N_1(D)$,
$T(D)$ includes the transition $\langle s_0,\alpha,s_1\rangle$,
where $s_i$ is the set of atoms \BCatom\ such that $i:\BCatom$ belongs to $X$ for $i\in\{0,1\}$,
and $\alpha$ is the set of action constants $a$ such that $0:a$ belongs to $X$.

We now provide a translation from the nested logic programs defining the semantics of \BC\
into present-centered temporal logic programs.
For any action description $D$ in \BC, we define the
temporal logic program $P(D)$ consisting of
\begin{itemize}
\item
  for each static law from $D$ of form \eqref{def:static_law},
  the rules
  \begin{align*}
    &                 \BCatom_1 \wedge\dots\wedge \BCatom_m \to \BCatom_0 \vee \neg\BCatom_{m+1} \vee\dots\vee \neg\BCatom_n\qquad\text{ and }\\
    \wnext \alwaysF (&\BCatom_1 \wedge\dots\wedge \BCatom_m \to \BCatom_0 \vee \neg\BCatom_{m+1} \vee\dots\vee \neg\BCatom_n),
  \end{align*}
\item
  for each dynamic law from $D$ of form \eqref{def:dynamic_law},
  the rule
  \begin{align*}
    \wnext \alwaysF (\previous \BCatom_1 \wedge\dots\wedge \previous \BCatom_m \to \BCatom_0 \vee \neg \BCatom_{m+1} \vee\dots\vee \neg \BCatom_n),
  \end{align*}
\item the formula $\BCatom\vee\neg\BCatom$ for every atom $\BCatom$ containing a regular fluent constant,
\item the formulas $a\vee\neg a$ and $\wnext \alwaysF (a\vee\neg a)$ for every action constant $a$,
\item
  for every fluent constant $f$, where $v_1,\dotsc,v_k$ are all elements of the domain of $f$,
  the rules
  \begin{align*}
                     &(f = v_1 ) \vee\dots\vee (f = v_k ) \qquad\text{ and }\\
    \wnext \alwaysF (&(f = v_1 ) \vee\dots\vee (f = v_k )),
  \end{align*}
\item
    for every fluent constant $f$ and every pair of distinct elements $v, w$ of its domain,
    the rules
  \begin{align*}
                     &(f = v) \wedge (f = w) \to \bot \qquad \text{ and }\\
    \wnext \alwaysF (&(f = v) \wedge (f = w) \to \bot).
  \end{align*}
\end{itemize}
With this, we get the following result.
\begin{theorem}[Relation between stable models of $N_l(D)$ and $P(D)$]\label{theorem:bijection_PNl_PT}
  Let $D$ be an action description in \BC\ over action signature $\langle \textbf{A, F, V}\rangle$,
  consisting of actions \textbf{A}, fluents \textbf{F} and values \textbf{V}.

  Then, we have the following:
  \begin{enumerate}
  \item If $({X}_i)_{i\in\intervo{0}{l}}$ is a temporal stable model of $P(D)$ of length $l>0$, then
  \begin{align*}
    \{i: a     \mid a     \in {X}_i, a \in \textbf{A},                   0 \leq i < l\} \cup
    \{i: (f=v) \mid (f=v) \in {X}_i, f \in \textbf{F}, v \in \textbf{V}, 0 \leq i < l\}
  \end{align*}
  is a stable model of $N_{l-1}(D)$.
\item If $X$ is a stable model of $N_l(D)$ with $l\geq 0$, then
  \begin{align*}
    (
    \{ (v=f) \mid (i: (v=f)) \in X, v \in \textbf{V}, f \in \textbf{F} \} \cup
    \{  a    \mid (i: a)     \in X, a \in \textbf{A}                   \}
    )_{0 \leq i \leq l}
  \end{align*}
  is a temporal stable model of $P(D)$.
\end{enumerate}
\end{theorem}

This provides us with an alternative characterization of the
transition system corresponding to an action description in \BC.
\begin{corollary}
  Let $D$ be an action description in \BC\ and $P(D)$ the logic program defined as above.
  Then,

  \begin{enumerate}
  \item
    for every stable model $({X}_i)_{i\in\intervo{0}{1}}$ of $P(D)$ of length one,
    the set of atoms that belong to ${X}_0$ is a state of the transition system $T(D)$ and
  \item
    for every stable model $({X}_i)_{i\in\intervo{0}{2}}$ of $P(D)$ of length two,
    the transition system $T(D)$ includes the transition $\langle s_0,\alpha,s_1\rangle$,
    where
    \begin{enumerate}
    \item $s_i$ is the set of atoms $\BCatom$ such that $\BCatom$ belongs to ${X}_i$ for $i \in \{ 0, 1\}$ and
    \item $\alpha$ is the set of action constants that belong to ${X}_0$.
    \end{enumerate}
  \end{enumerate}
\end{corollary}
\citeN{leliya13a} showed that paths of length $l$ in the transition system
described by an action description $D$
correspond to the stable models of $N_l(D)$.
Using Theorem \ref{theorem:bijection_PNl_PT}, we can also characterize these paths
in terms of temporal stable models.
\begin{corollary}[Translation from \BC\ into \TELf]\label{theorem:translation_bc}
  Let $D$ be an action description in \BC\ and
  $P(D)$ the logic program defined as above.

  A temporal trace $({X}_i)_{i\in\intervo{0}{l}}$ is a temporal stable model of $P(D)$ iff
  the sequence $(s_i )_{i\in\intervo{0}{l}}$ where $s_i$ is the set of atoms in $X_i$ is a
  path in the transition system corresponding to $D$.
\end{corollary}
 \section{Related Work}
\label{sec:related}

Covering the vast literature that falls in the intersection of temporal reasoning and logic programming is too ambitious for the purpose of the current survey (focused on extending ASP with modal, LTL operators) and would surely require a more extensive document.
In this section, we overview those approaches that show a closer connection to the kind of temporal logic programs considered in the paper.
We divide this related work into two categories: (1) extending the language with modal operators based on linear traces and (2) using logic programming for temporal reasoning.

\subsection{Extending the syntax with linear modalities}
As mentioned in the introduction, the first use of the term ``\emph{Temporal Logic Programming}''~\cite{abaman89a} dates back to the late 1980s, appearing soon after the introduction of modal extensions of Prolog~\cite{farinas86a,bifahe88a}.
Several Logic Programming (LP) languages dealing with \LTL\ operators were proposed~\cite{moszkowski86a,fukotamo86a,gabbay87b,abaman89a,orgwad92a,baudinet92a}.
The latter, a formalism called TEMPLOG, is perhaps a prominent case from a logical point of view.
It provides a logical semantics in terms of a least \LTL-model, in the spirit of the well-known least Herbrand model for positive\footnote{Note that the different semantics for negation as failure were still in their early steps at that moment.} logic programs~\cite{emdkow76a}.
\citeN[Theorem~5]{cadivi15a} proved that TEMPLOG is actually subsumed by \TEL, that is, the latter can be used as a generalization of the former for an arbitrary temporal syntax that includes default negation.
However, although \TEL\ provides a common underlying semantics to TEMPLOG and our temporal programs (which are the basis of \telingo), these two languages understand the (modal) temporal logic programming paradigm in a substantially different way, analogous to the differences between Prolog and ASP.
In particular, TEMPLOG understands rules in a top-down fashion where the head is considered a \emph{goal} and the body is seen as a \emph{method} (list of subgoals) to achieve the goal.
As a result, (future-time) \LTL\ operators are used to represent temporal relations affecting the achievement of these subgoals.
As an example, one possible TEMPLOG rule could look like
\begin{eqnarray*}
\alwaysF ( \ \mathit{printorder} \leftarrow \mathit{ack}, \next (\mathit{print}, \eventuallyF{finished}) \ )
\end{eqnarray*}
meaning that, at any moment, a $\mathit{printorder}$ is fulfilled by immediately sending an acknowledgment $\mathit{ack}$, then starting the printer $\mathit{print}$  and eventually sending a $\mathit{finished}$ message.
This TEMPLOG rule naturally corresponds to the \TEL\ implication
\begin{eqnarray*}
\alwaysF ( \  \mathit{ack} \wedge \next (\mathit{print} \wedge \eventuallyF{finished}) \to \mathit{printorder} \ )
\end{eqnarray*}
and has indeed the same semantics in \TEL.
However, it does not fit into the past-future syntactic fragment used in \telingo, since the $\mathit{printorder}$ in
the head is in the relative past of the conditions required in the body, talking about future satisfaction of
$\mathit{print}$ or $\mathit{finished}$.
Following the ASP bottom-up understanding of rules, which is closer to causal laws in action languages, this same example would be represented instead as the past-future formula:
\begin{eqnarray*}
\alwaysF ( \mathit{printorder} \to \mathit{ack} \wedge \next (\mathit{print} \wedge \eventuallyF \mathit{finished})  )
\end{eqnarray*}
which can be reduced afterwards to a present-centered logic program.
To sum up, \TEL\ is rich enough to cover both bottom-up and top-down readings of program rules, but for its use for temporal ASP, it suffices with a syntactic fragment where past is checked in the bodies and future is used in the head.
One last \LTL-based LP extension was the preliminary definition of \emph{temporal answer sets} provided by~\citeN{cabalar99a}.
The main difference of that approach with respect to \TEL\ is that the former relied on a different three-valued semantics that was not a proper extension of Equilibrium Logic, losing some of the interesting properties that the latter has shown as a logical foundation of ASP.

Apart from these LP extensions strictly based on \LTL, there exist other approaches that introduced other temporal modalities in LP or in ASP and relied on linear-time interpretations or traces.
The most relevant and closely related approach is the one by~\citeN{gimath13a}, where a different definition of \emph{temporal answer set} is provided.
In this case, rather than \LTL,
the temporal approach used as a basis was \emph{Dynamic Linear Temporal Logic} or \DLTL~\cite{henthi99a},
a linear-time variant of the well-known \emph{Dynamic Logic}~\cite{hatiko00a}.
Besides, that definition of temporal answer sets was only applicable to a syntactic fragment of \DLTL\ (also called there ``temporal logic programs'') since the semantics relied on a temporal extension of the classical (syntactic) \emph{reduct} transformation~\cite{gellif88a}.
In order to facilitate a formal comparison,~\citeN{agpevi13a} introduced a proper extension of \TEL\ that covers the full syntax of \DLTL.
The main result in that paper proved that, in fact, this extension of \TEL\ for \DLTL\ subsumes the semantics by~\citeN{gimath13a} and, in fact, provides its generalization for \DLTL\ arbitrary formulas, without depending on syntactic transformations.

Thanks to the incorporation of regular expressions, dynamic modalities are very interesting for the specification of \emph{control rules} describing steps that are required to be followed by any solution to a temporal problem.
Differently from \DLTL, another (perhaps more direct) approach to introducing dynamic logic modalities on top of a linear-time semantics is the so-called \emph{Linear Dynamic Logic} or \LDL~\cite{giavar13a}.
Both \DLTL\ and \LDL\ are more expressive than \LTL, while satisfiability in the three logics share the same \textsc{PSpace} complexity.
However, \DLTL\ uses regular expressions as modifiers of the \emph{until} and \emph{release} operators, whereas the syntax of \LDL\ is closer to the usual in dynamic logic, where modalities include the standard necessity $[\rho] \varphi$ and possibility $\tuple{\rho}\varphi$ constructs, where $\rho$ is a regular expression.
Besides, \DLTL\ regular expressions are built on a separate signature (a set of atomic \emph{actions}) and do not accept the test construct $\varphi?$ from dynamic logic.
\LTL\ operators can be encoded both in \LDL\ and \DLTL, although the encodings for the latter require an explicit use of the (finite) set of atomic actions.
One more advantage of \LDL\ is that a finite trace variant \LDLf\ has also been proposed and that both \LDL\ and \LDLf\ properly generalize \LTL\ and \LTLf\ respectively.
All these reasons make \LDL\ a more promising candidate to incorporate dynamic operators in ASP.
As a result, \citeN{bocadisc18a} followed analogous steps to \TEL\ with respect to \LTL\ and introduced a non-monotonic extension of \LDL\ called \emph{Linear Dynamic Equilibrium Logic} (\DEL).
Moreover, a first implementation of \LDL\ operators in \telingo\ has been recently presented by~\citeN{cadilasc20a},
even though they are only allowed in integrity constraints.
Although this restriction will be lifted in the future, it already supports an agreeable modeling methodology for dynamic domains separating action and control theories.
The idea is to model the actual action theory with temporal rules, fixing static and dynamic laws, while the control theory, enforcing certain (sub)trajectories, is expressed by integrity constraints using dynamic formulas.
This is similar to the pairing of action theories in situation calculus and Golog programs~\cite{lereleli97a}.

Other modal temporal extensions of LP that rely on linear time have been recently introduced for \emph{stream} reasoning, that is, reactive reasoning over time-varying data sequences.
An approach closely related to \TEL\ is LARS~\cite{bedaeifi15a} that deals with time windows or intervals and combines modal \LTL\ operators $\alwaysF$ and $\eventuallyF$ with metric and time-point based connectives.
Although LARS' semantics is also based on discrete linear-time,
an important difference is that its traces are organized in intervals of time points called \emph{streams}.
An informal discussion comparing \TEL\ and LARS was done by~\citeN{bedaei16a} although a formal connection to the corresponding metric extension of \TEL~\cite{cadiscsc20a} remains to be explored.
As proved by~\citeN{kamp68a},
the use of discrete time in \LTL\ is not a limitation because any Dedekind complete ordering structure can be used as a timeline,
and this includes natural numbers $\mathbb{N}$ or integer numbers $\mathbb{Z}$ but also real $\mathbb{R}$ numbers\footnote{Other dense structures that are not Dedekind complete, like the rational numbers $\mathbb{Q}$, can be covered by extending the language with Stavi connectives~\cite{gapnshsh80a}, while keeping decidability.}.
However, when metric operators come into play, the choice of \emph{dense} time actually leads to a richer expressiveness, usually at the price of undecidability for an arbitrary syntax.
For this reason, the most common approach for dense metric extensions is to identify syntactic fragments that keep amenable computational properties, like the case of the so-called \emph{bounded universal Horn formulas}~\cite{brzoska95a} for positive programs.
A more general fragment extending LP with metric modalities is DatalogMTL, a temporal metric extension of Datalog presented by~\citeN{wagrkaka19a} and with applications to stream reasoning.
Recently, this approach has been extended to programs with negation~\cite{watekocu21a} but forbidding the use of $\eventuallyF$, $\until$ or $\since$ in rule heads.
This fragment is decidable and keeps favourable computational properties for the case of an integer timeline, although it jumps to undecidability for dense time, even if considering the rational numbers.
Again, a formal connection to the metric extension of \TEL~\cite{cadiscsc20a} remains to be studied.

The combination of temporal modalities and non-monotonic reasoning is not exclusive to logic programming approaches and was also explored inside the area of Reasoning about Actions and Change.
An early approach using \LTL\ is, for instance, the \emph{Past Temporal Logic for Actions} or PTLA~\cite{mellloba96a}, an action language incorporating past \LTL\ operators in the conditions of action laws.
These action descriptions were then translated into logic programs under the ASP semantics using a similar transformation as the one presented in Section~\ref{sec:bounded}.
Several authors have explored the formalization of action and change using Dynamic
Logic~\cite{bagimapa96a,schwind97a,cagahe99a} or \DLTL~\cite{gimasc00a},
though only the first one relies on a logic programming paradigm.

Another prominent AI field for reasoning about dynamic systems where temporal logic has played an important role is Planning.
In this area, the system behavior is specified in terms of some planning-specific language like STRIPS or PDDL, with a carefully limited syntax that avoids the need for an explicit representation of the (non-monotonic) law of inertia.
This limited syntax restricts the specification of dynamic systems to a less expressive language than temporal ASP with \telingo, where defaults, induction or aggregates can be freely combined with temporal operators, but has the advantage of allowing the design and implementation of efficient planning algorithms.
Rather than in the description of the transition system, the introduction of temporal formulas in planning has been traditionally related to a richer specification of the \emph{goal}, so that it not only provides a condition about the final state to be reached, but is also extended with temporal formulas~\cite{backab00a,pistra01a,becirotr01a,kvhedo08a} imposing constraints on the sequence of actions that form the plan, something that can be used to improve the efficiency of the planning algorithm.
This strategy can also be extrapolated to temporal ASP, although the effect of temporal constraints in \telingo\ has a different impact on efficiency, given that the algorithm is based on incremental solving.
A different strategy for temporal ASP closer to classical planning algorithms was explored in the prototype presented by~\citeN{carevi19a}.
In this case, a classical graph search algorithm was implemented and multi-shot ASP solving was used to compute successor states during search.
This has the advantage of being able to explore the whole state space and deciding whether a given planning problem has no solution, something impossible by the incremental horizon strategy followed in \telingo.
However, this planning based prototype had a rather limited syntax and did not allow the use of temporal expressions to incorporate temporal goals.

Other AI areas where temporal logic has been successfully applied are, for instance, Multi Agent Systems (MAS) and Ontologies.
In the case of MAS, a prominent example is the system Concurrent \texttt{MetateM}~\cite{fisher94a}, a language using \LTL\ operators that exploits Gabbay's separation theorem~\cite{gabbay87a} asserting that \LTL\ descriptions can be reduced to a set of implications with the form ``\emph{past} $\to$ \emph{future}.''
Note that we also exploited this past-future form in \telingo\ and our Theorem~\ref{thm:normalform:pc}.
The main difference between \texttt{MetateM} and temporal ASP, however, is that the former relies on monotonic \LTL,
so it does not allow for closed world reasoning, defaults or induction.
On the other hand, \texttt{MetateM} incorporates high level constructs for agents specification and message passing that are not present in \telingo.

Finally, in the case of \LTL\ in Ontologies, the combination of Description Logics~\cite{DLHandbook} with temporal patterns is an important field of knowledge representation that has been widely studied in the literature (see, for instance, the surveys by~\citeN{artfra00a}, \citeN{artfra05a} and~\citeN{luwoza08a}).
However, as happened in the MAS case, the result of these combinations are also monotonic.
Combining ontological reasoning and logic programming has also been extensively studied,
with the ASP extension called \emph{Hybrid Knowledge Bases}~\cite{brpepova10a} being perhaps the closest one to our approach,
as it is based on Quantified Equilibrium Logic.
In a preliminary work, \citeN{cabsch19a} presented a combination of hybrid knowledge bases with the temporal ontology
framework ALC-\LTL\ by~\citeN{baghlu08a}.
This allowed the use of temporal ontological axioms in temporal logic programs, but no implementation has been made so far.

\subsection{Using Logic Programming for Temporal Reasoning}

A second line of related approaches are those that perform temporal reasoning \emph{using} logic programming (especially Datalog or ASP) as a reasoning engine.
In this case, the idea is not necessarily to extend the logic programming syntax with temporal modalities, but to use logic programming as a tool for query answering or model checking with respect to some standard, monotonic temporal logic (most frequently, LTL).
If we use Datalog or ASP, the main inconvenience appears when trying to reason about traces of infinite length, since
these formalisms historically required a finite domain to allow for a previous grounding of the logic program.
For instance, the use of functions was usually forbidden (or limited) to keep a finite Herbrand universe.
This limitation was overcome by the early proposal by~\citeN{choimi88a} consisting in the extension of Datalog with a unary \emph{successor} function, to be used inside a fixed parameter in predicates.
This approach, called Datalog$_{1S}$, allows for reasoning about infinite traces while keeping  decidability.
Moreover, it was proved~\cite{bachwo93a} to be equally expressive to TEMPLOG, although both languages treat time in a different way: TEMPLOG uses \emph{intensional}, temporal modalities, whereas Datalog$_{1S}$ uses an \emph{extensional} approach, representing time points as explicit parameters.
Later on, this formalism was extended to so-called Datalog$_{nS}$~\cite{choimi93a} and used as a reasoning system for answering temporal queries expressed in past LTL.
A similar approach to Datalog$_{1S}$, based on the introduction of a successor unary function, was adopted by~\citeN{eitsim09a} for the more general case of ASP.
Other extensions of Datalog that have been used as reasoning engines for temporal logics are Datalog LITE~\cite{gogrve02a} or \emph{inf}-Datalog~\cite{gufoanaf03a}.
Notice that all these cases are variants of Datalog used to perform model checking for temporal logics such as LTL, CTL or $\mu$-calculus, but are not modal temporal extensions of Datalog.

Another extensional alternative is, rather than using a successor function, to treat time points directly as integer numbers, relying on Constraint Logic Programming (CLP) to solve the integer constraints imposed by temporal expressions.
This was the approach adopted by~\citeN{brzoska91a}, who emulated the operational behavior of TEMPLOG using a CLP translation where each predicate received an additional integer argument, representing time.
One more extensional temporal approach based on CLP is, for instance, the Temporal Annotated CLP formalism~\cite{fruhwirth96a}.
In this case, predicates can be annotated with expressions such as `$\mathit{at}\ T$', `$\mathit{th}\ I$' or `$\mathit{in}\ I$', respectively meaning ``at time point $T$'', ``throughout interval $I$'' and ``in interval $I$'', respectively.
Finally, an early approach using CLP for translating temporal formulas was introduced by~\citeN{Hrycej88a}, but in this case, rather than treating time points as integers, the constraint solver used qualitative constraints for Allen's Interval Calculus~\cite{allen83a}.

A prominent application of ASP for temporal reasoning is the use of ASP solvers to perform bounded model checking of temporal formulas.
The technique of \emph{bounded model checking} (BMC) was introduced by~\citeN{biciclzh99a} as a variant of model checking that looks for countermodels of a temporal formula, but only for traces up to some bounded length $k \in \Nat$.
The original approach translated the LTL formula into a propositional theory and then used a SAT solver as a backend.
\citeN{helnie03a} proved that an ASP solver could be used instead of SAT, providing a more compact translation of LTL formulas while keeping an efficient implementation.
A similar approach was later followed by~\citeANP{gimath12a}~\citeyear{gimath12a,gimath13a} to implement BMC for action theories and temporal answer sets.
Again, although the use of ASP for BMC is clearly in the intersection of temporal logics and logic programming, it does not necessarily represent a temporal extension of the latter.
Note, for instance, that the tool by~\citeN{helnie03a} provides a model checker whose input language are temporal formulas with the standard \emph{monotonic} semantics of LTL.

Finally, one more application of ASP for temporal reasoning is its use for encoding Qualitative (Spatio-)Temporal Calculi~\cite{li12a,brfaba16a}.
In this case, the target temporal formalism is not a modal temporal logic, but some calculus for a set of qualitative relations like, for instance, the aforementioned Allen's Interval Calculus.
For this reason, the relation to the approach studied in the current survey is weaker while, once again, ASP is used as a reasoning engine for encoding the temporal calculi.

 \section{Conclusion}\label{sec:discussion}

We have provided a wide overview of the main definitions and recent results for the  formalism of \emph{Temporal Equilibrium Logic} (\TEL), a combination of Equilibrium Logic (a logical characterization of Answer Set Programming, ASP) with Linear-Time Temporal Logic.
After more than a decade of study, the knowledge about \TEL\ has achieved a high degree of maturity, both at a fundamental level and also at an incipient practical application.
An important breakthrough for the latter has been the introduction of a finite trace variant, \TELf, more aligned with the usual problem solving orientation followed in ASP.
This has opened the possibility of introducing temporal operators in ASP solving with the construction of the tool \telingo,
a temporal extension of \clingo\ that exploits the incremental solving capabilities of the latter.
Moreover, the definition of \TELf\ has also provided a methodology for the study of other extensions, like the introduction of dynamic logic operators~\cite{cadilasc20a} or, more recently, the definition of temporal metric constructions~\cite{cadiscsc20a}
that are planned to be further developed and combined with other features inside the software packages of \url{potassco.org}.

This survey has been focused on the \LTL\ extension of ASP where many interesting lines remain to be explored yet.
For instance, we still miss an axiomatization that covers the finite trace variant of the monotonic basis, \THTf, or a general axiomatization that is applicable to \THT, that is, regardless of the trace length.
Also, properties related to inter-definability of operators need to be further explored for the finite trace case.
Related to this,~\citeN{agcafapevi20a} proposed the introduction of an explicit negation operator that allows for
inter-definability of temporal operators like \emph{until} and \emph{release} using De~Morgan laws.
The connection of \TELf\ to finite automata is also a topic of active current research.
Automata construction techniques may become crucial in the future both for formal verification of dynamic systems specified with \telingo\ but also for exploiting temporal expressions in a more compact way during computation, rather than making a full unfolding into regular ASP programs.
Another important line of future study is how to exploit temporal constructions for grounding: right now, \telingo\ does not recognize temporal information at that stage.
A first study about temporal grounding was provided by~\citeN{agcapevidi17a}, but it focused on infinite traces,
making these results not extrapolatable to the case of \telingo,
where the hypothesis of finite traces can be exploited more efficiently.

In a broader picture, we are facing a true rebirth of (Modal) Temporal Logic Programming, which is becoming a thrilling area with an intense research both on foundations and in practical implementation.
As we have seen in Section~\ref{sec:related}, there exist several temporal extensions of ASP and Datalog with applications to planning and problem solving on dynamic domains, temporal Knowledge Representation and, more recently, stream reasoning.
It seems clear that practical applications may require a richer language than the \LTL\ syntax, as happens with some of the ASP extensions introducing metric constructs or path expressions from dynamic logic.
One advantage of the approach we have presented is its characterization in terms of Equilibrium Logic,
allowing for an easy adaptation to other modal temporal logics
(first extend the modal logic with \HT\ and then add the equilibrium models condition).
However, there exist many alternative temporal approaches and choosing which one is the most suited
(in terms of expressiveness, complexity and ease of use)
for a dynamic problem at hand may turn out to be a challenging question in the future.

 \subsubsection*{Acknowledgments}

We are thankful to the anonymous reviewers for their thorough work and their useful suggestions that have helped to
improve the paper.
A special thanks goes to Mirosław Truszczy{\'n}ski for his support in improving the quality of our paper.
We are especially grateful to David Pearce, whose help and collaboration on Equilibrium Logic was the seed for a great
part of the current paper.
This work was partially supported by MICINN, Spain, grant PID2020-116201GB-I00, Xunta de Galicia, Spain (GPC ED431B
2019/03), R{\'e}gion Pays de la Loire, France, (projects EL4HC and etoiles montantes CTASP), European Union COST action CA-17124, and DFG grants SCHA 550/11 and~15, Germany.

\bibliographystyle{include/latex-class-tlp/acmtrans}

\end{document}